\documentclass[twoside,11pt]{article}

\usepackage{include/jmlr2e}
\usepackage[utf8]{inputenc} 
\usepackage[T1]{fontenc}    
\usepackage{url}            
\usepackage{booktabs}       
\usepackage{amsfonts}  
\usepackage{amsmath} 
\usepackage{nicefrac}       
\usepackage{microtype}      
\usepackage{xcolor}
\usepackage{hyperref}  
\usepackage{graphicx}
\usepackage{array}
\usepackage{subcaption,mathrsfs}
\usepackage{natbib}
\usepackage{wrapfig}
\usepackage{tikz}
\usepackage[makeroom]{cancel}
\usepackage{algorithm}
\usepackage{algpseudocode}
\usepackage{lastpage}

\allowdisplaybreaks
\usetikzlibrary{decorations.markings}

\definecolor{mydarkblue}{rgb}{0,0.08,0.45}

\definecolor{tab10_3}{rgb}{0.8392156862745098, 0.15294117647058825, 0.1568627450980392}
\definecolor{tab10_2}{rgb}{0.17254901960784313, 0.6274509803921569, 0.17254901960784313}
\definecolor{tab10_4}{rgb}{0.5803921568627451, 0.403921568627451, 0.7411764705882353}
\definecolor{tab10_5}{rgb}{0.5490196078431373, 0.33725490196078434, 0.29411764705882354}

\hypersetup{ %
    pdftitle={},
    pdfauthor={},
    pdfsubject={},
    pdfkeywords={},
    pdfborder=0 0 0,
    pdfpagemode=UseNone,
    colorlinks=true,
    linkcolor=mydarkblue,
    citecolor=mydarkblue,
    filecolor=mydarkblue,
    urlcolor=mydarkblue,
    pdfview=FitH
}

\usepackage[utf8]{inputenc}
\usepackage{booktabs} 
\usepackage{hyperref}

\usepackage{bbm} 
\usepackage{bm}
\usepackage{verbatim}
\usepackage{float}
\usepackage{color,soul}
\usepackage{enumitem}
\usepackage{mathtools}
\usepackage{hhline}
\usepackage[title]{appendix}
\usepackage[nameinlink]{cleveref} 
\usepackage[font=small,labelfont=bf,
   justification=justified,
   format=plain]{caption}
\usepackage{tikz}
\usepackage{wrapfig,booktabs}
\usepackage{xspace}
\usepackage{cancel}
\usepackage{sidecap}

\graphicspath{ {../figures/} }

\DeclarePairedDelimiterX{\infdivx}[2]{(}{)}{%
  #1\;\delimsize\|\;#2%
}

\newcommand{\one}{\mathbbm{1}}

\usetikzlibrary{shapes.geometric, arrows, bayesnet, calc, positioning}

\newcommand{\md}{\mathrm{d}}

\newcommand{\calF}{\mathcal{F}}

\newcommand{\calH}{\mathcal{H}}

\newcommand{\calT}{\mathcal{T}}

\newcommand{\calP}{\mathcal{P}}

\newcommand{\calN}{\mathcal{N}}
\newcommand{\calO}{\mathcal{O}}

\newcommand{\calL}{\mathcal{L}}

\newcommand{\calX}{\mathcal{X}}

\newcommand{\calD}{\mathcal{D}}

\newcommand{\bH}{\mathbf{H}}

\newcommand{\R}{\mathbb{R}}
\newcommand{\N}{\mathbb{N}}
\newcommand{\Pb}{\mathbb{P}}

\newcommand{\cP}{\mathcal{P}}

\newcommand{\closer}[3]{{\kern-#1ex{#2}\kern-#3ex}}

\newcommand\PSi[2]{{ \left \langle #1 \right \rangle}_{\! #2}}

\mathchardef\mhyphen="2D

\DeclareMathOperator{\E}{\mathbb{E}}

\newcommand\reallywidehat[1]{\arraycolsep=0pt\relax%
\begin{array}{c}
\stretchto{
  \scaleto{
    \scalerel*[\widthof{\ensuremath{#1}}]{\kern-.5pt\bigwedge\kern-.5pt}
    {\rule[-\textheight/2]{1ex}{\textheight}} 
  }{\textheight} %
}{0.5ex}\\           
#1\\                 
\rule{-1ex}{0ex}
\end{array}
}

\renewcommand{\d}{\mathrm{d}}
\newcommand{\ps}[1]{\langle #1 \rangle}
\newcommand{\cF}{\mathcal{F}}

\newcommand{\Hk}{\mathcal{H}}
\newcommand{\id}{\iota}
\def\Id{\mathrm{I}}

\newcommand{\cT}{\mathcal{T}}
\newcommand{\MMD}{\mathrm{MMD}}

\definecolor{azure}{rgb}{0.0, 0.5, 1.0}
\definecolor{airforceblue}{rgb}{0.36, 0.54, 0.66}
\definecolor{darkgreen}{rgb}{0.0, 0.2, 0.13}

\newtheorem{innercustomassumption}{Assumption}

\usepackage{standalone}
\usepackage{tikz}
\usepackage{pgfplots}
\usepackage{pgf}
\usetikzlibrary{calc}
\usetikzlibrary{positioning}
\usetikzlibrary{angles,quotes}
\usetikzlibrary{backgrounds}
\usetikzlibrary{fit}
\usetikzlibrary{arrows}
\usetikzlibrary{arrows.meta}
\usetikzlibrary{shapes.symbols}
\usetikzlibrary{shadings}
\usetikzlibrary{shapes}
\usetikzlibrary{fadings}
\usetikzlibrary{bayesnet}
\usetikzlibrary{matrix}
\usetikzlibrary{plotmarks}
\usetikzlibrary{intersections}
\usetikzlibrary{pgfplots.fillbetween}
\pgfplotsset{compat=1.14}
\usepackage{siunitx}

\usepackage{colortbl}
\definecolor{mediumgray}{gray}{0.7}
\definecolor{lightgray}{gray}{0.85}
\definecolor{lightlightgray}{gray}{0.9}
\definecolor{C1}{HTML}{1F77B4}
\definecolor{C2}{HTML}{FF7F0E}
\definecolor{C3}{HTML}{2CA02C}
\definecolor{C4}{HTML}{D62728}
\definecolor{C5}{HTML}{9467BD}
\colorlet{C1light}{C1!70!white}
\colorlet{C2light}{C2!70!white}
\colorlet{C3light}{C3!70!white}
\colorlet{C4light}{C4!70!white}
\colorlet{C5light}{C5!70!white}
\colorlet{C1lighter}{C1!50!white}
\colorlet{C2lighter}{C2!50!white}
\colorlet{C3lighter}{C3!50!white}
\colorlet{C4lighter}{C4!50!white}
\colorlet{C5lighter}{C5!50!white}
\colorlet{C1vlight}{C1!20!white}
\colorlet{C2vlight}{C2!20!white}
\colorlet{C3vlight}{C3!20!white}
\colorlet{C4vlight}{C4!20!white}
\colorlet{C5vlight}{C5!20!white}
\colorlet{linkcolor}{violet}

\newcommand{\dd}{\mathrm{d}}

\newtheorem{thm}{Theorem}[section]
\newtheorem{cor}{Corollary}[section]
\newtheorem{lem}{Lemma}[section]
\newtheorem{defi}{Definition}
\newtheorem{prop}{Proposition}[section]
\newtheorem{rem}{Remark}[section]

\crefname{enumi}{}{}
\crefname{enumii}{}{}

\usepackage{xr}
\makeatletter

\newcommand*{\addFileDependency}[1]{
\typeout{(#1)}
\@addtofilelist{#1}
\IfFileExists{#1}{}{\typeout{No file #1.}}
}\makeatother

\definecolor{blush}{rgb}{0.87, 0.36, 0.51}

\newcommand{\dmmd}{\operatorname{DrMMD}}
\newcommand{\mmd}{\operatorname{MMD}}
\newcommand{\kale}{\operatorname{KALE}}
\newcommand{\kl}{\operatorname{KL}}
\newcommand{\HS}{\operatorname{HS}}

\newenvironment{talign*}
 {\csname align*\endcsname}
 {\endalign}

\begin{document}

\title{(De)-regularized Maximum Mean Discrepancy Gradient Flow}

\jmlrheading{26}{2025}{1-\pageref{LastPage}}{9/24; Revised
8/25}{10/25}{24-1574}{Zonghao Chen, Aratrika Mustafi, Pierre Glaser, Anna Korba, Arthur Gretton, and Bharath K. Sriperumbudur
}
\ShortHeadings{(De)-regularized Maximum Mean Discrepancy Gradient Flow}{Chen, Mustafi, Glaser, Korba, Gretton and Sriperumbudur}

\author{\name Zonghao Chen 
        \email zonghao.chen.22@ucl.ac.uk \\
       \addr Department of Computer Science, University College London, London, WC1V 6LJ, UK
       \AND
        \name Aratrika Mustafi 
        \email abm6733@psu.edu \\
        \addr Department of Statistics, Pennsylvania State University, University Park, PA, 16802 USA
        \AND
        \name Pierre Glaser 
        \email pierreglaser@msn.com \\
        \addr Gatsby Computational Neuroscience Unit, University College London, London, WC1V 6LJ, UK
        \AND
        \name Anna Korba 
        \email anna.korba@ensae.fr \\
        \addr ENSAE, CREST, Institut Polytechnique de Paris, Palaiseau, France
        \AND
        \name Arthur Gretton 
        \email arthur.gretton@gmail.com \\
        \addr Gatsby Computational Neuroscience Unit, University College London, London, WC1V 6LJ, UK
        \AND
        \name Bharath K. Sriperumbudur 
        \email bks18@psu.edu \\
        \addr Department of Statistics, Pennsylvania State University, University Park, PA, 16802 USA
}
\editor{Quentin Berthet}

\firstpageno{1}

\maketitle

\begin{abstract}%
We introduce a (de)-regularization of the Maximum Mean Discrepancy
(DrMMD) and its Wasserstein gradient flow. 
Existing gradient flows that transport samples from source distribution to target distribution with only target samples, either lack tractable numerical implementation ($f$-divergence flows) or require strong assumptions and modifications, such as noise injection, to ensure convergence (Maximum Mean Discrepancy flows). In contrast, DrMMD flow can simultaneously (i) guarantee near-global convergence for a broad class of targets in both continuous and discrete time, and (ii) be implemented in closed form using only samples.
The former is achieved by leveraging the connection between the DrMMD and the $\chi^2$-divergence, while the latter comes by treating DrMMD as MMD with a de-regularized kernel. 
Our numerical scheme employs an adaptive de-regularization schedule throughout the flow to optimally balance the trade-off between discretization errors and deviations from the $\chi^2$ regime.
The potential application of the DrMMD flow is demonstrated across several numerical experiments, including a large-scale setting of training student/teacher networks.
\end{abstract}

\section{Introduction}
Many applications in computational statistics and machine learning involve approximating 
a probability distribution $\pi$ on $\R^d$ (in terms of samples) when only partial information on $\pi$ is accessible.
For example, in Bayesian inference, $\pi$ is known up to an intractable normalizing constant for complex models. 
The setting of interest in this work is the so-called generative modeling setting \citep{brock2019large,ho2020denoising,song2020score,franceschi2023unifying}  where one assumes access to a set of samples from the target distribution $\pi $, with the goal being to generate new samples from $\pi$.
Recently, a popular framework to perform this task involves solving a minimization problem in
$\cP(\R^d)$, the space of probability distributions over $\R^d$, by choosing the objective function to be a \emph{dissimilarity} function $\calD(\cdot \| \pi)$ (a distance or divergence) between probability distributions that satisfies: $\calD(\nu\| \pi)=0$ if and only if $\nu=\pi$. 
Since only samples from $\pi$ are available, this problem is solved approximately---yet, the approximate minimizers may
converge to $\pi$ as the number of available samples increases. 
In particular, in the space of probability distributions with bounded second moment $\cP_2(\R^d)$, a common approach is to solve this optimization problem by running a (sample-based) approximation of the Wasserstein gradient flow of the functional $\cF = \calD(\cdot \| \pi)$, which defines a path of distributions with steepest descent for $\cF$ with respect to the Wasserstein-2 distance.

In generative modeling, the choice of $\calD(\cdot \| \pi) $ depends on two crucial aspects: First, its flow should admit consistent and preferably tractable numerical implementations using only samples from $\pi$, and second, under reasonable assumptions, it should guarantee convergence of its flow to $\pi$, the unique \emph{global} minimizer. Combined, these two properties ensure that, in the large sample limit, this algorithm generates new samples from the target.
Unfortunately, verifying these two properties simultaneously has proved to be a surprisingly challenging task. For instance, recent approaches based on Maximum Mean Discrepancy (MMD)
\citep{arbel2019maximum,hertrich2023generative}, the sliced-Wasserstein distance \citep{liutkus2019sliced} and the Sinkhorn divergence \citep{genevay2018learning} typically admit consistent finite-sample implementations,
but their global convergence guarantees---when they exist---do not apply to most practical targets $\pi$ of interest. 
To guarantee global convergence, one could instead choose $\calD$ as an $f$-divergence.
For instance, the (reverse) KL divergence 
and $\chi^2$-divergence are geodesically convex \citep[Definition 16.5]{villani2009optimal} when the target is log-concave (i.e. $\pi \propto e^{-V}$ with $V$ convex) \citep{ohta2011displacement}, and hence their flows enjoy better convergence behaviour. 
However, while the population Wasserstein gradient flows of the $\chi^2$ and KL divergences are well-defined \citep{jordan1998variational, chewi2020svgd}, they do not come naturally with consistent and tractable sample-based implementations. 
Multiple approaches propose to solve a surrogate optimization problem with samples at each iteration of the flow \citep{gao2019deep,ansari2020refining,simons2022variational,birrell2022f, gu2022lipschitz,liu2023variational}; however, it remains to be formally established whether these surrogate problems preserve the desirable convergence guarantees of $f$-divergence flows.

In the face of the trade-offs present in the current approaches, a natural question arises: \emph{Does there exist a divergence functional $\calD(\cdot \| \pi)$ whose gradient flow both globally converges, and admits a tractable, consistent sample-based implementation?}
In this work, we take a step towards a positive answer by constructing a ``de-regularized'' variant of the Maximum Mean Discrepancy ($\dmmd$) and its associated Wasserstein gradient flow. 
We prove that the $\dmmd$ gradient flow converges exponentially to the global minimum up to a controllable barrier term for targets $\pi$ that satisfy a Poincar\'{e} inequality, in both continuous and discrete time regimes. To do so, we establish and leverage a connection between the $\dmmd$ and the $\chi^2$ divergence, an $f$-divergence whose gradient flow benefits from strong convergence guarantees.
By alternatively viewing $\dmmd$ as MMD with a regularized kernel, $\dmmd$ flow comes with a consistent and tractable implementation when only samples from the target $\pi$ are available. 
In addition, given the empirical success of using adaptive kernels in MMD-based generative models~\citep{galashov2024deep, li2017mmd,arbel2018gradient}, 
our paper shows theoretically that using adaptive kernels through adaptive regularization indeed improves the convergence of MMD gradient flow.

This paper is organized as follows. \Cref{sec:background} introduces the necessary background on reproducing kernel Hilbert spaces (RKHS), the $\mmd$, $\chi^2$-divergence, and Wasserstein gradient flows. \Cref{sec:chard} introduces $\dmmd$ and shows that $\dmmd$ is a valid probability divergence that
metrizes weak convergence. \Cref{sec:chard_flow} uses $\dmmd$ as the optimization objective to define a Wasserstein gradient flow in $\calP_2(\R^d)$, and analyzes the convergence of $\dmmd$ flow in continuous time.
Sections \ref{sec:chard_particle_flow} and \ref{sec:space_dis} define an implementable $\dmmd$ particle descent scheme with both time and space discretization and analyze its convergence.
\Cref{sec:related_work} discusses other Wasserstein gradient flows related to our $\dmmd$ flow.
\Cref{sec:experiments} shows experiments that confirm our theoretical results. The proofs of all results are provided in Section~\ref{sec:proofs}, with the technical results being relegated to an appendix.


\section{Background}\label{sec:background}
In this section, we present the definitions and notation used throughout the paper.
\subsection{Notations}
Let $\calL^d$ be the Lebesgue measure on $\R^d$. $\cP_2(\R^d)$ denotes the set of all Borel probability measures $\mu$ on $\R^d$ with finite second moment. 
For $\mu \in \cP_2(\R^d)$, $\mu \ll \pi$  denotes that $\mu$ is absolutely continuous with respect to $\pi$. We use $\frac{\d\mu}{\d\pi}$ to denote the Radon-Nikodym derivative. We recall the standard definition of the Kullback-Leibler divergence, $\kl(\mu \| \pi)=\int \log(\frac{\d\mu}{\d\pi})\d\mu$ if $\mu \ll \pi$, $+\infty$ else.

For a continuous mapping $T:\R^d \to \R^d$, $T_{\#} \mu$ denotes the push-forward measure of $\mu$ by $T$.
For any $\pi \in \cP_2(\R^d)$, $L^2(\pi)$ is the Hilbert space of (equivalence class of) functions $f : \R^d \to \R$ such that $\int |f|^2 d\pi < \infty$. 
We denote by $\Vert \cdot \Vert_{L^2(\pi)}$ and $\ps{\cdot,\cdot}_{L^2(\pi)}$ the norm and the inner product of $L^2(\pi)$. 
We 
denote by $C_c^\infty (\R^d)$ the space of infinitely differentiable functions from $\R^d$ to $\R$ with compact support. 
For a vector valued functions $g: \R^d \to \R^p$, we abuse the notation of $L^2(\pi)$ and claim $g \in L^2(\pi)$ if $g_i \in L^2(\pi)$ for all $i = 1, \ldots, p$ along with $ \|g \|^2_{L^2(\pi)} := \sum_{i=1}^p \|g_i \|^2_{L^2(\pi)}$. 

If $f: \R^d \to \R$ is differentiable, we denote by $\nabla f$ the gradient of $f$ and $\bH f$ its Hessian. $f$ is $\alpha$-strongly convex if $\bH f \succeq \alpha \Id$, i.e, $\bH f(x) - \alpha \Id$ is positive semi-definite for any $x$, where $\Id$ is the identity matrix (also denotes an identity operator depending on the context). For a vector valued function $g: \R^d \to \R^d$, if $g_i$ is differentiable for all $i = 1, \cdots, d$, $\nabla \cdot g$ denotes the divergence of $g$. We also denote by $\Delta g$ the Laplacian of $g$, where $\Delta g = \nabla \cdot\nabla g$. 
We use $ \|\cdot\|_F$ to denote the matrix Frobenius norm.
$a \wedge b$ and $a \vee b$ denote the minimum and maximum of $a$ and $b$, respectively.

\subsection{Reproducing kernel Hilbert spaces}
For a positive semi-definite kernel $k : \R^d \times \R^d \to \R$, its corresponding reproducing kernel Hilbert space (RKHS) $\Hk$ is a Hilbert space with inner product $\ps{\cdot,\cdot}_{\Hk}$ and norm $\Vert \cdot \Vert_{\Hk}$~\citep{aronszajn1950theory}, such that (i) $k(x, \cdot) \in \Hk$ for all $x \in \R^d$, and (ii) the reproducing property holds, e.g. for all $f \in \Hk$, $x\in\mathbb{R}^d$,  $f(x)=\ps{f,k(x,\cdot)}_{\Hk}$.
We denote by $\Hk^d$ the Cartesian product RKHS consisting of elements $f=(f_1, \dots, f_d)$ with $f_i\in \Hk$ with inner product $\ps{f,g}_{\Hk^d}=\sum_{i=1}^d\ps{f_i,g_i}_{\Hk}$. 

When $\int k(x,x) \d\pi(x)<\infty$, $\mathcal H$ can be canonically injected into $L^2(\pi)$ using the operator $\iota_{\pi} : \Hk \to L^2(\pi),\,f\mapsto f$ with adjoint $ \iota_\pi^\ast: L^2(\pi) \rightarrow \Hk$ given by
\begin{align*}
   \iota_\pi^\ast f(\cdot) = \int k(x,\cdot)f(x)\d\pi(x).
\end{align*}
The operator $\iota_{\pi}$ and its adjoint can be composed to form an $L^2(\pi)$ endomorphism $\calT_{\pi} \coloneqq\iota_{ \pi} \iota_{ \pi}^\ast$ called the \emph{integral operator}, and a $\mathcal H$ endomorphism  $\Sigma_{\pi} := \iota_{ \pi}^\ast \iota_{ \pi}=\int k(\cdot, x) \otimes k(\cdot, x) d \pi(x)$
(where $(a\otimes b)c \coloneqq \langle b,c \rangle_{\Hk} a$ for $a,b,c\in \Hk$) called the \emph{covariance operator}.
$\calT_{\pi}$ is compact, positive, self-adjoint, and can thus be diagonalized into 
an orthonormal system in $\left\{e_i \right\}_{i \geq 1}$ of $L^2(\pi)$ with associated eigenvalues $\varrho_1 \geq \cdots \varrho_i \geq \cdots \geq 0$.

In this paper, we make the following assumption on our kernel.
\begin{innercustomassumption}\label{assumption:universal}
$k: \R^d \times \R^d \to \R$ is a  continuous and $c_0$-universal kernel, and there exists $K>0$ such that $\sup_{x} k(x,x) \leq K$.
\end{innercustomassumption}

We refer the reader to \cite{carmeli2010vector} for the definition of $c_0$-universal kernel. The implication of \Cref{assumption:universal} is that the RKHS $\calH$ is compactly embedded into $L^2(\pi)$~\citep[Lemma 2.3]{steinwart2012mercer}, and hence $k(x, x^\prime)$ has a absolute, uniform and pointwise convergent Mercer representation~\citep[Corollary 3.5]{steinwart2012mercer},
\begin{align}\label{eq:mercer}
    k(x, x^{\prime})=\sum_{i \geq 1} \varrho_i e_i(x) e_i (x^{\prime}),
\end{align}
for any $x$ and $x^\prime$ in the support of $\pi$. 
Since the kernel is $c_0$-universal, the RKHS $\calH$ is dense in $L^2(\pi)$ for all Borel probability measures $\pi$~\citep[Section 3.1]{sriperumbudur2011universality} and $\{e_i\}_{i \geq 1}$ becomes an orthornormal basis of $L^2(\pi)$~\citep[Theorem 3.1]{steinwart2012mercer}.

The power of the integral operator $\calT_\pi^r$ is defined as $\calT_\pi^r f:= \sum_{i \geq 1} \varrho_i^r \left\langle f, e_i \right\rangle_{L^2(\pi)} e_i, f \in L^2(\pi)$.
For $f \in \operatorname{Ran}(\calT_\pi^r)$, there exists $q \in L^2(\pi)$ such that $f=\calT_\pi^rq$. 
The exponent $r$ quantifies the smoothness of the range space relative to the original RKHS $\calH$ with $0<r<\frac{1}{2}$ (\emph{resp.} $r>\frac{1}{2}$) yields spaces that are less (\emph{resp.} more) smooth than $\mathcal{H}$ with $\operatorname{Ran}(\calT_\pi^{1/2}) $ being isometrically isomorphic to $ \mathcal{H}$~\citep{cucker2007learning, fischer2020sobolev}.

We make an additional assumption---commonly employed in the kernel-based gradient flow literature \citep{glaser2021kale, he2022regularized, korba2020non, arbel2019maximum}---on the regularity of the kernel that will be employed in studying the $\dmmd$ gradient flow.

\begin{innercustomassumption}\label{assumption:bounded_kernel}
$k: \R^d \times \R^d \to \R$ is twice differentiable in the sense of \cite[Definition 4.35]{steinwart2008support}, i.e., for $i, j \in\{1, \cdots, d\}$, both $ \partial_i \partial_{i+d} k$ and $\partial_i \partial_j \partial_{i+d} \partial_{j+d} k$ exist and are continuous. 
There exist constants $K_{1d}, K_{2d} > 0$ such that $\left\|\nabla_1 k(x,\cdot) \right\|_{\calH^d} := \sum_{i=1}^{d} \left\|\partial_i k(x,\cdot) \right\|_\calH \leq \sqrt{ K_{1 d} }$ and $\left\|\bH_1 k(x, \cdot) \right\|_{\calH^{d \times d}} := \sum_{i,j=1}^{d} \left\|\partial_i \partial_j k(x,\cdot) \right\|_\calH  \leq \sqrt{ K_{2 d} }$ for all $x\in\mathbb{R}^d$.
\end{innercustomassumption}

Many kernels satisfy both \Cref{assumption:universal} and \ref{assumption:bounded_kernel}, including the class of bounded, continuous, and translation invariant kernels on $\mathbb{R}^d$ whose Fourier transforms have finite second and fourth moments. This is easy to verify by employing the Fourier transform representation of the RKHS \citep[Theorem 10.12]{wendland2004scattered} and noting that the finiteness of the RKHS norm of $\nabla_1 k(\cdot,x)$ and $\mathbf{H}_1k(\cdot,x)$ for all $x$ corresponds to the existence and finiteness of the second and fourth moments of the Fourier transform of the kernel, respectively. This condition is satisfied by 
the Gaussian kernel, Mat\'{e}rn kernels of order $\nu$ with $\nu + \frac{d}{2} \geq 2$ and the inverse multiquadratic kernel.

\subsection{Maximum mean discrepancy and $\chi^2$-divergence}
The Maximum Mean Discrepancy ($\mmd$)~\citep{gretton2012kernel} between $\mu$ and $\pi$ is defined as the RKHS norm of the difference between the mean embeddings\footnote{Such mean embeddings are well-defined under \Cref{assumption:universal}.} $m_\mu := \int k(x,\cdot)d\mu(x)$ and $m_\pi := \int k(x,\cdot)d\pi(x)$.
\begin{align*}
    \mmd(\mu \| \pi) &\coloneqq \left\| \int k(x,\cdot)d\mu(x) - \int k(x,\cdot) d\pi(x) \right\|_{\mathcal{H}} = \left\|m_\mu -m_\pi \right\|_{\mathcal{H}}.
\end{align*}
The function $m_{\mu}-m_{\pi}$ is often referred to as the ``witness function''. 
When the kernel $k$ is $c_0$-universal, $\mmd(\mu \| \pi)=0$ if and only if $\mu = \pi$, and the $\mmd$ metrizes the weak topology between probability measures~\citep{sriperumbudur2010hilbert, sriperumbudur2016optimal}. 
Given samples $(x_1, \dots, x_n)$ and $(y_1, \dots, y_m)$ from $\mu$ and $\pi$ respectively, the $\mmd$ can be consistently estimated in multiple ways \citep{gretton2012kernel}. For instance, one can compute its ``plug-in'' estimator, e.g. $\mmd(\hat{\mu} \| \hat{\pi})$, where $\widehat{\mu} \coloneqq \frac{1}{n} \sum_{i=1}^{n} \delta_{x_i}$ and  $\widehat{\pi} \coloneqq \frac{1}{n} \sum_{i=1}^{n} \delta_{y_i}$.

The $\chi^2$-divergence --- a member of the family of $f$-divergences \citep{renyi1961measures} --- is defined as
the variance of the Radon-Nikodym derivative $\frac{d \mu}{d \pi}$ under $\pi$:
\begin{align*}
    \chi^2(\mu \| \pi) \coloneqq \int \left( \frac{\mathrm{d} \mu}{d \pi} - 1 \right)^2 \mathrm{d}\pi,
\end{align*}
when $\mu \ll \pi$, and $+\infty$ otherwise. 
The $\chi^2$-divergence has a variational form~\citep{nowozin2016f, nguyen2010estimating}:
\begin{align*}
    \chi^2(\mu \| \pi)=\sup _{h \in \mathcal{M}(\R^d)} \int h d \mu-\int\left(h+\frac{1}{4} h^2\right) d \pi,
\end{align*}
where $\mathcal{M}(\R^d)$ denote the set of all measurable functions from $\R^d$ to $\R$. 
When $\mu \ll \pi$, we prove in \Cref{lem:variation_chi2} that the optimal $h^\ast = \frac{d \mu}{d \pi} - 1 \in L^2(\pi)$ so that it is sufficient to restrict the variational set to $L^2(\pi)$ in contrast to $\mathcal{M}(\R^d)$ for general $f$-divergences. Since in most cases, $\chi^2(\hat{\mu} \| \hat{\pi}) = + \infty$,  the $\chi^2$-divergence does not admit plug-in estimators, and estimating it consistently involves more complicated strategies \citep{nguyen2010estimating}.

\subsection{Wasserstein gradient flows}\label{sec:wass}
Gradient flows are dynamics that use local (e.g. \emph{differential})
information about a given functional in order to minimize it as fast as possible.
Their exact definition depends on the nature of the input space;
in the familiar case of Euclidean space $\R^d$, the gradient flow of a sufficiently regular $F:\R^d \to \R$ given some initial condition $x_0$ is given by the solution $(x_t)_{t \ge 0}$ of $\partial_t x_t = - v(x_t)$,
where $v$ is the Fréchet subdifferential of $F$, a generalization of the notion of derivative to non-smooth functions~\citep{kruger2003frechet}.

Gradient flows can be extended from Euclidean spaces to the more general class of \emph{metric spaces}~\citep{ambrosio2005gradient}.
When the metric space in question is $\mathcal P_2(\R^d)$ endowed with
the Wasserstein-$2$ distance, this gradient flow is called the Wasserstein gradient flow $(\mu_t)_{t \geq 0}$. The Wasserstein gradient flow of $\calF:\calP_2(\R^d) \to \R$ takes the particular form~\citep[Lemma 10.4.1]{ambrosio2005gradient}:
\begin{align}\label{eq:cont_eqn_1}
    \partial_t \mu_t + \nabla \cdot \left(\mu_t v_t\right)=0, 
\end{align}
where $v_t$ is the Fr\'{e}chet subdifferential of $\calF: \calP_2(\R^d) \to \R$ evaluated at $\mu_t$~\citep[Definition 11.1.1]{ambrosio2005gradient}. \eqref{eq:cont_eqn_1} is an instance of the \emph{continuity equation} with velocity field $v_t$: under these dynamics, the mass of $\mu_t$ is transported in the direction $v_t$ that decreases $\calF$ at the fastest rate at each time $t$. 
While \eqref{eq:cont_eqn_1} can be time-discretized in various ways \citep{santambrogio2017euclidean, ambrosio2005gradient}, in this work, we will focus on the \emph{forward Euler} scheme, defined as $\mu_{n+1} \coloneqq (\Id - \gamma v_n)_{\#} \mu_n$ where $\gamma > 0$ is a step size parameter. Such a scheme is also known as the \emph{Wasserstein Gradient Descent} of $\mathcal F$.

Just as gradient descent in Euclidean spaces,  an instrumental property to characterize the convergence of the Wasserstein gradient descent of a functional $\calF$ is given by its \emph{geodesic}
convexity and smoothness. 
Among various ways, one can consider to characterize convexity and smoothness through lower and upper bounds on the Wasserstein Hessian of the functional $\calF$~\citep[Proposition 16.2]{villani2009optimal}.
Given any $\phi \in C_c^\infty(\R^d)$ that defines a constant speed geodesic\footnote{See \Cref{appsec:wass_discussion} for the definition.} starting at $\mu$: $\rho_t = (\Id + t \nabla \phi)_\# \mu$ for $0\leq t \leq 1$, the Wasserstein Hessian of a functional $\mathcal{F} : \calP_2(\R^d) \to \R$ at $\mu$, denoted as $\operatorname{Hess} \mathcal{F}_{\mid \mu}$, is an operator from $L^2(\mu)$ to $L^2(\mu)$:\footnote{Strictly speaking, $\operatorname{Hess} \mathcal{F}_{\mid \mu}$ is an operator over the tangent space $\calT_\mu \calP_2(\R^d)$ which is a subset of $L^2(\mu)$~\citep{villani2009optimal}.}
\begin{align*}
    \langle \operatorname{ Hess } \mathcal{F}_{\mid \mu} \nabla \phi , \nabla \phi \rangle_{L^2(\mu)}= \frac{d^2}{d t^2}\Big|_{t=0} \mathcal{F}(\rho_t).
\end{align*}
A functional $\calF$ is said to be \emph{geodesically} $M$-smooth at $\mu$ if $\langle \operatorname{ Hess } \mathcal{F}_{\mid \mu} \nabla \phi , \nabla \phi \rangle_{L^2(\mu)} \leq M \| \nabla \phi \|_{L^2(\mu)}$, and is said to be \emph{geodesically} $\Lambda$-convex at $\mu$ if $\langle \operatorname{ Hess } \mathcal{F}_{\mid \mu} \nabla \phi , \nabla \phi \rangle_{L^2(\mu)} \geq \Lambda \| \nabla \phi \|_{L^2(\mu)}$. Additionally, $\calF$ is \emph{geodesically semiconvex} if $ -\infty < \Lambda < 0$ and \emph{geodesically strongly convex} if $\Lambda > 0$.
Generally, a functional $\calF$ that is both smooth and strongly convex is preferred, because its Wasserstein gradient descent has an exponential rate of convergence under a small enough step size $\gamma$~\citep[Section 9.3.1]{boyd2004convex}\citep{bonet2024mirror}.
Given some probability measure $\pi \in \calP_2(\R^d)$, the $\mmd$ flow (\textit{resp.} $\chi^2$ flow) is the Wasserstein gradient flow of the functional $\calF_{\mmd}(\cdot) = \mmd(\cdot \| \pi )$ (\textit{resp.} $\calF_{\chi^2}(\cdot) = \chi^2(\cdot \| \pi)$). 
As with the $\mmd$, the $\mmd$ flow has an analytic finite sample implementation and 
may be used to construct generative modeling algorithms ~\citep{hertrich2023generative, hertrich2024wasserstein}.
The Wasserstein Hessian of $\calF_{\mmd}$ for smooth kernels is not positively lower bounded~\citep[Proposition 5]{arbel2019maximum}, however, so $\mmd$ flow only converges up to an unknown barrier~\citep[Theorem 6]{arbel2019maximum}, with global convergence only under a strong (and unverifiable) assumption~\citep[Proposition 7]{arbel2019maximum}.
More recent works ~\citep{boufadene2023global} have demonstrated the global convergence of the $\mmd$ flow when using the \emph{Coulomb kernel}. This kernel is non-smooth, however, which complicates numerical implementations. In contrast, the Wasserstein Hessian of $\calF_{\chi^2}$ is positively lower bounded~\citep{ohta2011displacement} for log-concave targets $\pi$, so $\calF_{\chi^2}$ is \emph{geodesically strongly convex} and $\chi^2$ flow enjoys exponential rate of convergence. 
The exponential convergence of $\chi^2$ flow towards the global minimum in fact holds for a broader class of targets $\pi$ that satisfy a Poincar\'{e} inequality~\citep{chewi2020svgd}.
The $\chi^2$ flow has so far lacked a tractable sample-based implementation, however, so it has not been widely used in practice. 

In the following sections, we will introduce a new Wasserstein gradient flow that combines the computational advantages of the MMD flow with the convergence properties of the $\chi^2$ flow.


\section{(De)-regularized Maximum Mean Discrepancy ($\dmmd$) }\label{sec:chard}
In this section, we introduce a (de)-regularized version of maximum mean discrepancy, or $\dmmd$ in short. 
%
The $\dmmd$ is rooted in a unified representation of
the MMD and the $\chi^2$-divergence, given in the following proposition, which is proved in \Cref{appsec:mmd_regularized_chi2}.
\begin{prop}[$\mmd$ and $\chi^2$-divergence]\label{prop:mmd_chi2}
Suppose
$\frac{d \mu}{d \pi} -1 \in L^2(\pi)$ for $\mu, \pi \in \calP_2(\R^d)$. Then 
\begin{align*}
    \mmd^2(\mu \| \pi) = \left \|  \mathcal  T_{\pi}^{1 / 2}\left(\frac{d \mu}{d \pi}-1\right) \right \|^2_{L^2(\pi)} & \text{  and  }\,\,
    \chi^2(\mu \| \pi) = \left\| \Id \left(\frac{d \mu}{d \pi}-1\right) \right\|^2_{L^2(\pi)}. 
\end{align*}
\end{prop}
\begin{rem}\label{rem:mmd-chi2}
The $ \chi^{2}$ identity follows from the definition but is provided for comparison purposes. 
Together, these identities express both the MMD and the $\chi^2$-divergence as functionals of the (centered) density ratio $\frac{d \mu}{d \pi} -1$.
While the $\chi^2$-divergence directly
computes the $ L^2(\pi)$ norm of the centred ratio, the $\mmd^2$ first computes the image by the operator $ \mathcal  T_{\pi}^{1 / 2} $ before taking the $ L^2(\pi) $ norm. 
The smoothing effect of the compact operator $ \mathcal  T_{\pi}^{1 / 2}$---note that $\mathcal{T}_\pi$ is compact if $k$ is bounded as assumed in \Cref{assumption:universal}---has both positive and negative consequences: 
$\calF_{\mmd}(\cdot) = \mmd^2(\cdot \| \pi)$ admits finite sample estimators but is not geodesically convex, making the first-order optimization of MMD objective (as done in generative modeling) challenging~\citep{arbel2019maximum}. 
In contrast, $ \calF_{\chi^2}(\cdot) = \chi^2(\cdot \| \pi) $ is geodesically convex for log-concave targets $\pi$~\citep{ohta2011displacement} but is hard to estimate with samples.
\end{rem}

With these facts in mind, we introduce a divergence whose purpose is to combine the beneficial properties of both the $\chi^2$-divergence and $ \operatorname{MMD} $. 
To do so, this divergence computes the $ L^2(\pi) $ norm of the image of $ \frac{ d \mu }{  d \pi } - 1 $ by an alternative operator which interpolates between $ \Id $ and $ \mathcal  T_{\pi}^{1 / 2} $. We set this operator to be $ ((\mathcal  T_{\pi} + \lambda \Id)^{-1} \mathcal  T_{\pi})^{1 / 2} $, where $ \lambda > 0 $ is a regularization parameter. 
The operator $ ((\mathcal  T_{\pi} + \lambda \Id)^{-1} \mathcal  T_{\pi})^{1 / 2} $ can be seen as a (de)-regularization of the operator $ \mathcal  T_{\pi}^{1 / 2} $ used by the MMD---- a similar idea has been used in kernel Fisher discriminant
analysis~\citep{mika99fisher}, goodness-of-fit testing~\citep{balasubramanian2017optimality,hagrass2023spectralgof}, and two-sample testing~\citep{eric2007testing, hagrass2022spectral}. We call the resulting divergence the (De)-regularized Maximum Mean Discrepancy ($\dmmd$).

\begin{defi}[$\dmmd$]\label{def:chard}
Suppose $\frac{d \mu}{d \pi} -1 \in L^2(\pi)$ where $ \mu, \pi \in \calP_2(\R^d) $. 
Then the (de)-regularized maximum mean discrepancy ($ \dmmd $) between $ \mu, \pi \in \calP_2(\R^d) $ is defined as
\begin{align}\label{eq:dmmd_with_inverse_operator}
    \dmmd (\mu || \pi) &= ( 1+\lambda) \left\|\left((\mathcal  T_{\pi} + \lambda \Id)^{-1}\mathcal  T_{\pi} \right )^{1 / 2} \left(\frac{d \mu}{d \pi}-1\right) \right\|_{L^2(\pi)}^2,
\end{align}
where $\lambda>0$.
\end{defi}
While all three operators,
$ \Id $, $ \mathcal  T_{\pi} $ and $ (\mathcal  T_{\pi} + \lambda \Id)^{-1}\mathcal  T_{\pi}$ are diagonalizable in the same eigenbasis of $ L^{2}(\pi)$, the key difference between them lies in the behavior of their eigenvalues. The identity operator  $\Id$ has all eigenvalues $1$, the integral operator $\mathcal{T}_\pi$ has eigenvalues $(\varrho_i)_{i \geq 1}$ which decay to zero as $i\rightarrow\infty$ and the (de)-regularized integral operator $(\mathcal{T}_\pi+\lambda \Id)^{-1}\mathcal{T}_\pi$ has eigenvalues $\left(\varrho_i/(\varrho_i+\lambda)\right)_{i \geq 1}$ which either decay to zero or converge to 1 depending on the choice of $\lambda$ as $i\rightarrow\infty$. 
$ (\mathcal  T_{\pi} + \lambda \Id)^{-1}  \mathcal  T_{\pi} $ is known in the statistical estimation literature as Tikhonov regularization; alternative definitions of $\dmmd$ could be obtained by using other regularizing operators, such as Showalter regularization in \cite{engl1996regularization}. In this paper, we primarily focus on Tikhonov regularization and leave other types of regularization for future work. 


One of the stated purposes of the $\dmmd$ is to retain the computational benefits of the MMD, which are crucial for its use in particle algorithms for generative modeling. 
To this end,
we provide an alternative representation of $\dmmd$ which does not involve the density ratio $ \frac{ d \mu }{ d \pi } $ directly, but only kernel expectations.

\begin{prop}[Density ratio--free and variational formulations]\label{prop:drmmd_representation_no_ratio}
$\dmmd$ can be alternately represented as 
\begin{align} 
\dmmd(\mu || \pi) &= ( 1+\lambda) \left\|\left(\Sigma_\pi + \lambda \Id\right)^{-\frac{1}{2}}\left(m_\mu - m_\pi \right) \right\|_{\mathcal{H}}^2 \label{eq:drmmd_representation_no_ratio} \\
&= (1+\lambda) \sup _{h \in \mathcal{H}}\left\{\int h \mathrm{~d} \mu-\int\left(\frac{h^2}{4}+h\right) \mathrm{d} \pi-\frac{\lambda}{4}\|h\|_{\mathcal{H}}^2\right\}
\label{eq:drmmd_variational}
\end{align}
    with $h_{\mu, \pi}^\ast = 2\left(\Sigma_{\pi} + \lambda \Id\right)^{-1} (m_{\mu} - m_{\pi})$ being the witness function. 
\end{prop}
The proof is in \Cref{appsec:proof_chard_also}.
The density ratio-free representation \eqref{eq:drmmd_representation_no_ratio} contains three expectations under $ \mu $ and $ \pi$: the mean embeddings $ m_{\mu} $ and $ m_{\pi} $, and the covariance operator $\Sigma_{\pi}$. 
Given samples $ \{ x_i \}_{i=1}^{M} \sim \mu $ and $ \{ y_i \}_{i=1}^{N} \sim \pi $, we can construct a plug-in finite sample estimator in \eqref{appeq:dmmd_samples} by replacing $ \mu $ and $ \pi $ with their empirical counterparts: $ \hat{\mu} = \frac{1}{M} \sum_{i=1}^{M} \delta_{x_i} $ and $ \hat{\pi} = \frac{1}{N} \sum_{i=1}^{N} \delta_{y_i} $. 

The density ratio-free representation \eqref{eq:drmmd_representation_no_ratio} frames $\dmmd$ as acting on the difference of $m_\mu $ and $m_\pi$ similarly to MMD. 
In fact, up to a multiplicative factor $(1+\lambda)$, $\dmmd$ is $\mmd$ computed with respect to another kernel $ \tilde{k} $ defined as 
\begin{align}\label{eq:tilde_k}
    \tilde{k}(x, x^\prime) = \left\langle \left(\Sigma_\pi + \lambda \Id\right)^{-\frac{1}{2} } k(x,\cdot),  \left(\Sigma_\pi + \lambda \Id\right)^{-\frac{1}{2} } k(x^\prime, \cdot) \right\rangle_{\calH}.
\end{align}
The derivations are provided in \Cref{appsec:drmmd_is_mmd}. The kernel $\tilde{k}$ is symmetric and positive semi-definite by construction, and its associated reproducing kernel Hilbert space is $\tilde{\calH}$. The density ratio-free representation of $\dmmd$ in \Cref{prop:drmmd_representation_no_ratio} is already known in~\citep{balasubramanian2017optimality, eric2007testing,hagrass2022spectral,hagrass2023spectralgof} in the context of non-parametric hypothesis testing.

\subsection{Properties of $\dmmd$}
In this section, we establish various properties of $\dmmd$. 
As discussed earlier, $\dmmd$ is constructed to interpolate between $\chi^{2}$-divergence and MMD to exploit the advantages associated with each. The following result formalizes the interpolation property.
\begin{prop}[Interpolation property]\label{prop:chard_interpolation}
Let $\mu, \pi \in \mathcal{P}_2(\mathbb{R}^d)$. 
If \Cref{assumption:universal} holds and $\frac{d \mu}{d \pi} -1 \in L^2(\pi)$, then
\begin{align*}
    \lim _{\lambda \rightarrow 0} \dmmd(\mu \| \pi) = \chi^2(\mu \| \pi), \quad\text{and}\quad \lim _{\lambda \rightarrow \infty} \dmmd(\mu \| \pi) =\MMD^2(\mu \| \pi). 
\end{align*}
\end{prop}
\Cref{prop:chard_interpolation}, whose proof can be found in \Cref{appsec:proof_chard_interpolation}, shows that $\dmmd$ asymptotically becomes a probability divergence in the small and large $ \lambda $ regimes.
We seek to use $\dmmd$ as a minimizing objective in generative modeling algorithms; however, we need to ensure that $\dmmd$ is a probability divergence for \emph{any} fixed value of $\lambda$. This result holds, as shown next.

\begin{prop}[$\dmmd$ is a probability divergence]\label{prop:topology_chard}
Under \Cref{assumption:universal}, 
for any $ \lambda \in \left( 0, \infty \right ) $, $\dmmd$ is a probability divergence, i.e., $\dmmd (\mu \| \pi) \geq  0 $, with  equality iff $ \mu = \pi $.
Moreover, $\dmmd$ metrizes the weak topology between probability measures, i.e., $ \dmmd (\mu_n \| \pi) \to 0 $ iff $\mu_n $ converges weakly to $\pi \in \cP_2(\R^d)$  as $n\rightarrow\infty$.
\end{prop} 
As MMD with $c_0$-universal kernels metrizes the weak convergence of distributions~\citep{sriperumbudur2016optimal, simon2023metrizing}, \Cref{prop:topology_chard}, whose proof can be found in \Cref{appsec:proof_of_topology_chard}, shows that for any $ \lambda > 0 $, $\dmmd$ is ``MMD--like'' topologically speaking, and is different from the $\chi^2$-divergence which induces a strong topology~\citep{agrawal2021optimal}. 
\begin{rem}
$\dmmd$ is a specific case of a so-called ``Moreau envelopes of  $f$-divergences in reproducing kernel Hilbert spaces'' introduced in \cite{neumayer2024wasserstein}, when the $f$-divergence is taken as the $\chi^2$-divergence. This connection is uncovered by the variational formulation of $\dmmd$ in \eqref{eq:drmmd_variational}. 
In contrast to general $f$-divergences, the Moreau envelope of the $\chi^2$-divergence enjoys a closed-form expression, as we highlight in this paper with various analytical formulas for $\dmmd$. 
The interpolation property (\Cref{prop:chard_interpolation}) and the metrization of weak convergence (\Cref{prop:topology_chard}) are proved concurrently in Corollaries 12 and 13 of \cite{neumayer2024wasserstein}, relying on formulation via Moreau envelopes.
In our case, we use direct computations thanks to the closed form of $\dmmd$.\footnote{We would like to clarify that \cite{neumayer2024wasserstein} appeared on arxiv when our paper was already under review at ICML 2024.}
\end{rem}

\section{Wasserstein Gradient Flow of $\dmmd$}\label{sec:chard_flow}
Having introduced the $\dmmd$ in the previous section, we now construct and analyze its \emph{Wasserstein Gradient Flow} (WGF). 
As discussed in \Cref{sec:background}, WGFs define dynamics $ (\mu_t)_{t \geq 0} $ in Wasserstein-2 space that minimize a given functional $ \mathcal  F $ by transporting $ \mu_t $ in the direction of steepest descent, given by the \emph{Fréchet subdifferential} of $ \mathcal  F $ evaluated at $ \mu_t $. 
Given some target distribution $ \pi $ from which we wish to sample, the WGF of $\calF_{\dmmd}(\cdot) = \dmmd(\cdot \| \pi)$, called $\dmmd$ flow, has the potential to form the basis of a generative modeling algorithm, since $ \mu_t $ progressively minimizes its distance (in the DrMMD sense) to the target $ \pi$. 

To fulfill this potential, two additional ingredients are necessary. The first is to formally establish that $ \mu_t $ reaches the global minimizer $ \pi $, and 
the second is to design a tractable finite-sample algorithm that inherits the convergence properties of the original $\dmmd$ flow. 
We defer the second point to \Cref{sec:chard_particle_flow} and focus in this section on showing how $\dmmd$ benefits from its interpolation towards $\chi^2$-divergence such that the $\dmmd$ flow achieves near-global convergence for a large class of target distributions.

\subsection{$\dmmd$ flow: Definition, existence, and uniqueness}
To prove that the $\dmmd$ flow is well-defined and admits solutions, the key is to show that the $\calF_{\dmmd}$ admits Fréchet subdifferentials, as formalized in the following proposition.
\begin{prop}[$\dmmd$ gradient flow]\label{prop:chard_gradient_flow}
Let $\lambda>0$, and $\mu_0, \pi \in \mathcal{P}_2\left(\R^d\right)$. Under \Cref{assumption:universal} and \ref{assumption:bounded_kernel}, the functional $\calF_{\dmmd}$ admits Fréchet subdifferential of the form $(1 + \lambda) \nabla h_{\mu_t, \pi}^\ast$, where $ h_{\mu_t, \pi}^\ast $ is the witness function defined in \Cref{prop:drmmd_representation_no_ratio}. Consequently, the $\dmmd$ flow is well-defined and is the solution to the following equation
\begin{align}\label{eq:cont_eqn}
   \partial_t \mu_t - \nabla \cdot \left(\mu_t(1+\lambda) \nabla  h_{\mu_t, \pi}^\ast \right)=0.
\end{align}
In addition, the $\dmmd$ flow starting at $\mu_0$ is unique because $\calF_{\dmmd}$ is 
semiconvex, i.e., for any $\phi \in C_c^\infty(\R^d)$ and $\mu \in \calP_2(\R^d)$,
\begin{align}\label{eq:dmmd_semi_convex}
    \left| \left\langle \operatorname{Hess} {\calF_{\dmmd} }_{\mid \mu} \nabla \phi, \nabla \phi \right\rangle_{L^2(\mu)} \right| \leq 2(1+\lambda) \frac{ 2 \sqrt{ K K_{2d}} + K_{1d} }{\lambda} \|\nabla \phi \|_{L^2(\mu)}^2 .
\end{align}
\end{prop}
The proof is in \Cref{appsec:proof_chard_gradient_flow}.
By recalling the discussion of Wasserstein Hessian in \Cref{sec:wass}, \eqref{eq:dmmd_semi_convex} indicates that  $\calF_{\dmmd}$ is both \emph{geodesically smooth} and \emph{geodesically semiconvex}, which is expected because $\dmmd$ is equivalent to $\mmd$ with a regularized kernel $\tilde{k}$ defined in \eqref{eq:tilde_k} and $\calF_{\mmd}$ is both \emph{geodesically smooth} and \emph{geodesically semiconvex}~\citep[Proposition 5]{arbel2019maximum}.
\Cref{prop:chard_gradient_flow} is proved concurrently in Corollaries 14 and 20 of \cite{neumayer2024wasserstein} relying on formulation via Moreau envelopes, while our proof uses the closed-form expression for $\dmmd$.

\subsection{Near-global convergence of $\dmmd$ flow}
Having defined the $\dmmd$ flow, we are now concerned with its convergence to the target $ \pi $. 
Since $\dmmd$ is constructed to interpolate between the $\mmd$ and the $ \chi^2 $-divergence, $\dmmd$ flow is expected to recover the convergence properties of the $\chi^2$ flow.
With this goal in mind, we first study the Wasserstein Hessian of $\dmmd$ and prove that it becomes \emph{asymptotically} positive as $ \lambda \to 0$ for strongly log-concave targets $\pi$. 
Next, to obtain a \emph{non-asymptotic} convergence rate, we take another route and show that the $\dmmd$ flow converges to $ \pi $ exponentially fast in KL divergence up to a barrier term that vanishes in the small $ \lambda $ regime, provided that $ \pi $ satisfies a Poincar\'{e} inequality.

\subsubsection{Near-Geodesic Convexity of $\calF_{\dmmd}$}\label{sec:near-geodesic-convexity}
One popular approach to proving that the $\dmmd$ flow $(\mu_t)_{t \geq 0}$ converges to the target $\pi$ in terms of $\dmmd$ is to show that $\calF_{\dmmd}$ is \emph{geodesically convex}~\citep[Theorem 4.0.4]{ambrosio2005gradient}, or, equivalently in our definition in \Cref{sec:wass}, its Wasserstein Hessian is positive definite. 
In the next proposition, we show that the Wasserstein Hessian of $\calF_{\dmmd}$ is indeed positive definite for small enough $\lambda$; however, as we will see below (\cref{lambda_difficult}), the form of the result will not allow us to show convergence besides in the limit $\lambda \rightarrow 0$, which leads us to take a different approach in subsequent sections.
\begin{prop}[Near-geodesic convexity of $\calF_{\dmmd}$] \label{prop:geometric_prop}
Let $\mu, \pi \in \calP_2(\R^d)$, $\mu, \pi \ll \calL^d$ and $\phi \in C_c^\infty(\R^d)$. 
Under \Cref{assumption:universal} and \ref{assumption:bounded_kernel}, 
let $\pi$ be $\alpha$ strongly log-concave, i.e., $\pi \propto \exp(-V)$, $\bH V \succeq \alpha \Id$, and assume additionally that $x \mapsto \bH V(x)$ is continuous. Then for all $ \mu $ such that $ x \mapsto \nabla \log \mu(x)$ is continuous and $\frac{d \mu}{d \pi} -1 \in \mathcal  H$, 
\begin{align}\label{eq:hessian_lower_bound}
    \left\langle \operatorname{Hess} {\calF_{\dmmd}}_{\mid \mu} \nabla \phi, \nabla \phi \right\rangle_{L^2(\mu)} \geq \alpha (1 + \lambda) \int \frac{d\mu}{d \pi}(x) \|\nabla \phi(x) \|^2  d \mu(x) - R(\lambda, \mu, \nabla \phi), 
\end{align}
where $\lim_{\lambda \to 0} R(\lambda, \mu, \nabla \phi) = 0$. 
\end{prop}
The proof can be found in \Cref{appsec:proof_geometric}. 
To obtain this result, we relate the Wasserstein Hessian of $\calF_{\dmmd}$ with that of $\calF_{\chi^2}$. 
When $ \frac{  d \mu }{ d \pi } - 1 \in \mathcal  H $, they coincide asymptotically as $ \lambda \to 0$, showing that the interpolation properties of  $\dmmd$ to the $\chi^2$-divergence hold at the level of Wasserstein derivatives.
Together with the fact that the Wasserstein Hessian of $\calF_{\chi^2}$ is positive definite for $ \alpha $-strongly log-concave $ \pi $~\citep{ohta2011displacement}, we obtain the lower bound in \eqref{eq:hessian_lower_bound}.
It is noteworthy that although $\dmmd$ can be viewed as squared $\mmd$ with a regularized kernel $\tilde{k}$, the near-geodesic convexity in \Cref{prop:geometric_prop} is not observed for standard $\mmd$ with a fixed kernel $k$ because the latter does not interpolate towards $\chi^2$-divergence.

\begin{rem}[Geodesic convexity/smoothness trade-off in $\calF_{\dmmd}$]\label{rem:convex_smooth_trade_off} 
The geodesic \newline smoothness and near-convexity of the functional $\calF_{\dmmd}$ are characterized by \eqref{eq:dmmd_semi_convex} and \eqref{eq:hessian_lower_bound} respectively via upper and lower bounds on the  Wasserstein Hessian of $\calF_{\dmmd}$. However, \eqref{eq:dmmd_semi_convex} and \eqref{eq:hessian_lower_bound}
impose \emph{contradictory} conditions on the (de)-regularization parameter $\lambda$: \eqref{eq:dmmd_semi_convex} indicates that $\dmmd$ is smoother if $\lambda$ is larger while \eqref{eq:hessian_lower_bound} indicates that $\dmmd$ is more convex if $\lambda$ is small enough.
Consequently, there is a trade-off between the geodesic convexity and smoothness of $\calF_{\dmmd}$, which will play an important role in \Cref{sec:chard_particle_flow}.
\end{rem}
\begin{rem}\label{lambda_difficult}
\Cref{prop:geometric_prop} shows that for fixed $\mu$ and $ \nabla \phi $, there exists $\lambda$ small enough yet positive such that 
$\left\langle \operatorname{Hess} {\calF_{\dmmd}}_{\mid \mu} \nabla \phi, \nabla \phi \right\rangle_{L^2(\mu)} > 0$ at $\mu$.  
The remainder term $R(\lambda, \mu, \nabla \phi)$ is only controlled in the limit as $\lambda \rightarrow 0$, however, which complicates the use of \Cref{prop:geometric_prop} to show global convergence of the $\dmmd$ flow. In the next section, we employ a different set of techniques that rely on the Poincar\'{e} condition on $\pi$, a condition which, as we show, will ensure a sufficient dissipation of KL divergence along the $\dmmd$ flow to obtain \emph{non-asymptotic} near-global convergence. 
\end{rem}

\subsubsection{Near-Global Convergence of $\dmmd$ flow via Poincar\'{e} inequality}
Even when a functional $ \mathcal  F $ is not geodesically convex, convergence guarantees for its Wasserstein gradient flow $(\mu_t)_{t \geq 0}$ can still be obtained if the target $\pi$ satisfies certain functional inequalities. 
Consider the $\chi^2$ flow $(\nu_t)_{t \geq 0}$, for example: if $ \pi $ satisfies the Poincar\'{e} inequality, then $(\nu_t)_{t \geq 0}$ converges exponentially fast to $\pi$ in terms of KL divergence~\citep[Theorem 1]{chewi2020svgd}, i.e., 
\begin{align}\label{eq:chi2_converge}
    \kl(\nu_T \| \pi) \leq \exp \left( -\frac{2T}{C_P}\right) \kl(\nu_0 \| \pi) .
\end{align}
Recall that $ \pi $ satisfies a Poincar\'{e} inequality~\citep[Definition 1]{pillaud2020statistical} if
for all functions $f \colon \R^d \to \R$ such that $f , \nabla f \in L^2(\pi)$, there exists a constant $C_P > 0$ such that
\begin{equation}\label{eqn:poincare}
\int f(x)^2 d \pi(x) - \left( \int f(x) d \pi(x) \right)^2 \leq C_P\|\nabla f\|^2_{L^2(\pi)}.
\end{equation}
The smallest constant $C_P$ for which \eqref{eqn:poincare} holds is called the Poincar\'{e} constant.
The Poincar\'{e} inequality is widely used for studying the convergence of  Langevin diffusions~\citep{chewi2024analysis} and $\chi^2$ flow~\citep{chewi2020svgd,garcia2020bayesian}.
The Poincar\'{e} condition is implied by the strong log-concavity of $\pi$, and is weaker than strong log-concavity because it also allows for nonconvex potentials. 
The set of probability measures satisfying a Poincar\'{e} inequality includes distributions with sub-gaussian tails or with exponential tails. This set is also closed under bounded perturbations and finite mixtures (see \cite{vempala2019rapid} and \cite{chewi2024analysis} for a more detailed discussion).

Given the interpolation property of DrMMD to $ \chi^2 $-divergence, ~\eqref{eq:chi2_converge} suggests investigating the convergence of the $\dmmd$ flow $(\mu_t)_{t \geq 0}$ in KL divergence under a Poincar\'{e} inequality. 
To this end,  we first derive an upper bound on $\kl(\mu_t \| \pi)$ along the $\dmmd$ flow.
\begin{thm}[KL control of the $\dmmd$ flow] \label{thm:continuous_time_convergence}
Suppose $k$ satisfies Assumptions \ref{assumption:universal} and \ref{assumption:bounded_kernel}, and the target $\pi$ and $\dmmd$ gradient flow $\left(\mu_t\right)_{t \geq 0}$ satisfy the following conditions:
\begin{enumerate}[itemsep=5.0pt,topsep=5pt,leftmargin=*]
\item $\pi$ satisfies a Poincar\'{e} inequality with constant $C_P$.
\item $\mu_t, \pi \ll \calL^d$.
\item $\frac{ d\mu_t}{d \pi} - 1 \in \operatorname{Ran} ( \mathcal{T}_\pi^{r} )$ with $r > 0$, i.e., there exists $ q_t \in L^2(\pi) $ such that $\frac{d\mu_t}{d\pi} - 1 = \calT_\pi^{r} q_t$.
\item $\left\| \nabla \left(\log \pi\right)^\top \nabla \left( \frac{d \mu_t}{d\pi} \right) \right\|_{L^2(\pi)} \leq \mathcal{J}_t$ and $\left\| \Delta \left( \frac{d \mu_t}{d\pi} \right) \right\|_{L^2(\pi)} \leq \mathcal{I}_t $.
\item For all $i = 1, \ldots, d$, $\lim\limits_{x\to \infty} \left( h_{\mu_t, \pi}^\ast(x) - 2 \frac{d \mu_t}{ d \pi} (x) \right) \left( \partial_i \frac{d \mu_t}{ d \pi} (x) \right) \pi(x) \to 0$.
\end{enumerate}
Then, for any $T\ge 0$,
\begin{align}\label{eq:KL_continuous}
    \kl(\mu_T \| \pi) & \leq \exp\left( -\frac{ 2 (1+\lambda)}{ C_P} T\right) \kl(\mu_0 \| \pi) \nonumber\\
    &+ 4 (1+\lambda) \lambda^r \int_0^T \exp\left( -\frac{ 2(1+\lambda)}{C_P} (T-t) \right)\left\| q_t \right\|_{L^2(\pi)} (\mathcal{J}_t + \mathcal{I}_t)  dt.
\end{align}
\end{thm}
The proof, which can be found in \Cref{appsec:proof_continuous_time_convergence}, leverages the fact that the $\dmmd$ can approximate not only the $\chi^2$-divergence, but also its \emph{Wasserstein gradient}. The $\dmmd$'s approximation properties can be combined with functional inequalities to obtain an upper-bound for the continuous-time dissipation of KL divergence along the flow, given by:
\begin{align}\label{eq:kl_d_dt_final_main}
    \frac{d}{dt} \kl(\mu_t \| \pi) \leq  - \frac{ 2 (1+\lambda)}{ C_P} \kl(\mu_t \| \pi) + \underbrace{ 4 (1+\lambda)\lambda^{r} \left\| q_t \right\|_{L^2(\pi)} \left( \mathcal{J}_t +\mathcal{I}_t \right)}_{ \text{Approximation error}},
\end{align}
from which \Cref{thm:continuous_time_convergence} follows upon applying the Growall's lemma~\citep{gronwall1919note}.
The first term is strictly negative, while the second term is an approximation error term arising from $\dmmd$ not perfectly matching the $\chi^2$-divergence for $\lambda > 0$.
When $\lambda = 0$, \Cref{thm:continuous_time_convergence} recovers the exponential decay of KL divergence along $\chi^2$ flow in \eqref{eq:chi2_converge}.
\begin{rem}\label{rem:continous_time}
(i) The second condition assumes that $\mu_t$ and $\pi$ have densities. \vspace{1mm}\\
(ii) The third condition that $\frac{ d\mu_t}{d \pi} - 1 \in \operatorname{Ran} ( \mathcal{T}_\pi^{r} )$ is a regularity condition on the density ratio so that it can be well approximated by the witness function $h_{\mu_t, \pi}^\ast$. 
This assumption is known as the range assumption in the literature of kernel ridge regression~\citep{cucker2007learning, fischer2020sobolev}. 
We posit that this assumption can be relaxed to $\frac{ d\mu_t}{d \pi} - 1 \in L^2(\pi)$ as in \Cref{prop:chard_interpolation} if only asymptotic convergence is needed with no explicit rate. \vspace{1mm} \\
(iii) The fourth condition is another regularity condition on the density ratio along the flow. This condition is automatically satisfied under a stronger range condition ($r = \frac{1}{2}$) in the third condition, i.e., $\frac{ d\mu_t}{d \pi} - 1 \in \calH$, along with a moment condition on the score function $\nabla  \log \pi$. 
To see this, notice that we can further write (derivations are provided in \Cref{appsec:proof_continuous_time_convergence})
\begin{align}\label{eq:fourth_condition}
\begin{aligned}
    \left\| \nabla \left( \log \pi \right)^\top \nabla \left( \frac{d \mu_t}{d \pi} \right) \right\|_{L^2(\pi)} &\leq \sqrt{K_{1d}} \| q_t \|_{L^2(\pi)} \left\| \nabla  \log \pi \right\|_{L^2(\pi)} \\
     \left\| \Delta \left(\frac{d \mu_t}{d \pi} \right)  \right\|_{L^2(\pi)} & \leq \sqrt{K_{2d} } \| q_t \|_{L^2(\pi)} .
\end{aligned}
\end{align}
As an illustration, we examine the DrMMD flow under a Gaussian approximation in \Cref{sec:example}, following \citet{lambert2022variational,liu2024towards}, with a Gaussian kernel and Gaussian target $\pi$, for which explicit upper bounds $\mathcal{I}_t$ and $\mathcal{J}_t$ can be derived. \\
(iv) The fifth condition is a boundary condition that allows integration by parts equality in the proof. We highlight that many works (e.g., Theorem 1 of \cite{he2022regularized}, Theorem 2 of \cite{nitanda2022convex}, Lemma 6 of \cite{vempala2019rapid}) on Wasserstein gradient flow apply integration by parts without explicitly stating this condition. 
\end{rem}
Additionally, if $\| q_t \|_{L^2(\pi)} \leq Q$,\,$\mathcal{J}_t \leq \mathcal{J},\, \mathcal{I}_t \leq \mathcal{I}$  for all $0 \leq t \leq T$, then \Cref{thm:continuous_time_convergence} will ensure KL convergence of the $\dmmd$ flow up to a controllable barrier term, e.g. \emph{near global convergence}.
\begin{cor}[Near global convergence of the $\dmmd$ flow]\label{cor:continuous_time_convergence}
In addition to the assumptions of \Cref{thm:continuous_time_convergence}, if $\left\| q_t \right\|_{L^2(\pi)} \leq Q$,  $\mathcal{J}_t \leq \mathcal{J}, \mathcal{I}_t \leq \mathcal{I}$ for all $0 \leq t \leq T$, where $Q, \mathcal{J},$ and $\mathcal{I}$ are universal constants independent of $\lambda$, then for any $T\ge 0$,
\begin{align*}
    \kl(\mu_T \| \pi) \leq \exp\left( -\frac{ 2 (1+\lambda)}{C_P} T\right) \kl(\mu_0 \| \pi) + 2 \lambda^{r}  C_P Q \left( \mathcal{J} + \mathcal{I} \right) .
\end{align*}
\end{cor}
The proof of \Cref{cor:continuous_time_convergence} follows directly from upper-bounding the second term of \eqref{eq:KL_continuous} with universal constants $\mathcal{Q}$, $\mathcal{J}$, and $\mathcal{I}$, and using the closed-form expression of the resulting integral. 
\Cref{cor:continuous_time_convergence} provides a condition under which the $\dmmd$ flow will exhibit an exponential rate of convergence (linear convergence) in terms of KL divergence up to an extra approximation error term which vanishes as $\lambda \to 0$. 
If $\left\| q_t \right\|_{L^2(\pi)} \leq Q, \mathcal{J}_t \leq \mathcal{J}, \mathcal{I}_t \leq \mathcal{I}$ for all $0 \leq t \leq T$ as shown in \Cref{cor:continuous_time_convergence}, the approximation error is of explicit order $\calO(\lambda^r)$. Therefore, in the continuous time regime, to have a smaller approximation error, it is beneficial to use a small (de)-regularization parameter $\lambda$, so that  $\dmmd$ flow operates closer to the regime of $\chi^2$ flow.
However, as we will see in the next section, when it comes to time-discretized $\dmmd$ flow, i.e., $\dmmd$ gradient descent, there is a trade-off between the approximation error and the time discretization error such that the selection of $\lambda$ would require more careful analysis to strike a good balance.
Finally, a smaller Poincar\'{e} constant $C_P$ results in both a faster rate of convergence and a smaller barrier.

Previously, \cite{arbel2019maximum} established (sublinear) global convergence of the $\mmd$ flow in terms of $\mmd$ distance by assuming that a Lojasiewicz inequality (or a variant of it if additionally performing noise injection, see \citealp[Proposition 8]{arbel2018gradient}) holds along the flow. 
Our result thus complements that of \cite{{arbel2019maximum}} by showing that MMD-type functionals can achieve near-global convergence for targets satisfying a Poincar\'{e} inequality regardless of whether such inequalities hold, by studying their behavior in the $f$-divergence interpolation regime.


\begin{rem}\label{rem:one}
    Since $\dmmd$ is asymmetric in its arguments, the reader may wonder why we focus on the gradient flow of $\dmmd(\cdot||\pi)$ instead of $\dmmd(\pi||\cdot)$. 
    From the convergence standpoint, \Cref{thm:continuous_time_convergence} shows that $\dmmd(\cdot||\pi)$ converges globally with an exponential rate up to a small barrier when $\pi$ satisfies Poincar\'{e} inequality along with an extra regularity condition on the density ratio. 
    This favorable convergence property is no longer true for $\dmmd(\pi||\cdot)$. 
    Practically speaking, the Wasserstein gradient of $\dmmd$ requires inverting a kernel integral operator.
    It is more efficient to do this just once with a kernel integral operator with respect to $\pi$ as in \eqref{eq:dmmd_with_inverse_operator}, rather than with respect to $\mu_t$, which would happen at each time (as done by \citealt{he2022regularized}). 
\end{rem}

Computing the $\dmmd$ flow is intractable since the dynamics are in continuous-time, and in practice, we do not have access to $\pi$, but only samples from it.
In the next two sections, we build a tractable approximation of the DrMMD flow which provably achieves near-global convergence under similar assumptions as the ones in \Cref{sec:chard_flow}. Compared to the $\dmmd$ flow, this approximation combines a discretization in time (introduced in \Cref{sec:chard_particle_flow}) with a particle-based space-discretization (introduced in \Cref{sec:space_dis}). While the space-discretization techniques that we used are well-known in the Wasserstein gradient flow literature, our time-discretized scheme deviates from standard approaches, which are insufficient to guarantee near-global convergence in our case.
\section{Time-discretized $\dmmd$ flow} \label{sec:chard_particle_flow}
In this section, we first construct and analyze the forward Euler scheme of \eqref{eq:cont_eqn}, a simple time-discretization of
the $\dmmd$ flow which we call $\dmmd$ Descent.  Compared to the $\dmmd$ flow, the convergence of  $\dmmd$ descent is affected by an additional smoothness-related time-discretization error that blows up as $\lambda$ approaches 0, thus preventing near-global convergence. 
To address this issue, we propose in \Cref{subsec:time_dis_adaptive} an alternative discrete-time scheme that adapts the value of the regularization coefficient $\lambda$  across the descent iterates, which we call \emph{Adaptive $\dmmd$ Descent}. We show that this scheme converges in the KL divergence up to a barrier term that vanishes as the discretization step size goes to zero.

\subsection{$\dmmd$ Descent}\label{subsec:time_dis}
The forward Euler discretization of the $\dmmd$ flow (or $\dmmd$ \emph{Descent} in short) with\newpage \noindent step size $\gamma >0$ consists of sequence of probabilities $(\mu_n)_{n \in \N}$ defined by the  recursion
\begin{align}\label{eq:time_discretized_flow}
\mu_{n+1} = \left(\Id - \gamma(1+\lambda) \nabla h_{\mu_n,\pi}^\ast \right)_{\#} \mu_n , \quad \mu_{n=0} = \mu_0,
\end{align}
where $h_{\mu_n,\pi}^\ast = 2\left(\Sigma_\pi + \lambda \Id\right)^{-1} \left(m_{\mu_n} - m_{\pi} \right)$ is the $\dmmd$ witness function.
This scheme was previously considered in \citet{arbel2019maximum}; in particular, Proposition 4 of \citet{arbel2019maximum} shows that the discrete-time $\mmd$ dissipation rate along the $\mmd$ Descent iterates follows the (continous time) rate of the $\mmd$ flow up to an error term proportional to the step size $\gamma$ and the smoothness parameters of the problem\footnote{In the $\mmd$ flow case, the main smoothness parameters is the Lipschitz constant of the kernel.}. 

Next, we turn to study the convergence of $\dmmd$ descent. One way is to treat $\dmmd$ as $\mmd$ with a regularized kernel $\tilde{k}$ and follow Proposition 4 of \cite{arbel2019maximum}, however this does not quantitatively take into account the role of (de)-regularization parameter $\lambda$ that balances the trade-off between geodesic convexity and smoothness of $\calF_{\dmmd}$.
Instead, we adopt the same strategy of \Cref{sec:chard_flow} that exploits the interpolation property of $\dmmd$ towards $\chi^2$-divergence: in the following proposition, we study the dissipation of KL divergence along the $\dmmd$ Descent when the target $\pi$ satisfies a Poincar\'{e} inequality, in which the role of (de)-regularization parameter $\lambda$ is highlighted in the approximation and discretization errors.

\begin{prop}[Descent lemma in KL]\label{prop:descent_lemma_KL}
Suppose $k$ satisfies \Cref{assumption:universal} and \ref{assumption:bounded_kernel}, and suppose the target $\pi$ and $\dmmd$ gradient descent iterates $\left(\mu_n \right)_{n \in \N}$ satisfy the following:
\begin{enumerate}[itemsep=5.0pt,topsep=5pt,leftmargin=*]
\item $\pi$ satisfies a Poincar\'{e} inequality with constant $C_P$  and its potential is $\beta$-smooth, i.e.,  $\pi \propto \exp(-V)$ with $\bH V \preceq \beta \Id$.
\item $\mu_n, \pi \ll \calL^d$.
\item $\frac{ d\mu_n}{d \pi} - 1 \in \operatorname{Ran} ( \mathcal{T}_\pi^{r} )$ with $r > 0$, i.e., there exists $ q_n \in L^2(\pi) $ such that $\frac{d\mu_n }{ d\pi} - 1 = \calT_\pi^{r} q_n$.
\item $\|q_n\|_{L^2(\pi)} \leq Q$, $\left\| \nabla V^\top \nabla \left( \frac{d \mu_n}{d\pi} \right) \right\|_{L^2(\pi)} \leq \mathcal{J}$, $\left\| \Delta \left( \frac{d \mu_n}{d\pi} \right) \right\|_{L^2(\pi)} \leq \mathcal{I}$ for all $n = 1, \cdots, n_{max}$. 
\item For all $i = 1, \ldots, d$, $\lim\limits_{x\to \infty} \left( h_{\mu_n, \pi}^\ast(x) - 2 \frac{d \mu_n}{ d \pi} (x) \right) \left( \partial_i \frac{d \mu_n}{ d \pi} (x) \right) \pi(x) \to 0$.
\item There exists a constant $1 < \zeta < 2$ such that for all $n = 1, \ldots, n_{\max}$, the step size $\gamma$ satisfies 
\begin{align}\label{appeq:gamma_condition_1}
\gamma  \leq \frac{\zeta-1}{2\zeta (1+\lambda) \sqrt{\chi^2\left(\mu_n \| \pi\right) \frac{K_{2 d}}{\lambda } } } .
\end{align}
\end{enumerate}
Then for all $ 0 \leq n \leq n_{\max}$ and $0 < \lambda \leq 1$,
\begin{align}\label{eq:KL_descent_1}
\begin{aligned}
\kl(\mu_{n+1} \| \pi) - \kl(\mu_{n} \| \pi) &\leq -\frac{2}{C_P} \chi^2(\mu_n \| \pi) \gamma \\
&+  \underbrace{ 4 \gamma \lambda^{r} Q \left( \mathcal{J} +\mathcal{I} \right)}_{\text{Approximation error}}  +  \underbrace{ 8 \gamma^2 (\beta + \zeta^2) \chi^2(\mu_n \| \pi) \frac{K_{1d}+K_{2d} }{\lambda}}_{\text{Discretization error}}. 
\end{aligned}
\end{align}
\end{prop}
The proof is provided in \Cref{appsec:proof_kl_descent_lemma}. 
Conditions 1-5 of \Cref{prop:descent_lemma_KL} are similar to  conditions 1-5 of \Cref{thm:continuous_time_convergence}, which assumes a functional inequality on the target $\pi$ and  regularity on the density ratio $\frac{d \mu_n}{d \pi} - 1$. 
In a similar spirit to \Cref{thm:continuous_time_convergence}, the fourth regularity condition is automatically satisfied under a stronger range assumption on $ \frac{ d\mu_n}{d \pi} - 1 $.
For the sake of brevity, we directly assume uniform upper bounds in the fourth condition rather than writing it out in a separate corollary like \Cref{cor:continuous_time_convergence}.
Compared to the continuous time regime, two extra conditions are necessary. The first one is a smoothness condition on the potential, $\bH V \preceq \beta \Id$, which is commonly used in the convergence analysis of discrete-time Langevin-based samplers 
~\citep{dalalyan2017further,dalalyan2019user,durmus2019analysis, dalalyan2022bounding,vempala2019rapid}.
It can be relaxed to $\nabla V$ being Hölder-continuous with exponent $s \in [0, 1]$~\citep{chatterji2020langevin}. 
The second one, \eqref{appeq:gamma_condition_1}, is an upper bound on the step size $\gamma$, aligning with the principle that step size should be small enough for the discrete-time scheme to inherit the properties of its continuous analog. This condition will be more thoroughly discussed in \Cref{rem:step_size} when all the conditions on $\gamma$ in \Cref{thm:discrete_time_KL} and \Cref{thm:discrete_time_KL_2} are presented.
%
\begin{rem}[Approximation-discretization trade-off of $\dmmd$ Descent]
$  $ \\ 
If we compare the discrete-time KL dissipation of \eqref{eq:KL_descent_1} with its continuous-time counterpart in \eqref{eq:kl_d_dt_final_main}, we see that the first two terms on the RHS of \eqref{eq:KL_descent_1} admit continuous-time analogs present in \eqref{eq:kl_d_dt_final_main}. The discrete-time KL dissipation contains an additional (positive) term representing the \emph{time discretization error}:  unlike the approximation error term that vanishes as $\lambda $ approaches 0, this term actually diverges as $\lambda $ approaches 0. Therefore, replicating the arguments of the continuous-time result of \Cref{thm:continuous_time_convergence} in the discrete-time regime would yield a barrier that does not vanish as $\lambda \to 0$, hinting at a trade-off between approximation and discretization similar to that of \Cref{rem:convex_smooth_trade_off}. In the next section, we propose a refined adaptive discrete-time descent scheme that addresses the convergence issues.
\end{rem}

\subsection{Adaptive $\dmmd$ Descent}\label{subsec:time_dis_adaptive}
The $\dmmd$ flow and descent dynamics are defined for a value of $\lambda$ that remains fixed throughout time.
The KL dissipation provided in \Cref{prop:descent_lemma_KL} is
a function of $\lambda$, however; thus, to obtain a sequence of measures with better convergence guarantees than the $\dmmd$ descent, we now construct and analyze a sequence of iterates obtained by selecting, at each iteration, the value of $\lambda$ minimizing the sum of the approximation error and time-discretization error presented in \eqref{eq:KL_descent_1}. This sequence, which we term \emph{Adaptive $\dmmd$ descent}, is given by
\begin{align}\label{eq:adaptive_time_discretized_flow}
\mu_{n+1} = \left(\Id - \gamma(1+\lambda) \nabla h_{\mu_n,\pi}^\ast \right)_{\#} \mu_n , \quad  \lambda_n = \left( 2 \gamma \chi^2(\mu_n \| \pi) \frac{( \beta + \zeta^2 ) (K_{1d} + K_{2d})}{ Q (\mathcal{J} + \mathcal{I}) } \right)^{\frac{1}{r + 1}}.
\end{align}
The particular choice of $\lambda_n$ above minimizes the sum of the approximation error and time-discretization error presented in \eqref{eq:KL_descent_1}. 
This optimal choice indicates that $\lambda_n$ should shrink towards $0$ as $\chi^2(\mu_n\|\pi)$ decreases along $\dmmd$ gradient descent: at the early stages of the scheme, it is desirable to have a larger $\lambda_n$, corresponding to a smoother objective functional and enabling larger step sizes; then as $\dmmd$ gradient descent iterates $\mu_n$ get closer to $\pi$, a smaller $\lambda_n$ enables the scheme to operate closer to the $\chi^2$ flow regime, which metrizes a stronger topology and can better witness the difference between $\mu_n$ and $\pi$. 
Additionally, as the potential $V$ becomes less smooth, i.e., $\beta$ gets larger, then $\lambda_n$ should also increase to account for the loss of smoothness from $V$. 
Unlike the $\dmmd$ descent in \Cref{subsec:time_dis}, since the Wasserstein gradient updating $\mu_n$ comes from a different $\dmmd$ at each iteration, the adaptive scheme of \eqref{eq:adaptive_time_discretized_flow} constitutes a significant departure from related works in Wasserstein gradient descent \citep{glaser2021kale, arbel2019maximum, korba2021kernel, hertrich2023generative, hertrich2023wasserstein, chewi2020svgd}. 

\begin{rem}[Adaptive kernel]
Recent applications of MMD-based generative modeling algorithms with adaptive kernels (in particular, time-dependent kernel hyperparameters) demonstrate improved empirical performance over fixed kernels in both Wasserstein gradient flow on $\mmd$~\citep{galashov2024deep} and generative adversarial networks with an $\mmd$ critic~\citep{li2017mmd,arbel2018gradient}: the latter can be related to gradient flow on the critic where $\mu_n$ is restricted to the output of a generator network \citep[see e.g.][]{franceschi2023unifying}. As the $\dmmd$ is an $\mmd$ with a regularized kernel $\tilde{k}$ that depends on $\lambda_n$, the Adaptive $\dmmd$ Descent thus falls into the former category.
\citet[Proposition 3.1]{galashov2024deep} demonstrates faster convergence for $\mmd$ gradient flow with an adaptive kernel, for the parametric setting of Gaussian distributions $\pi$ and $\mu_t$.
Our analysis is the first to prove theoretically that adaptive kernels can result in improved convergence for more general nonparametric settings. We believe that the theoretical analysis of adaptivity by varying other hyperparameters (such as the kernel bandwidth for RBF kernels) remains an interesting avenue for future work.
\end{rem}


By leveraging the quasi-descent lemma in KL divergence in \Cref{prop:descent_lemma_KL}, we are able to establish the following theorem, which provides a near-global convergence result of the Adaptive $\dmmd$ gradient descent iterates in KL divergence. 

\begin{thm}[Near-global convergence of adaptive $\dmmd$ gradient descent]\label{thm:discrete_time_KL}
$  $ \\
Suppose $k$ satisfies \Cref{assumption:universal} and  \ref{assumption:bounded_kernel} and $K \leq 1$, and suppose the target $\pi$ and adaptive $\dmmd$ gradient descent iterates $\left(\mu_n \right)_{n\in \N}$ satisfy the following conditions:
\begin{enumerate}[itemsep=3.0pt,topsep=5pt,leftmargin=*]
\item $\pi$ satisfies a Poincar\'{e} inequality with constant $C_P$  and its potential is $\beta$-smooth, i.e.  $\pi \propto \exp(-V)$ with $\bH V \preceq \beta \Id$.
\item $\mu_n, \pi \ll \calL^d$.
\item $\frac{ d\mu_n}{d \pi} - 1 \in \operatorname{Ran} ( \mathcal{T}_\pi^{r} )$ with $r > 0$, i.e., there exists $ q_n \in L^2(\pi) $ such that $\frac{d\mu_n }{ d\pi} - 1 = \calT_\pi^{r} q_n$.
\item $\|q_n\|_{L^2(\pi)} \leq Q$, $\left\| \nabla V^\top \nabla \left( \frac{d \mu_n}{d\pi} \right) \right\|_{L^2(\pi)} \leq \mathcal{J}$, $\left\| \Delta \left( \frac{d \mu_n}{d\pi} \right) \right\|_{L^2(\pi)} \leq \mathcal{I}$ for all $n = 1, \cdots, n_{max}$. 
\item For all $i = 1, \ldots, d$, $\lim\limits_{x\to \infty} \left( h_{\mu_n, \pi}^\ast(x) - 2 \frac{d \mu_n}{ d \pi} (x) \right) \left( \partial_i \frac{d \mu_n}{ d \pi} (x) \right) \pi(x) \to 0$.
\item There exists a constant $1 < \zeta < 2$ such that the step size $\gamma$ satisfies 
\hspace{-100pt}
\begin{align}\label{eq:gamma_condition}
    \gamma \leq \frac{1}{8} \left( \frac{\zeta - 1}{\zeta} \right)^{\frac{2r+ 2}{2r+1} } \left( \frac{1}{Q} \right) \left( \frac{1}{\mathcal{J} + \mathcal{I}} \right)^{\frac{1}{2r + 1}} \left( \frac{1}{K_{2d}} \frac{1}{ \beta + \zeta^2} \right)^{\frac{r}{2r + 1}} 
    \wedge \frac{1}{4} \frac{\zeta - 1}{\zeta} \frac{1}{Q^2 K_{2d}}  
    \wedge \frac{C_P}{2} \wedge 1 .
\end{align}
\end{enumerate}
Then by taking (de)-regularization parameter $\lambda_n = \left( 2\gamma \chi^2(\mu_n \| \pi) \frac{( \beta + \zeta^2 ) (K_{1d} + K_{2d})}{ Q (\mathcal{J} + \mathcal{I}) } \right)^{\frac{1}{r + 1}} \wedge 1$, we have
\begin{align}\label{eq:discrete_time_rate_1}
    \kl(\mu_{n_{max} } \| \pi) &\leq \exp\left(- \frac{ 2 n_{max} \gamma}{C_P} \right) \kl(\mu_0 \| \pi) \nonumber \\
    &+ 4 \gamma^{ \frac{ r }{ r + 1 } } C_P Q^{\frac{2r+1}{r+1}} \Big( (K_{1d} + K_{2d})(\beta + \zeta^2) \Big)^{\frac{ r }{ r + 1 }} ( \mathcal{J} +\mathcal{I})^{\frac{ 1 }{ r + 1 }}  .
\end{align}
\end{thm}

The proof can be found in \Cref{appsec:proof_discrete_time_convergence}. 
The conditions 1-5 of \Cref{thm:discrete_time_KL} are the same as conditions 1-5 of \Cref{prop:descent_lemma_KL}. Since the RHS of \eqref{eq:gamma_condition} are all constants, condition 6 is satisfied when the step size $\gamma$ is small enough.

The implication of \Cref{thm:discrete_time_KL} is that $\dmmd$ gradient descent exhibits an exponential rate of convergence (linear convergence) in terms of KL divergence up to an extra barrier of order $\calO (\gamma^{\frac{r}{r+1}})$.
The barrier term shows up as the result of picking the optimal regularization parameter $\lambda_n$ that best trades off the approximation error and discretization error in \Cref{prop:descent_lemma_KL}.
\Cref{thm:discrete_time_KL} is reminiscent of the convergence result of the Langevin Monte Carlo sampling algorithm, whose KL divergence also decreases exponentially up to an extra barrier, but of order $\calO(\gamma)$ ~\citep[Theorem 2]{vempala2019rapid}.
Unlike the continuous-time result of \Cref{thm:continuous_time_convergence}, in which the barrier can be made arbitrarily small by taking small enough regularization, taking the step size $\gamma$ in \Cref{thm:discrete_time_KL} to be arbitrarily small will significantly impact the rate of convergence, even though it is exponential in terms of $n_{\max}$. 
By making the step sizes adaptive with the number of iterations and imposing an extra condition, the barrier term actually vanishes, as demonstrated in the following theorem. 

\begin{thm}[Global convergence of $\dmmd$ gradient descent]\label{thm:discrete_time_KL_2}
Suppose that $k$ satisfies Assumptions \ref{assumption:universal} and \ref{assumption:bounded_kernel}, and that the conditions in \Cref{thm:discrete_time_KL} on $\dmmd$ gradient descent iterates $(\mu_n)_{n\in \N}$, target distribution $\pi$, regularization coefficient $\lambda_n$ and step size $\gamma_n$ are satisfied.
If additionally, the step size $\gamma_n$ satisfies
\begin{align}\label{eq:gamma_condition_2}
    &\gamma_n \leq  \frac{1}{(K_{1d} + K_{2d})(\beta + \zeta^2)} \left( \frac{1}{Q( \mathcal{J} +\mathcal{I}) } \right)^{\frac{1}{r}} \left( \frac{1}{8 C_P} \right)^{\frac{r+1}{r}} \chi^2(\mu_n \| \pi)^{\frac{1}{r}},
\end{align}
for all $n = 1, \cdots, n_{max}$,
then
\begin{align}\label{eq:discrete_time_rate_2}
\kl(\mu_{n_{max}} \| \pi) \leq \prod_{n=1}^{n_{max}} \left(1 -\frac{1}{C_P} \gamma_n \right) \kl(\mu_{0} \| \pi) . 
\end{align}
\end{thm}
The proof can be found in \Cref{appsec:proof_discrete_time_convergence_2}. 
Compared with \eqref{eq:discrete_time_rate_1}, \eqref{eq:discrete_time_rate_2} provides a cleaner upper bound without the barrier term and leads to global convergence. 

\begin{rem}[Iteration complexity]
We now turn to analyze the iteration complexity of $\dmmd$ gradient descent from Theorems \ref{thm:discrete_time_KL} and \ref{thm:discrete_time_KL_2}. \vspace{1mm} \\
(i) From \Cref{thm:discrete_time_KL}, given an error threshold $\delta > 0$, $\dmmd$ descent would reach $\mathrm{KL}(\mu_{n_{\max}} \| \pi) \leq \delta$ after $n_{max} \geq \frac{C_P}{2 \gamma} \log \frac{\mathrm{KL}(\mu_0 \| \pi)}{\delta} = \calO ( (\frac{1}{\delta})^{\frac{r+1}{r} } \log \frac{1}{\delta} )$ iterations. By comparison, when $\pi$ satisfies a Poincar\'{e} inequality, Langevin Monte Carlo (LMC) has an iteration complexity of $\mathcal{O}(\frac{1}{\delta})$ up to logarithmic terms~\citep[Theorem 7]{chewi2024analysis}. 
    As an approximation to the $\chi^2$ flow, the iteration complexity of DrMMD flow is worse than LMC, because at each iterate of DrMMD flow, there is an extra approximation error in addition to time-discretization error (see our Proposition 5.1). With the optimal choice of regularization $\lambda$ that balances these two errors, the overall per-step error is of order $\calO(\gamma^{1+ \frac{r}{r+1}})$. 
    In contrast, at each iteration, LMC only incurs the time-discretization error, which is of order $\calO(\gamma^2)$, smaller than that of DrMMD flow. 
    As a result, LMC exhibits better iteration complexity than our DrMMD flow; however, LMC requires knowledge of the score of $\pi$, while DrMMD flow only requires samples from $\pi$.
\vspace{1mm} \\
(ii) From \Cref{thm:discrete_time_KL_2}, we consider two cases. On the one hand, if there exists a threshold $N_0$ such that $\chi^2(\mu_n \| \pi) \geq n^{-r}$ holds for all $n \geq N_0$, then we select step size $\gamma_n \asymp C_P n^{-1}$ for all $n \geq N_0$ such that both \eqref{eq:gamma_condition} and \eqref{eq:gamma_condition_2} are satisfied, and consequently $\prod_{n=N_0}^{n_{max}} \left( 1 -\frac{1}{C_P} \gamma_n \right) = \calO \left( \frac{1}{n_{\max}} \right) \to 0$ so the iteration complexity of $\dmmd$ gradient descent is $\calO\left( \frac{1}{\delta} \right)$. On the other hand, if such a threshold $N_0$ does not exist, then there exists a subsequence $n_1, n_2, \ldots, n_S, \ldots$ such that $ \chi^2(\mu_{n_s} \| \pi) \leq n_s^{-r}$ for all $s \geq 1$. Since KL divergence is smaller than $\chi^2$-divergence~\citep{van2014renyi}, we have $ \kl(\mu_{n_s} \| \pi) \leq n_s^{-r}$ for all $s \geq 1$. Notice that KL divergence is monotonically decreasing based on \eqref{eq:kl_monotone}, we have $ \kl(\mu_n \| \pi) \leq n_s^{-r}$ for all $n_s \leq n \leq n_{s+1}$ so that $\lim_{n \to \infty} \kl(\mu_n \| \pi) = 0$. 
Unfortunately, we are not able to derive iteration complexity in this case because the growth rate of $\{n_s\}_{s \geq 1}$ is unknown. 
\end{rem}

\begin{rem}[Step size $\gamma$]\label{rem:step_size}
\Cref{thm:discrete_time_KL} imposes a condition on the step size $\gamma$ in \eqref{eq:gamma_condition} and \Cref{thm:discrete_time_KL_2} imposes an additional condition in \eqref{eq:gamma_condition_2}.
These conditions subsume the condition \eqref{appeq:gamma_condition_1} on step size in \Cref{prop:descent_lemma_KL}. (See derivations in \Cref{appsec:proof_discrete_time_convergence}.) The conditions \eqref{eq:gamma_condition} and \eqref{eq:gamma_condition_2} become more stringent as the potential $V$ becomes less smooth, i.e., when $\beta$ gets larger, similar to the analysis in Langevin Monte Carlo~\citep{balasubramanian2022towards, vempala2019rapid} and Stein Variational Gradient Descent~\citep{korba2020non}.
The condition also becomes more stringent as the density ratio becomes less regular, i.e., when $r$ gets closer to $0$ and $Q, \mathcal{J}, \mathcal{I}$ get larger, similar to \cite{he2022regularized}. 
\end{rem}


The adaptive $\dmmd$ descent schemes of \eqref{eq:time_discretized_flow} 
and \eqref{eq:adaptive_time_discretized_flow} defined via push-forward operations can be equivalently expressed by the following update scheme that defines a trajectory of samples $(y_n)_{n \in \N}$ whose distributions are precisely the Adaptive $\dmmd$ descent iterates $(\mu_n)_{n \in \N}$,
\begin{align}\label{eq:time_discretized_particle_system}
    y_{n+1} = y_n -\gamma (1+\lambda_n)\nabla h_{\mu_n,\pi}^\ast(y_n), \quad y_0 \sim \mu_0.
\end{align}
%
Unfortunately, \eqref{eq:time_discretized_particle_system} is still intractable in practice because $h_{\mu_n,\pi}^\ast$ depends on the unknown distribution $\mu_n$. Therefore, an additional discretization in space is needed to approximate \eqref{eq:time_discretized_particle_system} within a tractable algorithm. We propose to do so in the next section through a system of interacting particles, i.e., $\dmmd$ particle descent.

\section{$\dmmd$ particle descent}\label{sec:space_dis}
Suppose we have $M$ samples from the target distribution $\{x^{(i)}\}_{i=1}^M \sim \pi$ and $N$ samples from the initial distribution $\{y^{(i)}_0\}_{i=1}^N \sim \mu_0$.
The $\dmmd$ particle descent is defined as:\newpage
\begin{align}\label{eq:particle_biased}
  y_{n+1}^{(i)}=y_n^{(i)}-\gamma(1+\lambda_n) \nabla h_{\hat{\mu}_n, \hat{\pi}}^\ast (y_n^{(i)}),
\end{align}
where $\hat{\mu}_n = \frac{1}{N} \sum_{i=1}^N \delta_{ y_n^{(i)} }$ and $\hat{\pi} = \frac{1}{M} \sum_{i=1}^M \delta_{ x^{(i)} }$ denote respectively the empirical distribution of the particles at time step $n$ and the target, and
where $h_{\hat{\mu}_n, \hat{\pi}}^\ast = 2 \left(\Sigma_{\hat{\pi}} + \lambda_n \Id \right)^{-1} ( m_{\hat{\mu}_n} -  m_{\hat{\pi}} )$. Unfortunately, the KL divergence is ill-defined on empirical distributions $\hat{\mu}_n$, which means the analysis of \Cref{thm:discrete_time_KL} is no longer applicable to study the convergence of $\dmmd$ particle descent. Therefore, in the next theorem, we instead resort to the Wasserstein-2 distance to analyze the convergence of $\dmmd$ particle descent. 
For simplicity of presentation below, we assume $K \leq 1$.

\begin{thm}\label{thm:final_rate}
Suppose that $k$ satisfies Assumptions \ref{assumption:universal} and \ref{assumption:bounded_kernel} with $K \leq 1$, and that all the conditions in \Cref{thm:discrete_time_KL} on $\dmmd$ gradient descent iterates $(\mu_n)_{n\in \N}$, target distribution $\pi$, regularization coefficient $\lambda_n$ and step size $\gamma$ are satisfied. In addition, suppose $(\mu_n)_{n\in \N}$ has bounded fourth moment and the target $\pi$ satisfies a Talagrand-2 inequality with constant $C_T$. 
Let the number of samples $M, N$ satisfy
\begin{align}\label{eq:M_N_condition}
    M &\gtrsim \left( \frac{1}{\gamma}\right)^2 \left( \frac{1}{ \min\limits_{i = 1, \ldots, n_{\max}} \kl(\mu_i \| \pi) Z \wedge 1 } \right)^{\frac{2}{r+1}} \exp\left( \frac{ 8 n_{\max} \gamma^{\frac{r}{ r + 1}} R }{ \left( \min\limits_{i = 1, \ldots, n_{\max}} \kl(\mu_i \| \pi) Z \right)^{\frac{1}{r+1}} \wedge 1 } \right), \nonumber \\
    N &\gtrsim \left( \frac{1}{\gamma} \right)^{\frac{2r}{r+1}} \exp\left( \frac{ 8 n_{\max} \gamma^{\frac{r}{ r + 1}} R }{ \left( \min\limits_{i = 1, \ldots, n_{\max}} \kl(\mu_i \| \pi) Z \right)^{\frac{1}{r+1}} \wedge 1 } \right) \vee \left( \frac{1}{\gamma} \right)^{\frac{r (d \vee 4)}{2r + 2} },
\end{align}
where $\gtrsim$ means $\geq$ up to constants, $R = K_{1 d} + \sqrt{K K_{2 d}}$ is a constant that only depends on the kernel, and $Z$ is a constant that only depends on $\beta, \zeta, K_{1d}, K_{2d}, Q, \mathcal{J}, \mathcal{I}$. 
Then we have
\begin{align*}
    \E \left[W_2 \left(\hat{\mu}_{n_{\max}}, \pi \right)\right] \leq \sqrt{ 2 C_T} \exp \left(-\frac{ n_{\max} \gamma}{C_P}\right) \sqrt{ \kl\left(\mu_0 \| \pi\right)} + \calO \left( \gamma^{\frac{r}{2r + 2}} \right),
\end{align*}
where the expectation is taken over initial samples $\{y_0^{(i)} \}_{i=1}^N$ drawn from $\mu_0$.
\end{thm}
The proof is provided in  \Cref{sec:proof_final}. \Cref{thm:final_rate} shows that for sufficiently large sample size $M$ and $N$, $\dmmd$ particle descent exhibits an exponential rate of convergence (linear convergence) in terms of Wasserstein-2 distance up to an extra barrier of order $\calO \left( \gamma^{\frac{r}{2r + 2}} \right)$. 

\begin{rem}[Talagrand-2 inequality]
We say that the target distribution $\pi$ satisfies a Talagrand-2 inequality with constant $C_T$ if for any $\nu \in \calP_2(\R^d)$, 
\begin{align*}
    W_2(\mu, \pi) \leq \sqrt{2 C_T \kl(\mu \| \pi)} .
\end{align*}
A Talagrand-2 inequality implies the Poincar\'{e} inequality in \eqref{eqn:poincare} with constant $C_P \leq C_T$, so the condition that $\pi$ satisfies a Talagrand-2 inequality is stronger than the condition in Theorems~\ref{thm:continuous_time_convergence} and \ref{thm:discrete_time_KL}.
A Talagrand-2 inequality allows linking of the two key components of \Cref{thm:final_rate}: the population convergence in terms of KL divergence proved in \Cref{thm:discrete_time_KL}, and the finite-particle propagation of chaos bound in terms of Wasserstein-2 distance proved in \Cref{prop:space_discretized}.
Talagrand inequality is widely used in finite-particle convergence analysis of Wasserstein gradient flows~\citep{shi2024finite}. 

\end{rem}
\begin{rem}[Iteration and sample complexity]
The choice of $\gamma = \calO \left( \delta^{\frac{2r + 2}{r}} \right)$ yields 
$\E \left[W_2 \left(\hat{\mu}_{n_{\max}}, \pi \right)\right] \leq \delta$ with an iteration complexity of $n_{\max } = \mathcal{O} \left(\left(\frac{1}{\delta}\right)^{\frac{2r + 2}{r}} \log \frac{1}{\delta}\right)$, which equals the square of the iteration complexity in \Cref{thm:discrete_time_KL} because Wasserstein-2 distance is of the same order as the square root of KL divergence. 
From \Cref{thm:discrete_time_KL}, we have
\begin{align*}
    \min\limits_{i = 1, \ldots, n_{\max}} \kl(\mu_i \| \pi) \leq \kl(\mu_{n_{\max}} \| \pi) = \calO\left( \delta^2 \right). 
\end{align*}
Therefore, from \eqref{eq:M_N_condition}, the sample complexity is at least
\begin{align*}
    M = \operatorname{poly} \exp \left(\delta^{ -\frac{2}{r} -\frac{2}{r + 1 } } \right) , \quad N = \operatorname{poly} \exp \left(\delta^{-\frac{2}{r} -\frac{2}{r + 1 }} \right) \calO \left( \delta^{- { d \vee 4} } \right).
\end{align*}
The poly-exponential sample complexity originates from the propagation of chaos in interacting particle systems~\citep{kac1956foundations}. Although $\calO \left(\delta^{-d \vee 4 } \right)$ is subsumed by the poly-exponential term, we still make it explicit to show that $N$ suffers from the curse of dimensionality as expected from the Wasserstein-2 distance between an empirical distribution and a continuous distribution~\citep{kloeckner2012approximation,lei2020convergence}.
For comparison, the sample complexity of SVGD in Theorem 3 of \cite{shi2024finite} is $\calO(\exp \exp\left( \delta^{-2} \right))$.
Similar results have also been established for mean-field Langevin dynamics~\citep{suzuki2023uniform, chen2024uniform}.

Recently, \citet{balasubramanian2024improved} proposed a refined finite-particle analysis of Stein Variational Gradient Descent (SVGD), which gives a uniform-in-time convergence bound, i.e., the bound does not blow up exponentially fast as the number of iterations $n_{\max}$ grows. 
The analysis of \citet{balasubramanian2024improved}, however, heavily relies on the relation that the time derivative of the KL divergence in the course of SVGD equals the squared kernel Stein discrepancy (KSD), a fact which does not hold for our DrMMD flow.
Recently, \citet{chen2025stationary} extended that analysis to finite-particle MMD gradient descent, but their analysis requires noise injection. 
We leave a more refined analysis of our finite-particle convergence result to future work.
\end{rem}

Having established the convergence of $\dmmd$ gradient flow/descent,  we next show that $\dmmd$ particle descent admits a closed-form implementation.
\Cref{prop:empirical_witness} shows that $h_{\hat{\mu}_n, \hat{\pi}}^\ast$ in \eqref{eq:particle_biased}, defined through the inverse of covariance operators, is computable using Gram matrices.
\begin{prop}\label{prop:empirical_witness}
Given empirical distributions $\hat{\mu}_n = \frac{1}{N} \sum_{i=1}^N y_n^{(i)}$, $\hat{\pi} = \frac{1}{M} \sum_{i=1}^M x^{(i)}$ and Gram matrices $K_{xx} = k(x^{1:M}, x^{1:M}) \in \R^{M \times M}$ and $ K_{xy} = k(x^{1:M}, y_n^{1:N}) \in \R^{M \times N}$, the witness function $h_{\hat{\mu}_n, \hat{\pi}}^\ast$ can be computed as:
\begin{align}\label{eq:empirical_witness}
h_{\hat{\mu}_n, \hat{\pi}}^\ast(\cdot) &= \frac{2}{N \lambda_n} k(\cdot, y^{1:N}) \one_N - \frac{2}{M \lambda_n} k(\cdot, x^{1:M}) \one_M - \frac{2}{N \lambda_n} k(\cdot, x^{1:M}) (M \lambda_n\Id + K_{xx} )^{-1} K_{xy} \one_N \nonumber \\
&+\frac{2}{M \lambda_n} k(\cdot, x^{1:M})(M \lambda_n\Id+K_{xx} )^{-1} K_{xx} \one_M, 
\end{align}
where $\one_M \in \R^M, \one_N \in \R^N$ are column vectors of ones.
\end{prop}
The proof of \Cref{prop:empirical_witness} can be found in \Cref{appsec:proof_empirical_chard}.
The gradient of $h_{\hat{\mu}_n, \hat{\pi}}^\ast$ can be  obtained using automatic differentiation libraries such as \textrm{JAX} \citep{jax2018github}. 

As indicated in \eqref{eq:adaptive_time_discretized_flow}, the regularization parameter $\lambda_n$ should be chosen to be proportional to $\chi^2(\mu_n \| \pi)^{\frac{1}{r+1}}$. In practice, however, both $\chi^2(\mu_n \| \pi)$ and $r$ are not accessible and $\chi^2(\mu_n \| \pi)$ is not even well-defined for the particle descent algorithm. To address this, we use $\mathrm{DrMMD}(\hat{\mu}_n \| \hat{\pi})$ as a proxy for $\chi^2(\mu_n \| \pi)$  (see Algorithm 1), which admits a closed-form expression with particles. The parameter $r$ is picked via a search over a pre-defined set $\{0.1, 0.5, 1.0\}$. The step size $\gamma$ should be chosen to satisfy the upper bound (21), which contains several constants that cannot generally be computed. In practice, our approach has been to select $\gamma$ to be sufficiently small for the flow to converge empirically. 
The final algorithm is summarized in \Cref{alg:dmmd_flow}.

At every iteration, computing $h_{\hat{\mu}_n, \hat{\pi}}^\ast$ with adaptive regularization $\lambda_n$ has a time complexity of $\calO(M^3 + NM + N^2)$ due to matrix inversion and multiplication. For $\dmmd$ particle descent with fixed $\lambda$, however,  the total computational cost can be reduced to $\calO(NM + N^2)$, which is exactly the same as $\mmd$ flow, because inversion of the $M\times M$ Gram matrix is only required once, and so it can be pre-computed at initialization (see \Cref{rem:one}). 
In contrast, when $N=M$, the complexity of Sinkhorn flow is $\calO(N^2/ \epsilon^3)$~\citep{feydy2019interpolating} with $\epsilon$ being the hyperparameter in Sinkhorn divergence, and the complexity of KALE flow is $\calO(N^3)$~\citep{glaser2021kale}.

\begin{algorithm}[t]\caption{$\dmmd$ particle descent}
\label{alg:dmmd_flow}
\textbf{Input:} Target samples $\{x^{(i)}\}_{i=1}^M \sim \pi$ and initial source samples $\{y^{(i)}_0\}_{i=1}^N \sim \mu_0$. Hyperparameters: step size $\gamma$, initial (de)-regularization coefficient $\lambda_0$, maximum number of iterations $n_{max}$ and regularity $r$. \\
\textbf{For} $n=0$ to $n_{max}$: \\
\hspace*{\algorithmicindent} 1. Compute witness function $h_{\hat{\mu}_n, \hat{\pi}}$ from \eqref{eq:empirical_witness}. \\ 
\hspace*{\algorithmicindent} 2. Compute $\dmmd(\hat{\mu}_n, \hat{\pi})$ with $h_{\hat{\mu}_n, \hat{\pi}}$ from \eqref{appeq:dmmd_samples}. \\
\hspace*{\algorithmicindent} 3. Rescale regularization coefficient $\lambda_n \propto \dmmd(\hat{\mu}_n, \hat{\pi})^{\frac{1}{r + 1}}$. \\
\hspace*{\algorithmicindent} 4. Update particles using \eqref{eq:particle_biased}:
\vspace{-5pt}
\begin{align*}
    y_{n+1}^{(i)}=y_n^{(i)}-\gamma(1+\lambda_n) \nabla h_{\hat{\mu}_n, \hat{\pi}}^\ast (y_n^{(i)})
\vspace{-10pt}
\end{align*}
\textbf{EndFor} \\
\textbf{Output:} $\{y^{(i)}_{n_{max}}\}_{i=1}^N$.
\end{algorithm}

\section{Related Work}\label{sec:related_work}
In this section, we discuss the works in the literature that are related to our proposed $\dmmd$ flow and spectral (de)-regularization.
\subsection{Gradient flows}

Stein Variational Gradient Descent (SVGD) is a popular algorithm for sampling from distributions using only an unnormalized density.
It can be written as either a gradient flow of the Kullback-Leibler (KL) divergence where the Wasserstein gradient of the KL is preconditioned by $\calT_{\mu_t}$ \citep{liu2016stein, liu2017stein, korba2020non}, or as a gradient flow  of the $\chi^2$-divergence whose Wasserstein gradient is preconditioned by $\calT_{\pi}$ \citep{chewi2020svgd},
\begin{align*}
    \frac{\partial \mu_t}{\partial t} = \nabla \cdot \left( \mu_t \cT_{\mu_t} \nabla \log \frac{\d\mu_t}{\d\pi} \right) = \nabla \cdot \left(\mu_t \cT_{\pi} \nabla \frac{\d\mu_t}{\d\pi} \right).
\end{align*}
Since SVGD may smooth the trajectory too much, \cite{he2022regularized} considered a (de)-regularized SVGD flow, 
\begin{equation}\label{eq:regularized_svgd}
    \frac{\partial \mu_t}{\partial t} = \nabla \cdot \left( \mu_t (\cT_{\mu_t} + \lambda \Id )^{-1} \calT_{\mu_t} \nabla \log \frac{\d\mu_t}{\d\pi} \right),
\end{equation}
which approaches the KL gradient flow (Langevin diffusion) as $\lambda\rightarrow 0$, demonstrating faster convergence than SVGD. A key difference between (de)-regularization in \eqref{eq:regularized_svgd} of \cite{he2022regularized} and our $\dmmd$ flow is that the flow in \eqref{eq:regularized_svgd} is driven by the regularized version of the Wasserstein gradient of KL divergence while $\dmmd$ flow is driven by the Wasserstein gradient of the regularized $\chi^2$-divergence.
An alternative interpretation of this difference is that the flow in \eqref{eq:regularized_svgd} is the gradient flow of KL divergence w.r.t.~the regularized Stein geometry \citep{duncan2019geometry}, whereas the $\dmmd$ flow is the gradient flow of regularized $\chi^2$-divergence w.r.t.~the Wasserstein geometry.

In addition to sampling from unnormalized distributions, Wasserstein gradient flows (particularly $\mmd$ flows) are widely used in the field of generative modelling~\citep{birrell2022f, gu2022lipschitz,hertrich2023wasserstein, hertrich2023generative, hertrich2024wasserstein, galashov2024deep}.
The MMD flow (with a smooth kernel)~\citep{arbel2019maximum} can be written as
\begin{align}\label{eq:mmd_flow}
    \frac{\partial \mu_t}{\partial t} = \nabla \cdot \left( \mu_t \nabla \cT_\pi \left( \frac{\d\mu_t}{\d\pi} - 1 \right) \right).
\end{align}
The MMD flow is known to get trapped in local minima, and several modifications have been proposed to avoid this in practice, such as noise injection \citep[see][Proposition 8]{arbel2019maximum} or non-smooth kernels, e.g., based on negative distances \citep{sejdinovic13energy}. 
 $\mmd$ gradient flows with non-smooth kernels have better empirical performance~\citep{hertrich2024wasserstein, hertrich2023generative}, but they do not preserve discrete measure and rely on approximating implicit time discretizations \citep{hertrich2024wasserstein} or slicing \citep{hertrich2023generative}; and they have no local minima apart from the global one~\citep{boufadene2023global}.

Recall that our $\dmmd$ flow takes the form
\begin{align*}
    \frac{\partial \mu_t}{\partial t} = \nabla \cdot \left(\mu_t \nabla (\calT_\pi + \lambda \Id)^{-1} \calT_\pi \left( \frac{\d\mu_t}{\d\pi} - 1 \right) \right), 
\end{align*}
which (de)-regularizes the $\mmd$ flow similarly to how \eqref{eq:regularized_svgd} (de)-regularizes SVGD. 
It is both proved theoretically in \Cref{thm:continuous_time_convergence}, \ref{thm:discrete_time_KL}, \ref{thm:final_rate}, and verified empirically in \Cref{sec:experiments}, that (de)-regularization results in faster convergence than $\mmd$ flow. 

Another closely related flow called LAWGD is considered in \cite{chewi2020svgd}, which swaps the gradient and integral operators of SVGD, leading to the following flow:
\begin{align*}
    \frac{\partial \mu_t}{\partial t} = \nabla \cdot \left(\mu_t \nabla \cT_{\pi} \frac{\d\mu_t}{\d\pi} \right). 
\end{align*}
LAWGD closely resembles the MMD flow in \eqref{eq:mmd_flow}, but \cite{chewi2020svgd} proposes to replace $\cT_{\pi}$ with an inverse diffusion operator, which requires computing the eigenspectrum of the latter and is unlikely to scale in high dimensions.

KALE~\citep{glaser2021kale} kernelizes the variational formulation of the KL divergence in a similar way as $\dmmd$ kernelizes the $\chi^2$-divergence in \eqref{eq:drmmd_variational}, but the KALE witness function does not have a closed form expression, so it requires solving a convex optimization problem, which makes the simulation of KALE gradient flow with particles computationally more expensive. 
Recently, \cite{neumayer2024wasserstein} studied kernelized variational formulation of $f$-divergences (referred to as Moreau envelopes of $f$-divergences in RKHS), which subsume both KALE \citep{glaser2021kale} and $\dmmd$. They prove that these functionals are lower semi-continuous and that their Wasserstein gradient flows are well-defined for smooth kernels. They do not study the convergence properties of their proposed flows, however.

The proposed DrMMD interpolates between MMD and $\chi^2$-divergence by using a specific spectral regularization known as Tikhnov regularization $(1+\lambda)(\calT + \lambda\Id)^{-1}\calT$, which interpolates between the identity operator $\Id$ as $\lambda \to 0$ and $\calT$ as $\lambda\to\infty$. 
DrMMD and its associated Wasserstein gradient flow can be easily extended to other spectral regularization strategies, such as the Showalter regularization, Landweber iteration, or cutoff regularization \citep{engl1996regularization} using the techniques in \citep{BAUER200752,hagrass2022spectral}. Alternative approximations to $\chi^2$-divergence have been proposed in the literature based on the idea of mollifiers, whose Wasserstein gradient flows have been constructed and for which convergence of the flows has been analyzed \citep{li2023sampling, craig2023blob,craig2023nonlocal}. Compared to DrMMD flow, these gradient flows rely on additional approximations, such as the use of log-sum-exp in \citet{li2023sampling} and the use of numerical integration to estimate convolution in \citet{craig2023blob,craig2023nonlocal}—and are not directly applicable in generative modeling settings where only samples are available. 

\subsection{Comparison with  diffusion-based generative models}
Diffusion-based generative models have been widely adopted in practice and are closely related to Wasserstein gradient flows~\citep{song2020score,ho2020denoising}. These models generate high-quality samples by reversing a pre-defined forward diffusion process, which gradually corrupts data with noise. To implement the reverse process,  the established practice is to estimate the score function via denoising score matching \citep{song2020score}. In contrast, Wasserstein gradient flows directly construct a trajectory by descending the objective in the steepest direction with respect to the Wasserstein metric. In particular, our proposed DrMMD gradient flow offers a tractable velocity field with a consistent finite-sample estimator without solving an additional optimization problem like score matching.

We emphasize that the main contribution of our paper is to establish convergence of the DrMMD flow, and that diffusion models for image generation require additional implementation details---most notably, the inductive biases introduced by deep neural networks. 
One possible avenue for future work is to simulate DrMMD gradient flows with kernels induced by learned deep neural network features on the data. In the case of MMD, this idea has been explored by~\citet{galashov2024deep}, who demonstrate generation performance comparable to established diffusion models. 
Another promising direction, proposed by~\citet{hertrich2023generative}, involves first using MMD gradient flow with Riesz kernels to generate particle trajectories, and then distilling these trajectories into a neural network-based generator. Both approaches would be of interest to extend our DrMMD gradient flow to domains such as image generation, as a topic for future work.

\subsection{(De)-regularization for supervised learning and hypothesis testing}
The idea of (de)-regularization is not new, and has been used in kernel Fisher discriminant analysis~\citep{mika99fisher} and kernel ridge regression \citep{caponnetto2007optimal, scholkopf2002learning}.
Subsequently, \cite{eric2007testing} employed this statistic in two-sample testing, where they constructed a test statistic that (de)-regularizes $\mmd(\mu \| \pi)$ with both covariance operators $\Sigma_\mu, \Sigma_\pi$. This work has been recently generalized in \cite{hagrass2022spectral} to more general spectral regularizations.
A (de)-regularized statistic is also employed by \cite{balasubramanian2017optimality,hagrass2023spectralgof} in the context of a goodness-of-fit test. \cite{balasubramanian2017optimality} refers to (de)-regularized $\mmd$ as `Moderated MMD'. To the best of our knowledge, the present work represents the first instance of the (de)-regularized $\mmd$ being used as a distance functional in  Wasserstein gradient flow. 
By only (de)-regularizing with $\Sigma_\pi$, $\dmmd$ approaches the $\chi^2$-divergence in the limit, a crucial property that is exploited in the proofs of the convergence results of Theorems~\ref{thm:continuous_time_convergence} and \ref{thm:discrete_time_KL}.
\section{Experiments}\label{sec:experiments}
In this section, we demonstrate the superior empirical performance of the proposed $\dmmd$ descent in various experimental settings.

\tikzset{example/.style={draw, circle, fill=orange, inner sep=1pt}}

\subsection{Three ring experiment}
We follow the experimental set-up in \cite{glaser2021kale} in which the target distribution $\pi$ ($\textcolor[rgb]{0.121, 0.467, 0.706}{\bullet}$) is defined on a manifold in $\R^2$ consisting of three non-overlapping rings. 
The initial source distribution $\mu_0$ ($\textcolor{orange}{\bullet}$) is a Gaussian distribution close to the vicinity of the first ring.
In this setting, all $f$-divergence gradient flows, including Langevin diffusion, are ill-defined because the target $\pi$ is not absolutely continuous with respect to the initial source $\mu_0$.
Nevertheless, we will simulate $\chi^2$ flow with an existing implementation of \cite{liu2023variational} that estimates the velocity field with a local linear estimator as one of the baseline methods.
In contrast, kernel-based gradient flows like $\mmd$, $\kale$, and $\dmmd$ gradient flows are well-defined in this setting, and are also used as baseline methods for comparison.

We sample $N=M=300$ samples from the initial source and the target distributions and run $\dmmd$ descent with adaptive $\lambda$ for $n_{max} = 100,000$ iterations, at which point all methods have converged. As in \cite{glaser2021kale}, we use a Gaussian kernel $ k(x, x^\prime)=\exp \left(- 0.5 \|x-x^\prime \|^2 / l^2 \right)$ with  bandwidth $l=0.3$. The step size for $\mmd$ descent is $\gamma = 10^{-2}$ and the step size for $\kale$ and $\dmmd$ descent is $\gamma = 10^{-3}$. We enforce a positive lower bound $\tilde{\lambda} = 10^{-3}$ for numerical stability and the regularity hyperparameter $r$ is optimized over the set of $\{0.1, 0.5, 1.0\}$.

From \Cref{fig:three_ring} \textbf{Left} and \textbf{Middle}, we can see that $\dmmd$ descent outperforms $\mmd$, $\kale$, and $\chi^2$ descent in terms of all dissimilarity metrics with respect to the target $\pi$: $\mmd$ and Wasserstein-2 distance.
\Cref{fig:three_ring_animation} is an animation plot visualizing the evolution of particles under these descent schemes, which demonstrates that both $\kale$ and $\dmmd$ descent are sensitive to the mismatch of support and stay concentrated in the support of the target $\pi$, while particles of $\mmd$ descent can diffuse outside the support of $\pi$. 
Note that "$\chi^2$" denotes an alternate estimate of the $\chi^2$ divergence due to ~\citet{liu2023variational}: being an $f$-divergence, we would expect "$\chi^2$ descent" to match the support of the target (as in $\kale$ and $\dmmd$). This is not the case due to bias in the velocity field being learned from samples.
Compared to $\kale$ descent, $\dmmd$ descent does not suffer from the numerical approximation error of the optimization routine when solving the velocity field of $\kale$, which explains its improved performance.

\begin{figure}
    \centering
    \includegraphics[width=\linewidth]{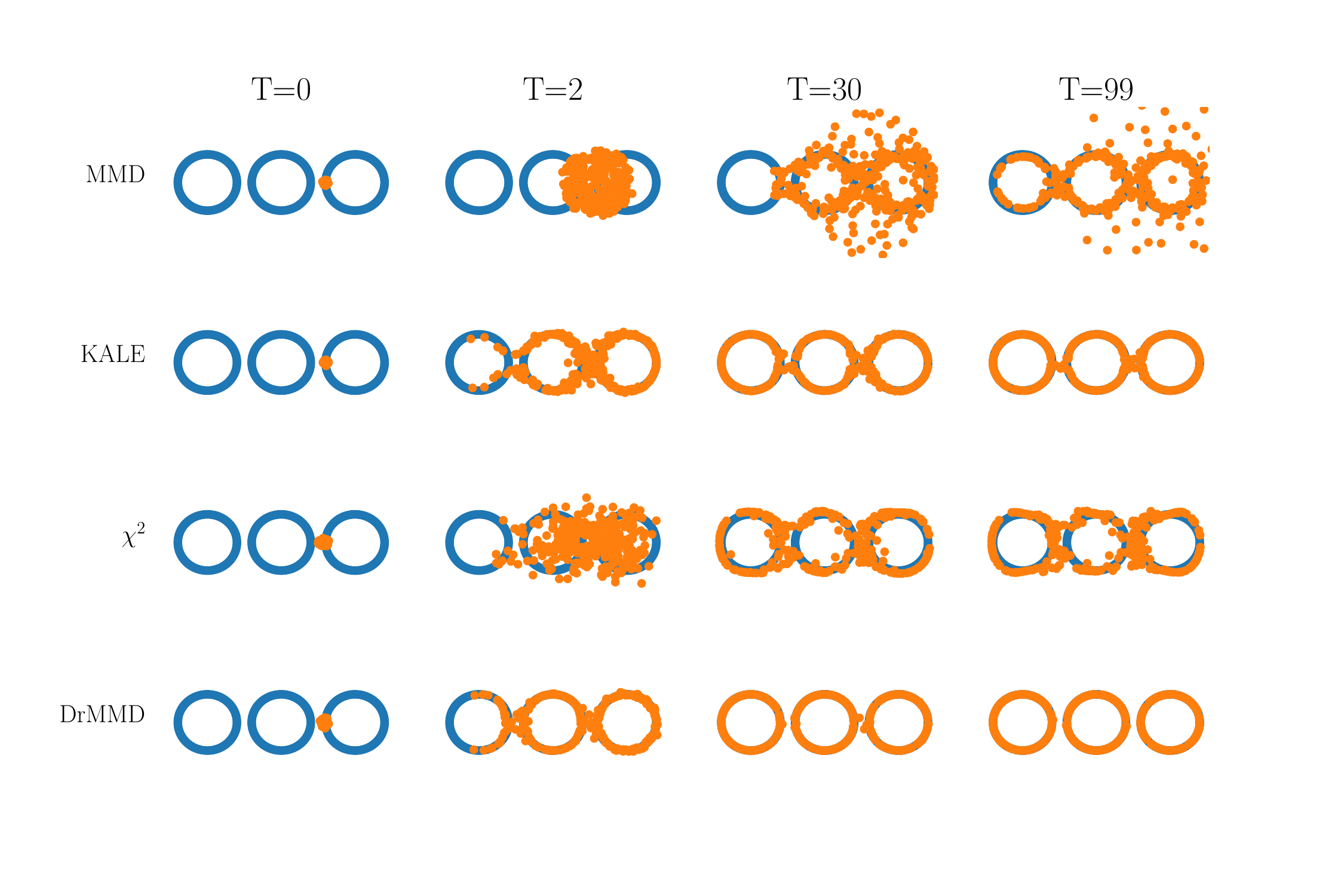}
    \vspace{-25pt}
    \caption{Animation of MMD, KALE, $\chi^2$ and DrMMD gradient descent on the Three-ring dataset.}
    \label{fig:three_ring_animation}
\end{figure}

\subsection{Gradient flow for training student/teacher networks}\label{sec:student_teacher}
Next, we consider a large-scale setting following \cite{arbel2019maximum}, where a student network is trained to imitate the outputs of a
teacher network. 
We consider a two-layer neural network of the form
\begin{align*}
    \psi(z, x) = G \left(b^1+W^1 \sigma\left(W^0 z+b^0\right)\right) ,
\end{align*}
where $\sigma$ is the ReLU non-linearity and $x$ is the concatenation of all network parameters $\left(b^1, W^1, b^0, W^0\right) \in \R^d$. $G$ is an element-wise non-linear function $G : \R \to \R, x \mapsto \exp(-\frac{1}{4} x^2)$.
The teacher network is of the form: $\Psi^{T}(z, \pi) = \int \psi(z, x) d\pi(x)$ where $\pi$ denotes the teacher distribution, and the student network is $\Psi^{S}(z, \mu) = \int \psi(z, x) d\mu(x)$ where $\mu$ denotes the student distribution.
Here we consider Gaussian distributed $\mu$ and $\pi$  for simplicity. 
The student network can imitate the behavior of the teacher network by minimizing the  objective\footnote{Note that our setting is slightly different from \cite{chizat2018global} in which $\mu,\pi$ are measures over the hidden neurons, while our setting follows \cite{arbel2019maximum} in which $\mu,\pi$ are measures over all the network parameters.} 
\begin{align}\label{eq:student_teacher_1}
    \min_{\mu \in \calP_2(\R^d)} \E_{z \sim \Pb_\text{data}} \left(\Psi^{T}(z, \pi) - \Psi^{S}(z, \mu) \right)^2,
\end{align}
where $\Pb_\text{data}$ is the distribution of the input data.
If we define the kernel as the inner product of the neural network feature maps,
\begin{align*}
    k(x, x^{\prime}) = \E_{z \sim \Pb_\text{data}} [\psi(z, x)^\top \psi(z, x^{\prime})], 
\end{align*}
then the objective of \eqref{eq:student_teacher_1} can be equivalently expressed as
\begin{align*}
    \min_{\mu \in \calP_2(\R^d)} \iint_\calX k(x, x^\prime) d (\pi - \mu)(x) d (\pi - \mu)(x^\prime) ,
\end{align*}
which is precisely the $\mmd(\mu \| \pi)^2$  under the kernel $k$. Since $G(x) = \exp(-\frac{1}{4} x^2)$, the kernel is bounded and so the $\mmd$ is well-defined. Also, since the kernel $k$ has bounded first and second-order derivatives, it satisfies the \Cref{assumption:bounded_kernel}.

\begin{figure}[t]
\vspace{-8pt}
    \begin{minipage}{0.65\textwidth}
    \hspace{-10pt}
    \includegraphics[width=220pt]{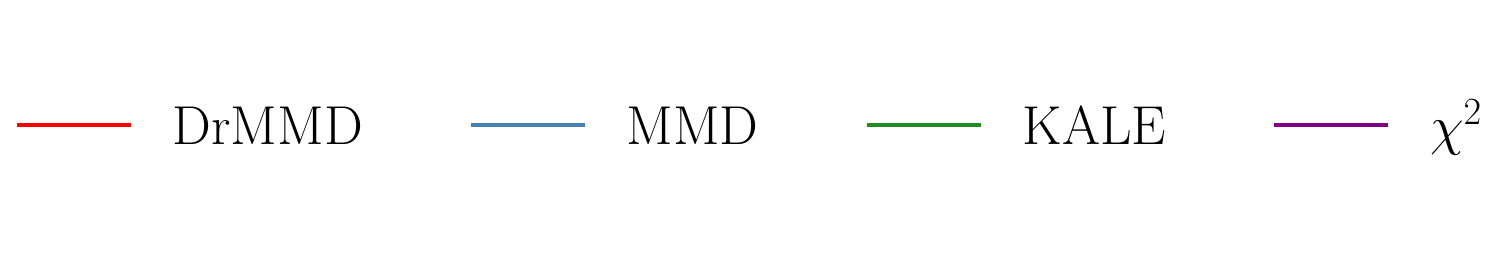}
    \vspace{-10pt}
    \end{minipage}
    \hspace{-50pt}
    \begin{minipage}{0.30\textwidth}
    \centering
    \includegraphics[width=180pt]{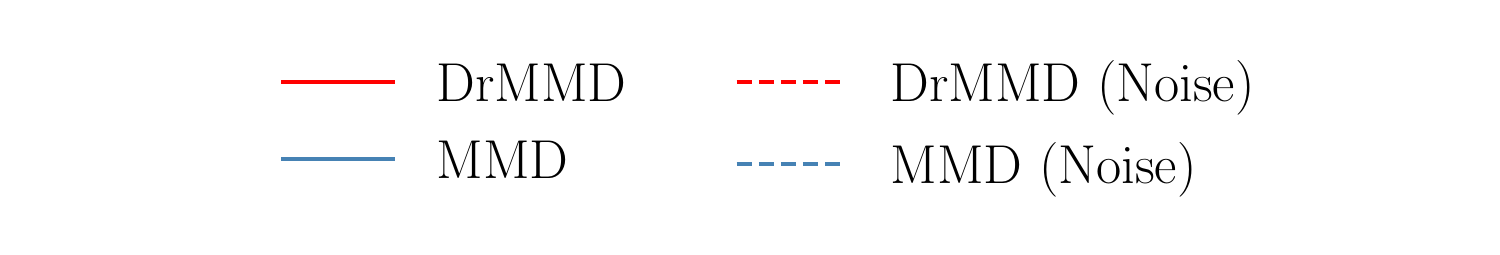}
    \vspace{-20pt}
    \end{minipage}
    \centering
    \begin{subfigure}{0.66\textwidth}
        \centering
        \includegraphics[width=1.0\linewidth]{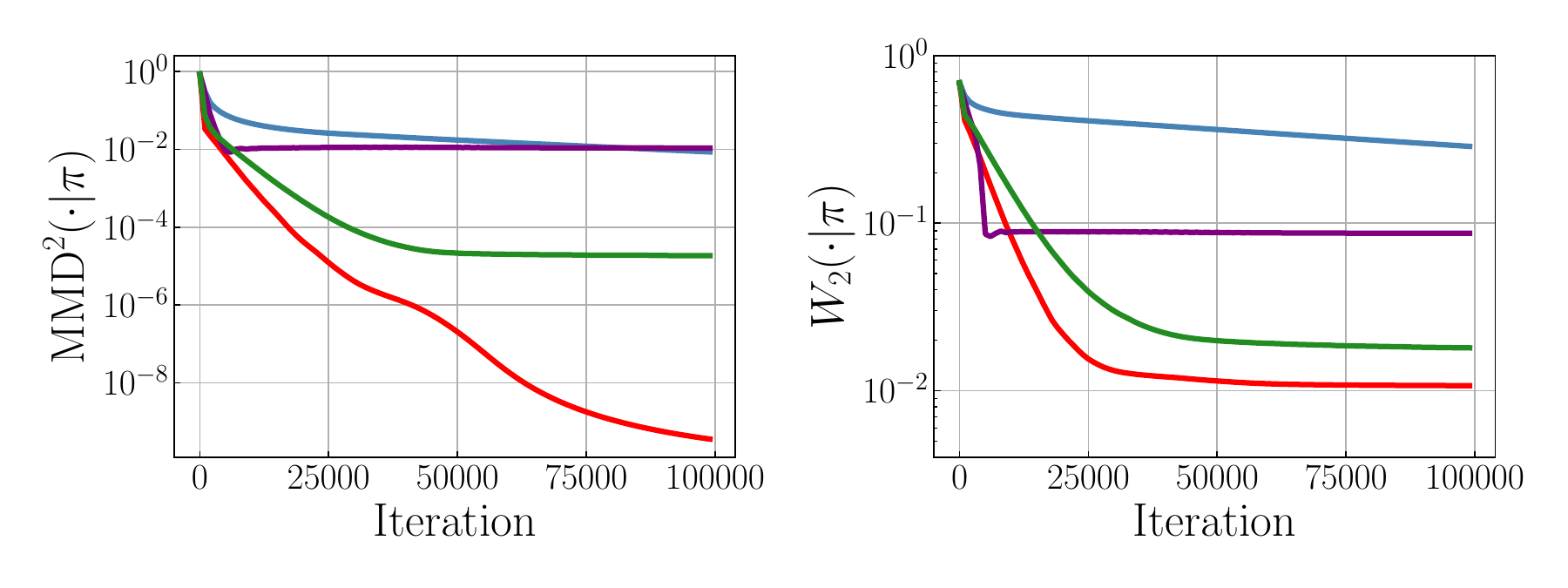}
        \label{fig:three_mmd_wass}
    \end{subfigure}%
    \hspace{-5pt}
    \begin{subfigure}{0.33\textwidth}
        \centering
        \includegraphics[width=1.0\linewidth]{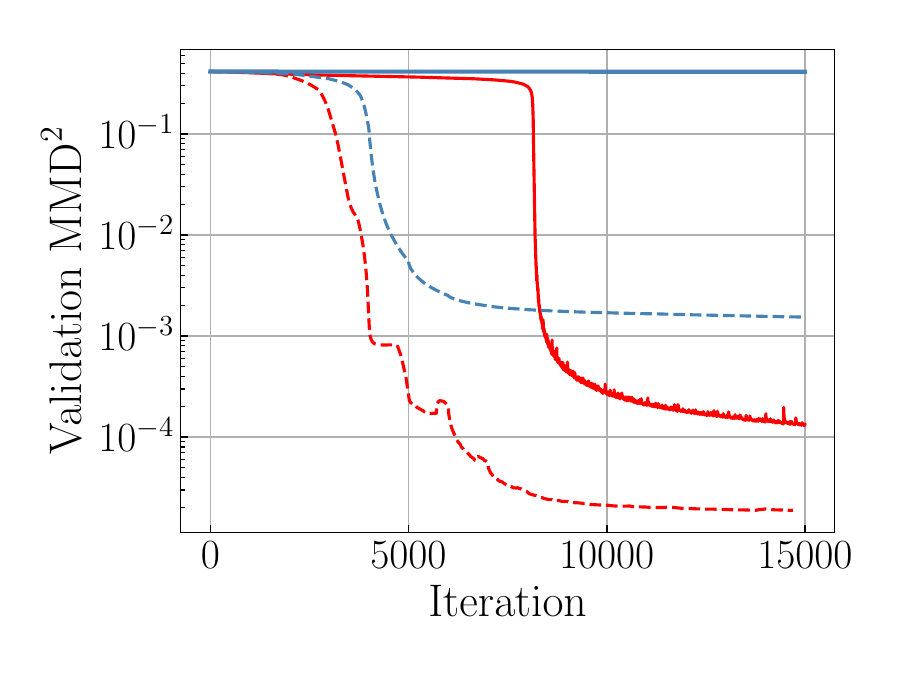}
        \label{fig:student_teacher}
    \end{subfigure}%
    \vspace{-20pt}
    \caption{\textbf{Left \& Middle:} Comparison of $\dmmd$ descent with adaptive $\lambda$, $\mmd$ and $\kale$ descent on three-ring synthetic data in terms of $\mmd$ and Wasserstein-2 distance with respect to the target $\pi$. \textbf{Right:} Comparison of MMD descent with and without noise injection, $\dmmd$ descent with and without noise injection on training
    student/teacher networks in terms of validation $\mmd^2$ distance.}
    \label{fig:three_ring}
    \vspace{-15pt}
\end{figure}

Therefore, the training of the student network with objective \eqref{eq:student_teacher_1} can be treated as an optimization problem of $\mmd^2$ distance in $\calP_2(\R^d)$, i.e., as $\mmd$ gradient flow. 
It is shown in \cite{arbel2019maximum} that $\mmd$ flow and its descent scheme will generally get stuck in  local optima because 
$\mmd$ is not \emph{geodesically} convex; therefore, noise injection has been proposed to escape these local optima. 

With support from Theorems \ref{thm:continuous_time_convergence} and \Cref{thm:discrete_time_KL} on the convergence of $\dmmd$ flow and its descent scheme,  we propose to minimize $\dmmd(\mu\|\pi)$ and use $\dmmd$ descent rather than minimizing  $\mmd(\mu\|\pi)^2$ directly. 
Although this does not directly minimize the objective in \eqref{eq:student_teacher_1}, the favorable convergence performance of $\dmmd$ descent should result in a smaller $\mmd(\mu\|\pi)^2$ at convergence. 

In our experimental setting, we are given $M=10$ particles $x^{(1)}, \cdots, x^{(M)}$ from the teacher distribution $\pi = \calN(0,\Id)$ and $N=1000$ particle $y_0^{(1)}, \cdots, y_0^{(N)}$ from the initial student distribution $\mu_0 = \calN(0, 10^{-3} \Id)$. The teacher particles are fixed while
the student particles are  updated according to \Cref{alg:dmmd_flow} at each time step. 
The initial (de)-regularization parameter is $\lambda_0 = 0.1$, the step size is $\gamma = 0.1$, we apply a lower bound $\tilde{\lambda} = 10^{-3}$, and the regularity $r$ is optimized over the set of $\{ 0.1, 0.5, 1.0\}$.
For the architecture of the neural network, there are $3$ neurons in the hidden layer, and the output dimension is $1$. 
The data distribution $\Pb_{\text{data} }$  is a uniform distribution on the sphere in $\R^p$ with $p = 50$.
$2000$ data are sampled from $\Pb_{\text{data} }$ with $1000$ as training dataset and another $1000$ as validation dataset.
The kernel $k(x, x^\prime) = \mathbb{E}_{z \sim \Pb_{\text{data} } } \left[\psi(z, x)^{\top} \psi\left(z, x^{\prime}\right)\right]$ is estimated by the average over $100$ randomly selected samples from the training dataset at each iteration. 
$\dmmd$ and $\mmd$ descent stop after $n_{max}=15, 000$ iterations when both converge.
The final performance is evaluated in terms of $\mmd^2(\mu_{n_{max}} \| \pi)$ with kernel $k$ estimated by the average of $1000$ samples in the validation dataset.

In \Cref{fig:three_ring} \textbf{Right}, we report the performance of $\mmd$ descent (with and without noise injection) along with the $\dmmd$ descent (with and without noise injection) in terms of $\mmd$ distance on the validation dataset.
We can see that the $\dmmd$ descent does not get stuck in a local optimum, and leads to much lower validation $\mmd(\mu_{n_{max}} \| \pi)^2$ even without noise injection. 
We also run $\dmmd$ descent with the noise injection scheme and find that noise injection can further improve the performance of $\dmmd$ descent and outperforms $\mmd$ descent with noise injection.
Although it is unclear whether the density ratio has enough regularity to meet the condition of \Cref{thm:discrete_time_KL}, the kernel $k$ satisfies the boundedness and smoothness conditions of \Cref{assumption:bounded_kernel} and the target $\pi$ satisfies the Poincar\'{e} inequality since it is Gaussian. The $\dmmd$ descent benefits from more favorable convergence properties, which explains its superior performance.


The code to reproduce all the experiments can be found in the following GitHub repository. \url{https://github.com/hudsonchen/DrMMD}.

\section{Discussion}
In this paper, we introduced (de)-regularization of the MMD (called $\dmmd$) and its associated Wasserstein gradient flow. 
As an interpolation between the $\mmd$ and $\chi^2$-divergence, the $\dmmd$ gradient flow inherits strengths from both sides: it is easy to simulate in closed form with particles, and it has an exponential rate of convergence towards the global minimum up to a controllable barrier term when the target $\pi$ satisfies a Poincar\'{e} inequality.
Additionally, we provide the optimal adaptive selection of a regularization coefficient
that best balances the approximation and time discretization errors in $\dmmd$ gradient descent. 
Our work is the first to prove theoretically that an adaptive kernel through adaptive regularization can result in improved convergence of MMD gradient flow. The theoretical results are consistent with the empirical evidence in several numerical experiments. 

Following our work, there remain a number of interesting open problems. For example, 
(i) Since the kernel bandwidth has been known to play an important role in the performance of kernel-based algorithms, it is of interest to study the adaptive choice of kernel bandwidth in the context of $\dmmd$ gradient flow. (ii) To generalize our convergence analysis to the Wasserstein gradient flow of all Moreau envelopes of $f$-divergences in reproducing kernel Hilbert space, even when they do not have a closed-form expression as $\dmmd$. (iii) While the current work proposes an approximation to the $\chi^2$-squared flow in the generative modeling setting, i.e., where the target distribution $\pi$ is known only through samples, it will be interesting to construct approximations to $\chi^2$-flow in the sampling setting, i.e., where $\pi$ is known in closed form (at least up to normalization).




\section{Proofs} \label{sec:proofs}
In this section, we present the proofs of the results presented in Sections~\ref{sec:chard}--\ref{sec:space_dis}.
\subsection{Proof of \Cref{prop:mmd_chi2}}\label{appsec:mmd_regularized_chi2}
Note that
\begin{align}
    \mmd(\mu \| \pi)^2 &= \left\| m_\mu - m_\pi \right\|_\calH^2 = \left\| \int k(\cdot, x) d\mu(x) - \int k(\cdot, x) d\pi(x) \right\|_\calH^2 \nonumber \\ 
    &=\left\| \int k(\cdot,x) \left(\frac{d\mu}{d\pi}(x)-1\right)\,d\pi(x) \right\|^2_\calH=\left\|\id^*_\pi \left(\frac{d\mu}{d\pi} - 1\right)\right\|^2_\calH\nonumber\\
    &=\left\langle \mathcal{T}_\pi\left(\frac{d\mu}{d\pi} - 1\right),\frac{d\mu}{d\pi} - 1 \right\rangle_{L^2(\pi)}
    = \left\| \calT_\pi^{\frac{1}{2} } \left(\frac{d \mu}{d \pi}-1\right)\right\|_{L^2(\pi)}^2 .\nonumber
\end{align}
Also recall that the $\chi^2$-divergence between $\mu$ and $\pi$ is
\begin{align*}
    \chi^2(\mu \| \pi) = \int \left( \frac{d\mu}{d\pi} - 1 \right)^2 d\pi = \left\|\mathrm{I}\left(\frac{d \mu}{d \pi}-1\right)\right\|_{L^2(\pi)}^2.\nonumber
\end{align*}

\subsection{Proof of \Cref{prop:drmmd_representation_no_ratio}}\label{appsec:proof_chard_also}
Let $\mu \ll \pi$. 
In order to prove the alternative form of $\dmmd$ in \eqref{eq:drmmd_representation_no_ratio}, we start from \eqref{eq:drmmd_representation_no_ratio} and show that it recovers \eqref{eq:dmmd_with_inverse_operator}.
\begin{align}\label{eq:drmmd_derivations}
\dmmd(\mu || \pi) &= ( 1+\lambda) \left\|\left(\Sigma_\pi + \lambda \Id\right)^{-\frac{1}{2}}\left(m_\mu - m_\pi \right) \right\|_{\calH}^2 \nonumber \\
& =(1+\lambda)\left\|\left(\iota_\pi^\ast \iota_\pi + \lambda \Id \right)^{-\frac{1}{2}} \iota_\pi^\ast \left( \frac{d \mu}{ d \pi} - 1 \right) \right\|_{\calH}^2 \nonumber \\
& =(1+\lambda)\left\langle\left(\iota_\pi^\ast \iota_\pi + \lambda \Id \right)^{-\frac{1}{2}} \iota_\pi^\ast \left( \frac{d \mu}{ d \pi} - 1 \right),\left(\iota_\pi^\ast \iota_\pi + \lambda \Id \right)^{-\frac{1}{2}} \iota_\pi^\ast \left( \frac{d \mu}{ d \pi} - 1 \right)\right\rangle_{\calH}\nonumber \\
& =(1+\lambda)\left\langle \iota_\pi \left(\iota_\pi^\ast \iota_\pi + \lambda \Id \right)^{-1} \iota_\pi^\ast \left( \frac{d \mu}{ d \pi} - 1 \right), \frac{d \mu}{ d \pi} - 1 \right\rangle_{L^2(\pi)} \nonumber \\
&=(1+\lambda)\left\langle \iota_\pi \iota_\pi^\ast\left(\iota_\pi \iota_\pi^\ast+\lambda \Id \right)^{-1} \left( \frac{d \mu}{ d \pi} - 1 \right), \frac{d \mu}{ d \pi} - 1 \right\rangle_{L^2(\pi)},  
\end{align}
where the last equality follows by noticing $\iota_\pi \left( \iota_\pi^* \iota_\pi + \lambda \Id \right)^{-1} \iota_\pi^* = \iota_\pi \iota_\pi^* \left( \iota_\pi \iota_\pi^* + \lambda \Id \right)^{-1} $.
Therefore,
\begin{align*}
\quad \dmmd(\mu || \pi) &= (1+\lambda)\left\langle \calT_\pi \left( \calT_\pi + \lambda \Id \right)^{-1} \left( \frac{d \mu}{ d \pi} - 1 \right),  \frac{d \mu}{ d \pi} - 1 \right\rangle_{L^2(\pi)} \\ 
&= (1+\lambda)\left\|\left( \left( \mathcal{T}_\pi+\lambda \mathrm{I}\right)^{-1} \mathcal{T}_\pi\right)^{1 / 2}\left(\frac{d \mu}{d \pi}-1\right)\right\|_{L^2(\pi)}^2,
\end{align*}
which follows from the positivity and self-adjointness of $\calT_\pi \left( \calT_\pi + \lambda \Id \right)^{-1}$.
So \eqref{eq:drmmd_representation_no_ratio} is proved. 
Next, we are going to prove the variational formulation in \eqref{eq:drmmd_variational}. Similarly, we start from \eqref{eq:drmmd_variational} and show it recovers \eqref{eq:dmmd_with_inverse_operator}. Consider
\begin{align}
    &\quad (1+\lambda) \sup _{h \in \calH}\left\{\int h \md \mu-\int\left(\frac{h^2}{4}+h \right) \md \pi - \frac{\lambda}{4}\|h\|_{\calH}^2\right\} \nonumber \\
    & =-(1+\lambda) \inf _{h \in \calH}\left\{\int \frac{h^2}{4} \md \pi - \int h (\md \mu - \md \pi) + \frac{\lambda}{4} \|h\|_{\calH}^2\right\}  \nonumber \\ 
    & = -(1+\lambda) \inf _{h \in \calH}\left\{\frac{1}{4} \left\langle h, \Sigma_\pi h \right \rangle_{\calH} - \left \langle h, m_\mu - m_\pi \right\rangle_{\calH}+\frac{\lambda}{4}\|h\|_{\calH}^2\right\}  \nonumber \\
    &= -(1+\lambda) \inf _{h \in \calH} \left\{\left\|\left(\frac{1}{4} \Sigma_\pi+\frac{\lambda}{4} \Id\right)^{1 / 2} h -\frac{1}{2}\left(\frac{1}{4} \Sigma_\pi+\frac{\lambda}{4} \Id\right)^{-1 / 2} (m_\mu - m_\pi) \right\|_{\calH}^2\right\}  \nonumber \\
    &\qquad\qquad+ \frac{(1+\lambda)}{4}\left\|\left(\frac{1}{4} \Sigma_\pi +\frac{\lambda}{4} \Id\right)^{-1 / 2} (m_\mu - m_\pi) \right\|_{\calH}^2 . \label{appeq:kernel_chi2_supremum}
\end{align}
The last equality follows from completing the squares, based on which it is easy to see that the infimum is achieved at $h_{\mu, \pi}^\ast = 2 \left(\Sigma_\pi + \lambda \Id\right)^{-1}\left( m_\mu - m_\pi \right)$. 
For $\mu \ll \pi$, following the same derivations in \eqref{eq:drmmd_derivations}, $h_{\mu, \pi}^\ast$ can be alternatively expressed as
\begin{align}\label{eq:witness_alternative}
    h_{\mu, \pi}^\ast = 2 \left( \mathcal{T}_\pi+\lambda \mathrm{I}\right)^{-1} \mathcal{T}_\pi \left( \frac{d\mu}{d \pi} - 1 \right)
\end{align}
Plugging $h_{\mu, \pi}^\ast$ back into \eqref{appeq:kernel_chi2_supremum} recovers \eqref{eq:dmmd_with_inverse_operator}, so \eqref{eq:drmmd_variational} is proved.

\subsection{$\dmmd$ is $\mmd$ with a regularized kernel $\tilde{k}$}\label{appsec:drmmd_is_mmd}
Given the definition of $\tilde{k}(x, x^\prime) = \left\langle \left(\Sigma_\pi + \lambda \Id\right)^{-\frac{1}{2} } k(\cdot,x),\left(\Sigma_\pi + \lambda \Id\right)^{-\frac{1}{2} } k(\cdot,x')\right\rangle_\calH$, it is clear that $\tilde{k}$ is symmetric and positive definite so it has a unique associated reproducing kernel Hilbert space $\tilde{\calH}$~\citep[Theorem 4.21]{steinwart2008support} with canonical feature map $\tilde{k}(x, \cdot)$. Therefore,
\begin{align*}
&\quad \dmmd(\mu \| \pi) \\
& = (1 + \lambda) \left\Vert \left(\Sigma_\pi + \lambda \Id\right)^{-\frac{1}{2} } (m_\mu-m_\pi) \right\Vert^2_\calH \\
&= (1 + \lambda) \left\langle \left(\Sigma_\pi + \lambda \Id\right)^{-\frac{1}{2} } \int k(x, \cdot) d\left( \pi - \mu \right)(x),  \left(\Sigma_\pi + \lambda \Id\right)^{-\frac{1}{2} } \int k(x^\prime, \cdot) d\left( \pi - \mu \right)(x^\prime) \right\rangle_\calH \\
&= (1 + \lambda) \left\langle  \int \left(\Sigma_\pi + \lambda \Id\right)^{-\frac{1}{2} } k(x, \cdot) d\left( \pi - \mu \right)(x),  \int \left(\Sigma_\pi + \lambda \Id\right)^{-\frac{1}{2} } k(x^\prime, \cdot) d\left( \pi - \mu \right) (x^\prime) \right\rangle_\calH \\
&= (1 + \lambda) \iint \left\langle  \left(\Sigma_\pi + \lambda \Id\right)^{-\frac{1}{2} } k(x, \cdot) , \left(\Sigma_\pi + \lambda \Id\right)^{-\frac{1}{2} } k(x^\prime, \cdot) \right\rangle_\calH d \left( \pi- \mu \right) (x) d\left( \pi - \mu \right) (x^\prime) \\
&=(1 + \lambda) \iint \tilde{k}(x, x^\prime) d\left( \pi- \mu \right) (x) d\left( \pi - \mu \right) (x^\prime) \\
&= (1 + \lambda) \left\Vert \int \tilde{k}(x, \cdot) d(\mu-\pi)(x) \right \Vert^2_{\tilde{\calH}}.
\end{align*}
In the third and fourth equality above, we are using the fact that $\left(\Sigma_\pi + \lambda \Id\right)^{-\frac{1}{2} } k(x, \cdot) \in \calH$ is Bochner integrable and the Bochner integral preserves inner product structure. So $\dmmd$ is essentially $\mmd^2$ with a different kernel $\tilde{k}$ up to a multiplicative factor of $1 + \lambda$.

Next, we present the Mercer decomposition of $\tilde{k}$. 
Notice that $\{e_i\}_{i \geq 1}$ are the eigenfunctions of $\calT_\pi$, so $\{ \sqrt{\varrho_i} e_i\}_{i \geq 1}$ are the eigenfunctions of $\Sigma_\pi$.
For $x$ and $x^\prime$ in the support of $\pi$, $\tilde{k}$ also enjoys a pointwise convergent Mercer decomposition
\begin{align}\label{appeq:tilde_k_mercer}
    \tilde{k}(x, x^\prime) &= \left\langle \left(\Sigma_\pi + \lambda \Id\right)^{-\frac{1}{2} } k(\cdot,x),\left(\Sigma_\pi + \lambda \Id\right)^{-\frac{1}{2} } k(\cdot,x')\right\rangle_\calH \nonumber \\
    &= \left\langle \left(\Sigma_\pi + \lambda \Id\right)^{-\frac{1}{2} } \left( \sum_{i \geq 1} \varrho_i e_i(x) e_i \right), \left(\Sigma_\pi + \lambda \Id\right)^{-\frac{1}{2} } \left( \sum_{i \geq 1} \varrho_i e_i(x^\prime) e_i \right) \right\rangle_\calH \nonumber \\
    &= \left\langle \sum_{i \geq 1} \frac{\varrho_i}{\sqrt{\varrho_i + \lambda} } e_i(x) e_i , \sum_{i \geq 1} \frac{\varrho_i}{\sqrt{\varrho_i + \lambda} } e_i(x^\prime) e_i \right\rangle_\calH \nonumber \\
    &= \sum_{i \geq 1} \frac{\varrho_i}{\varrho_i+\lambda} e_i(x) e_i(x^{\prime}).
\end{align}
More properties of the regularized kernel $\tilde{k}$ are provided in \Cref{lem:k_tilde}.

\subsection{Proof of \Cref{prop:chard_interpolation}}\label{appsec:proof_chard_interpolation}
Given that $\frac{d \mu}{ d \pi} - 1 \in L^2(\pi)$, so 
\begin{align*}
    \dmmd(\mu \| \pi) = (1+\lambda) \left\|\left(\left(\mathcal{T}_\pi + \lambda \mathrm{I} \right)^{-1} \mathcal{T}_\pi\right)^{1 / 2}\left(\frac{d \mu}{d \pi}-1\right)\right\|_{L^2(\pi)}^2 \leq (1+\lambda) \left\|\frac{d \mu}{ d \pi} - 1 \right\|_{L^2(\pi)}^2,
\end{align*}
which is finite for any $\lambda \geq 0$. 
We are allowed to interchange the limit and integration according to the dominated convergence theorem~\citep{rudin1976principles} to achieve, 
\begin{align*}
    \lim_{\lambda \to 0} \dmmd(\mu \| \pi) &=  \left\| \lim_{\lambda \to 0} \left(\left(\mathcal{T}_\pi + \lambda \mathrm{I} \right)^{-1} \mathcal{T}_\pi\right)^{1 / 2}\left(\frac{d \mu}{d \pi}-1\right)\right\|_{L^2(\pi)}^2 \\
    &= \left\|  \frac{d \mu}{ d \pi} - 1 \right\|_{L^{2}(\pi)}^{2} = \chi^{2}(\mu \| \pi).
\end{align*}
From \eqref{eq:drmmd_representation_no_ratio}, we have that,
\begin{align}
    &\sqrt{\dmmd (\mu \| \pi)} 
    = \sqrt{1+\lambda} \left\|\left(\Sigma_{\pi} + \lambda \Id\right)^{-\frac{1}{2} }(m_{\mu} - m_{\pi})  \right\|_{\calH} \nonumber \\
    &\leq \sqrt{1+\lambda} \left\|\left(\Sigma_{\pi} + \lambda \Id\right)^{-\frac{1}{2} }\right\|_{op}\left\|m_{\mu} - m_{\pi}  \right\|_{\calH}
    \leq \sqrt{\frac{1+\lambda}{\lambda}} \times \mmd (\mu,\pi) , \label{appeq:dmmd_one_tide}
\end{align}
where the last inequality follows by noticing that $\Sigma_\pi = \iota_\pi^\ast \iota_\pi$ shares the same 
eigenvalues as $\calT_\pi = \iota_\pi \iota_\pi^\ast$, and hence the eigenvalues of $(\Sigma_\pi + \lambda \Id)^{-1}$ are ${(\varrho_i + \lambda)}^{-1}$ which all smaller than $\frac{1}{\lambda}$.
Therefore, $\dmmd(\mu \| \pi) \leq \frac{1+\lambda}{\lambda} \mmd^{2}(\mu \| \pi)$. 

\vspace{1em}
\noindent
On the other hand, using Lemma A.10 from \citep{hagrass2022spectral}, we have
\begin{align}
    &\mmd \left(\mu \| \pi \right) = \left\|m_\mu - m_\pi \right\|_{\calH} 
    \leq \left\| ( \Sigma_\pi+\lambda \Id\right ) ^{\frac{1}{2} } \|_{o p} \left\|\left(\Sigma_{\pi} + \lambda \Id\right)^{-\frac{1}{2} }(m_{\mu} - m_{\pi})  \right\|_{\calH} \nonumber  \\
    &= \sqrt{ \frac{1}{1+\lambda} }  \left\| ( \Sigma_\pi+\lambda \Id\right ) ^{\frac{1}{2} } \|_{o p} \sqrt{\dmmd \left(\mu \| \pi \right)} 
    \leq \sqrt{ \frac{K + \lambda}{1+\lambda} } \sqrt{ \dmmd \left(\mu \| \pi \right)} , \label{appeq:dmmd_tecond_tide}
\end{align}
where $K$ is the upper bound on kernel $k$ in \Cref{assumption:bounded_kernel} and hence an upper bound on the operator norm of $\Sigma_\pi$.
Combining \eqref{appeq:dmmd_one_tide} and \eqref{appeq:dmmd_tecond_tide}, we have 
\begin{align*}
    \frac{1 + \lambda}{K + \lambda} \mmd^{2}(\mu \| \pi) \leq \dmmd \left(\mu \| \pi \right) \leq \frac{1 + \lambda}{\lambda} \mmd^{2}(\mu \| \pi) .
\end{align*}
Therefore, $\lim_{\lambda \to \infty} \dmmd(\mu \| \pi) = \mmd^2 (\mu \| \pi) $, and the proposition is proved.

\subsection{Proof of \Cref{prop:topology_chard}}\label{appsec:proof_of_topology_chard}
To show that $\dmmd$ is a probability divergence, we need to show that $\dmmd$ enjoys non-negativity and definiteness.
It is easy to see that $\dmmd(\mu \| \pi)$ is non-negative from its definition in \Cref{def:chard}.
Then, we prove definiteness, i.e., $\dmmd(\mu || \pi)=0 $ if and only if $ \mu = \pi$.
For the first direction, assume $\dmmd(\mu \| \pi)=0$, so $\| (\Sigma_\pi+\lambda \Id)^{-1 / 2} \left(m_\mu-m_\pi\right)\|_{\calH}^2=0$. 
Since $\left(\Sigma_\pi+\lambda \Id\right)^{-1 / 2}$ is a non-singular operator, we must have that $m_\mu=m_\pi$ which implies $\mu=\pi$ as $k$ is $c_0$-universal and hence characteristic~\citep{sriperumbudur2011universality}. 
For the other direction, when $\mu = \pi$, immediately we can see $\dmmd(\mu || \pi) = 0$.

Then we prove that $\dmmd$ metrizes weak convergence. 
For the first direction, we know from \eqref{appeq:dmmd_one_tide} that $\dmmd(\mu_n \| \pi) \leq \frac{1 + \lambda}{\lambda} \mmd^2(\mu_n \| \pi)$ and $\mmd^2(\mu_n \| \pi) \to 0$ as $\mu_n$ converges weakly to $\pi$~\citep{simon2023metrizing}.
For the converse direction, we assume that $\dmmd\left(\mu_n \| \pi \right) \to 0$.
From \eqref{appeq:dmmd_tecond_tide}, we know that $\mmd^2 \left(\mu \| \pi \right) \leq \frac{K + \lambda}{1+\lambda} \dmmd \left(\mu \| \pi \right) $, therefore $\dmmd\left(\mu_n \| \pi \right) \to 0$ implies $\mmd \left(\mu_n \| \pi \right) \to 0$, implying the weak convergence of $\mu_n$ to $\pi$, if $k$ is characteristic~\citep{simon2023metrizing}.

\subsection{Proof of \Cref{prop:chard_gradient_flow}}\label{appsec:proof_chard_gradient_flow}
In order to show that $\calF_{\dmmd}(\cdot) = \dmmd(\cdot \|\pi)$ admits a well-defined gradient flow, we follow the same techniques in Proposition 7 of \cite{glaser2021kale} and Lemma B.2 of \cite{chizat2018global}, where the key is to show that $(1 + \lambda) \nabla h_{\mu, \pi}^\ast$ is the Fr\'{e}chet subdifferential of $\calF_{\dmmd}$ evaluated at $\mu$.\footnote{Although $\dmmd$ can be viewed as squared $\mmd$ with a regularized kernel $\tilde{k}$, we are not using the technique in \cite{arbel2019maximum} because it relies on Lemma 10.4.1 of \cite{ambrosio2005gradient} which only provides the Fr\'{e}chet subdifferential on probability measures $\mu$ that admit density functions. To construct the Wasserstein gradient flow of $\calF_{\dmmd}$ up to full generality, we resort to the techniques of \cite{glaser2021kale} and \cite{chizat2018global} instead.} 
According to Definition 10.1.1 of \cite{ambrosio2005gradient}, it is equivalent to prove that, for any $\mu \in \calP_2(\R^d)$ and $\phi \in C_c^\infty(\R^d)$, 
\begin{align}\label{appeq:target_subdiff}
&\dmmd \left( (\Id + \nabla \phi)_\# \mu \| \pi \right) - \dmmd \left(\mu \| \pi\right) \geq (1 + \lambda) \int \nabla \phi(x)^\top \nabla h_{\mu, \pi}^\ast(x) d \mu(x) \nonumber\\
&\qquad\qquad\qquad\qquad\qquad+ o\left( \| \nabla \phi \|_{L^2(\mu)} \right) .
\end{align}
Define $\rho_t = (\Id + t \nabla \phi)_\# \mu $, $\varphi_t : \R^d \to \R^d, x \mapsto x + t \nabla \phi(x)$ and $g(t) = \dmmd( \rho_t || \pi)$. Then from \Cref{lem:wass_gradient_hess_chard} we know that
$g(t)$ is continuous and differentiable with respect to $t$ and
\begin{align*}
    \frac{d}{dt}\Big|_{t=0} g(t)  = (1 + \lambda) \int \nabla \phi(x)^\top \nabla h_{\mu, \pi}^\ast(x) d\mu(x) .
\end{align*}
Since $t \mapsto g(t)$ is differentiable, using Taylor's theorem and mean value theorem \citep{rudin1976principles}, we know that there exists $0 < \kappa < 1$ such that
\begin{align*}
&\quad \dmmd \left( (\Id + \nabla \phi)_\# \mu \| \pi \right) - \dmmd \left(\mu \| \pi\right) = g(1) - g(0) = \frac{d}{dt}\Big|_{t=0} g(t)  + \frac{d^2}{dt^2} \Big|_{t=\kappa} g(t).
\end{align*}
Therefore, to prove \eqref{appeq:target_subdiff}, the goal is to prove that $\frac{d^2}{dt^2} \big|_{t=\kappa} g(t) \geq o\left( \| \nabla \phi \|_{L^2(\mu)} \right)$. To this end, since we know from \Cref{lem:wass_gradient_hess_chard} that $t \mapsto \frac{d}{dt} g(t)$ is continuous and differentiable, we have 
\begin{small}
\begin{align*}
    &\quad \frac{d^2}{d t^2} \dmmd(\rho_t||\pi) = 2 (1+\lambda) \iint \nabla \phi(x)^{\top} \nabla_1 \nabla_2 \tilde{k}\left(\varphi_t(x), \varphi_t(y)\right) \nabla \phi(y) d \mu(x) d \mu(y) \\
    &+ 2 (1+\lambda) \int \nabla \phi(x)^\top \left( \int\bH_1 \tilde{k} \left( \varphi_t(x), \varphi_t(y) \right) d \mu(y) -  \int \bH_1 \tilde{k} \left( \varphi_t(x), y \right) d \pi(y) \right) \nabla \phi(x) d\mu(x),
\end{align*}
\end{small}
where $\tilde{k}$ is the regularized kernel defined in \eqref{eq:tilde_k} and $\tilde{\calH}$ is the associated RKHS.
Using \Cref{lem:k_tilde}, the first term above can be upper-bounded by,
\begin{align*}
    &\quad  \left|2(1+\lambda) \iint \nabla \phi(x)^{\top} \nabla_1 \nabla_2 \tilde{k}\left(\varphi_t(x), \varphi_t(y)\right) \nabla \phi(y) d \mu(x) d \mu(y) \right| \\
    &\leq 2(1+\lambda) \iint \left\| \nabla \phi(x) \right\| \left\|  \nabla_1 \nabla_2 \tilde{k}\left(\varphi_t(x), \varphi_t(y)\right) \right\|_F \left\| \nabla \phi(y) \right\| d \mu(x) d \mu(y) \\
    &\leq 2(1+\lambda) \frac{K_{1d}}{\lambda} \left( \int \left\| \nabla \phi(y) \right\| d \mu(y) \right)^2 
    \leq 2(1+\lambda) \frac{K_{1d}}{\lambda} \|\nabla \phi \|_{L^2(\mu)}^2 .
\end{align*}
Using \Cref{lem:k_tilde} again, the second term can be upper-bounded by
\begin{align*}
    &\quad \left| \int \nabla \phi(x)^\top \left( \int\bH_1 \tilde{k} \left( \varphi_t(x),  \varphi_t(y) \right) d \mu(y) -  \int \bH_1 \tilde{k} \left( \varphi_t(x), y \right) d \pi(y) \right) \nabla \phi(x) d\mu(x) \right| \\
    &\leq \left| \int \nabla \phi(x)^\top \left( \int \bH_1 \tilde{k} \left( \varphi_t(x),  \varphi_t(y) \right) d \mu(y) \right) \nabla \phi(x) d\mu(x) \right| \\
    &\qquad\qquad+ \left| \int \nabla \phi(x)^\top \left( \int \bH_1 \tilde{k} \left( \varphi_t(x), y  \right) d \pi(y) \right) \nabla \phi(x) d\mu(x) \right| \\
    &\leq \int \left\| \nabla \phi(x) \right\| \left( \int \left\|  \bH_1 \tilde{k} \left( \varphi_t(x),  \varphi_t(y) \right) \right\|_F d \mu(y) \right) \left\|\nabla \phi(x) \right\| d\mu(x) \\
    &\qquad\qquad+ \int \left\| \nabla \phi(x) \right\| \left( \int \left\|  \bH_1 \tilde{k} \left( \varphi_t(x), y \right) \right\|_F d \pi(y) \right) \left\|\nabla \phi(x) \right\| d\mu(x) \\
    &\leq 2 \frac{\sqrt{K K_{2d}}}{\lambda} \|\nabla \phi \|_{L^2(\mu)}^2 .
\end{align*}
Combining the above two inequalities to have
\begin{align}\label{appeq:dmmd_hessian_bounds}
    \left| \frac{d^2}{d t^2} \dmmd(\rho_t||\pi) \right| \leq 2(1+\lambda) \frac{K_{1d}}{\lambda} \|\nabla \phi \|_{L^2(\mu)}^2 + 4(1+\lambda) \frac{\sqrt{K K_{2d}}}{\lambda} \|\nabla \phi \|_{L^2(\mu)}^2 . 
\end{align}
Therefore, we have $\frac{d^2}{d t^2}\big|_{t=\kappa} g(t) = \calO(\|\nabla \phi \|_{L^2(\mu)}^2)=o(\|\nabla \phi \|_{L^2(\mu)})$ as $\|\nabla \phi \|_{L^2(\mu)} \to 0$.
So \eqref{appeq:target_subdiff} is proved, which means that $(1 + \lambda) \nabla h_{\mu, \pi}^\ast(x)$ is the Fr\'{e}chet subdifferential of $\calF_{\dmmd}$ evaluated at $\mu$.
According to Definition 11.1.1 of \cite{ambrosio2005gradient}, there exists a solution $(\mu_t)_{t\geq 0}$ such that the following equation holds in the sense of distributions,
\begin{align*}
    \partial_t \mu_t - (1 + \lambda) \nabla \cdot \left(\mu_t  \nabla h_{\mu_t, \pi}^\ast \right)=0,
\end{align*}
and such $(\mu_t)_{t\geq 0}$ is indeed the $\dmmd$ gradient flow, so existence is proved.

Next, we are going to prove uniqueness. 
\eqref{appeq:dmmd_hessian_bounds} indicates that $\calF_{\dmmd}$ is geodesically $-(1+\lambda) \frac{ 4 \sqrt{K K_{2 d}} + 2 K_{1 d}}{\lambda}$ semiconvex. Therefore the uniqueness of $(\mu_t)_t$ follows from Theorem 11.2.1 of \cite{ambrosio2005gradient}.
\eqref{eq:dmmd_semi_convex} is proved in \eqref{appeq:dmmd_hessian_bounds}.

\subsection{Proof of \Cref{prop:geometric_prop}} \label{appsec:proof_geometric}
Given $\mu \in \calP_2\left(\mathbb{R}^d\right)$ and $\phi \in C_c^\infty\left(\mathbb{R}^d\right)$, consider the path $(\rho_t)_{0 \leq t \leq 1}$ from $\mu$ to $(\Id + \nabla \phi)_{\#} \mu$ given by $\rho_t = (\Id + t \nabla \phi)_{\#} \mu$. $(\rho_t)_{0 \leq t \leq 1}$ is a constant-time geodesic in the Wasserstein-2 space by construction (\Cref{appsec:wass_discussion}).
Define $\varphi_t: \R^d \to \R^d, x \mapsto x + t \nabla \phi(x)$. We know from \citep[Example 15.9]{villani2009optimal} (by taking $m=2$) that, 
\begin{align}\label{appeq:hess_chi2}
    \frac{d^2}{d t^2}\Big|_{t=0} \chi^2(\rho_t \| \pi) &= \int \frac{d \mu}{ d \pi} (x) \left( \nabla V(x)^\top \nabla \phi(x) - \nabla \cdot \nabla \phi(x) \right)^2 d\mu(x) \nonumber \\
    &\qquad+ \int \frac{d \mu}{ d \pi} (x) \nabla \phi(x)^\top \bH V(x) \nabla \phi(x) d\mu(x) \nonumber \\ 
    &\qquad\qquad+ \int \frac{d \mu}{ d \pi} (x) \left\| \bH \phi(x) \right\|_F^2 d \mu(x) .
\end{align}
$V$ is twice differentiable, so $\nabla V$ as a function from $\R^d$ to $\R^d$ is continuous.
$\phi \in C_c^\infty(\R^d)$ has compact support, so $x \mapsto \nabla V(x)^\top \nabla \phi(x) - \nabla \cdot \nabla \phi(x)$ is a continuous function over a compact domain, so its image is also compact and hence bounded. 
Using similar arguments, $x \mapsto \nabla \phi(x)^\top \bH V(x) \nabla \phi(x)$ and $x \mapsto \| \bH \phi(x) \|_F^2$ are also continuous functions over compact domains, so they are all bounded. 

Since $\frac{d \mu}{d \pi} - 1 \in \calH \subset L^2(\pi)$ and $\int \frac{d \mu}{ d \pi} (x) d \mu(x) = \| \frac{d \mu}{d \pi} - 1 \|_{L^2(\pi)}^2 + 1 < \infty$, we have
\begin{align}\label{appeq:hess_chi2_bounded}
    \frac{d^2}{d t^2}\big|_{t=0} \chi^2(\rho_t \| \pi) < \infty .
\end{align}
Next, from \Cref{lem:wass_chi2} we know that $\frac{d^2}{d t^2}\big|_{t=0} \chi^2(\rho_t \| \pi)$ can be alternatively expressed as,
\begin{align}\label{appeq:hessian_chi2_copy}
    \frac{d^2}{d t^2}\Big|_{t=0} \chi^2(\rho_t \| \pi) &= 2 \int \left( \nabla \cdot \left( \nabla \phi(x) \mu(x) \right) \frac{1}{\pi(x)} \right)^2 \pi(x) d x \nonumber \\
    & \qquad\qquad+ 2 \int \nabla \phi(x)^\top \bH \left(\frac{\mu(x) }{\pi(x)}\right) \nabla \phi(x) \mu(x) d x .
\end{align}
Recall from \Cref{lem:wass_gradient_hess_chard} that
\begin{align}\label{appeq:chard_wass_hessian_copy}
    &\quad \frac{d^2}{d t^2} \Big|_{t=0} \dmmd(\rho_t||\pi) \nonumber \\
    &=  2 (1+\lambda) \iint \nabla \phi(x)^{\top} \nabla_1 \nabla_2 \tilde{k}\left(x, y \right) \nabla \phi(y) d \mu(x) d \mu(y) \nonumber \\
    &\qquad+ 2 (1+\lambda) \int \nabla \phi(x)^\top \left( \int\bH_1 \tilde{k} \left( x, y \right) d \mu(y) -  \int \bH_1 \tilde{k} \left( x, y \right) d \pi(y) \right) \nabla \phi(x) d\mu(x) .
\end{align}
To prove \eqref{eq:hessian_lower_bound}, our aim then is to compare and bound the difference of $\frac{d^2}{d t^2}\big|_{t=0} \chi^2(\rho_t \| \pi)$ in \eqref{appeq:hessian_chi2_copy} and $\frac{1}{1+\lambda} \frac{d^2}{d t^2}\big|_{t=0} \dmmd(\rho_t ||\pi)$ in \eqref{appeq:chard_wass_hessian_copy}, so we compare and bound their first and second term separately. 

The first term of \eqref{appeq:hessian_chi2_copy} can be rewritten as
\begin{align}
&\int \left(\nabla \cdot( \nabla \phi(x) \mu(x) ) \frac{1}{\pi(x)} \right)^2 \pi(x) d x = \sum_{i \geq 1} \PSi{\nabla \cdot( \nabla \phi \mu ) \frac{1}{\pi}, e_i}{L^2(\pi)}^2 \nonumber \\
&= \sum_{i \geq 1} \left( \int \nabla \cdot \left( \nabla \phi(x) \mu(x) \right) e_i(x) dx \right)^2 
= \sum_{i \geq 1} \left( \int \nabla \phi(x)^\top \nabla e_i(x) \mu(x)  dx \right)^2,
\label{appeq:chi2_hessian_first_term}
\end{align}
where we use integration by parts in the last line since $\phi \in C_c^\infty(\R^d)$.
And the first term of \eqref{appeq:chard_wass_hessian_copy} after rescaling by $\frac{1}{1 + \lambda}$ can be rewritten as,
\begin{align}
&\quad 2 \iint \nabla \phi(x)^{\top} \nabla_1 \nabla_2 \tilde{k}\left(x , y \right) \nabla \phi(y) d \mu(x) d \mu(y) \nonumber \\
&= 2 \iint \nabla \phi(x)^{\top} \nabla_x \nabla_y \left( \sum_{i \geq 1} \frac{\varrho_i}{\varrho_i + \lambda} e_i(x) e_i(y) \right) \nabla \phi(y) d \mu(x) d \mu(y) \nonumber \\
&=  2 \iint \Big( \nabla \cdot \left( \nabla \phi(x) \mu(x)  \right) \Big) \Big( \nabla \cdot \left( \nabla \phi(y) \mu(y) \right) \Big) \sum_{i \geq 1} \frac{\varrho_i}{\varrho_i + \lambda} e_i(x) e_i(y) dx dy .
\label{appeq:infinite_tum_integral}
\end{align}
Since $\pi \ll \calL^d$ so \eqref{appeq:tilde_k_mercer} is true for all $x, y \in \R^d$, hence the second equality is true, and the last equality uses integration by parts since $\phi \in C_c^\infty(\R^d)$. Notice that
\begin{align*}
    &\quad \sum_{i\geq 1} \frac{\varrho_i}{\varrho_i + \lambda} \iint \left| e_i(x) e_i(y) \Big( \nabla \cdot ( \nabla \phi(x) \mu(x)) \Big) \Big( \nabla \cdot ( \nabla \phi(y) \mu(y) ) \Big) \right| dx dy \\
    &= \sum_{i\geq 1} \frac{\varrho_i}{\varrho_i + \lambda} \left(  \int \left| \Big( \nabla \cdot \left( \nabla \phi(x) \mu(x)  \right) \Big) e_i(x) \right| dx \right)^2 \\
    &\leq \frac{1}{\lambda} \sum_{i\geq 1} \varrho_i \left( \int \left( \frac{\nabla \cdot \left( \nabla \phi(x) \mu(x)  \right)}{\mu(x)} \right)^2 \mu(x) dx \right) \left( \int e_i(x)^2 \mu(x) dx \right)  \\
    &= \frac{1}{\lambda} \left( \int \left( \nabla \log \mu(x)^\top \nabla \phi(x) + \nabla \cdot \nabla \phi(x) \right)^2 d \mu(x) \right) \left( \int \sum_{i\geq 1} \varrho_i e_i(x)^2 d \mu(x) \right) \\
    &\leq \frac{K}{\lambda} \left( \int \left( \nabla \log \mu( x )^\top \nabla \phi(x) + \nabla \cdot \nabla \phi(x) \right)^2 d\mu(x) \right) 
    < +\infty .
\end{align*}
The first inequality uses Cauchy-Schwartz, the second inequality uses $\sum_{i \geq 1} \varrho_i (e_i(x) )^2 = k(x,x) \leq K$. 
The last quantity is finite because $x \to \nabla \log \mu( x )$ is a continuous function and $\phi \in C_c^\infty(\R^d)$ has compact support, hence the integral of a continuous function over a compact domain is always finite. 
Then, by using Fubini's theorem~\citep{rudin1976principles}, 
we are allowed to interchange the infinite sum and integration of \eqref{appeq:infinite_tum_integral} to reach,
\begin{align*}
    &\quad 2 \iint \Big( \nabla \cdot \left( \nabla \phi(x) \mu(x)  \right) \Big) \Big( \nabla \cdot \left( \nabla \phi(y) \mu(y) \right) \Big) \sum_{i \geq 1} \frac{\varrho_i}{\varrho_i + \lambda} e_i(x) e_i(y) dx dy \nonumber \\
    &= 2 \sum_{i \geq 1} \frac{\varrho_i}{\varrho_i + \lambda} \iint \Big( \nabla \cdot \left( \nabla \phi(x) \mu(x)  \right) \Big) \Big( \nabla \cdot \left( \nabla \phi(y) \mu(y) \right) \Big) e_i(x) e_i(y) dx dy \nonumber \\
    &= 2 \sum_{i \geq 1} \frac{\varrho_i}{\varrho_i + \lambda} \left(  \int \Big( \nabla \cdot \left( \nabla \phi(x) \mu(x)  \right) \Big) e_i(x)  dx \right)^2 \nonumber \\
    &= 2 \sum_{i \geq 1} \frac{\varrho_i}{\varrho_i + \lambda} \left( \int \nabla \phi(x)^\top \nabla e_i(x) \mu(x)  dx \right)^2,
\end{align*}
where the last equality uses integration by parts.

So the difference between the first term of \eqref{appeq:hessian_chi2_copy} and \eqref{appeq:chard_wass_hessian_copy} rescaled by $\frac{1}{1+\lambda}$ is,
\begin{align}
&\quad 2 \left| \sum_{i \geq 1} \left( \int \nabla \phi(x)^\top \nabla e_i(x) \mu(x)  dx \right)^2 - \sum_{i \geq 1} \frac{\varrho_i}{\varrho_i + \lambda} \left( \int \nabla \phi(x)^\top \nabla e_i(x) \mu(x)  dx \right)^2 \right| \nonumber \\
&= 2 \sum_{i \geq 1} \frac{\lambda}{\varrho_i + \lambda} \left( \int \nabla \phi(x)^\top \nabla e_i(x) \mu(x)  dx \right)^2.
\label{appeq:diff_first_term}
\end{align}
Now we turn to the second term. 
The second term of \eqref{appeq:hessian_chi2_copy} can be rewritten as
\begin{align}
&\quad 2 \int \nabla \phi(x)^\top \bH \left( \frac{d \mu}{ d \pi} (x) \right) \nabla \phi(x) \mu(x) dx \nonumber \\
&= 2 \int \nabla \phi(x)^\top \bH \left( \frac{d \mu}{ d \pi} (x) - 1 \right) \nabla \phi(x) \mu(x) dx \nonumber \\
&= 2 \int \nabla \phi(x)^\top \bH \left( \sum_{i \geq 1} \PSi{\frac{d \mu}{d \pi}-1, e_i}{L^2(\pi)} e_i(x) \right) \nabla \phi(x) \mu(x) dx, \label{appeq:chi2_hessian_tecond_term}
\end{align}
and the second term of \eqref{appeq:chard_wass_hessian_copy} rescaled by $\frac{1}{1+\lambda}$ can be rewritten as,
\begin{align}
&\quad 2 \int \nabla \phi(x)^\top \left( \int\bH_1 \tilde{k} \left(x, y\right) d \mu(y) -  \int \bH_1 \tilde{k} \left(x, y\right) d \pi(y) \right) \nabla \phi(x) d\mu(x) \nonumber \\
&= 2 \int \nabla \phi(x)^\top \bH \left( \int \tilde{k} \left(x, y\right) d \mu(y) -  \int \tilde{k} \left(x, y\right) d \pi(y) \right) \nabla \phi(x) d\mu(x) \nonumber \\
&= 2 \int \nabla \phi(x)^\top  \bH \left( \sum_{i \geq 1} \frac{\varrho_i}{\varrho_i + \lambda} e_i(x) \int e_i(y)d (\mu- \pi)(y)  \right) \nabla \phi(x) d\mu(x) \nonumber \\
&= 2 \int \nabla \phi(x)^\top  \bH \left( \sum_{i \geq 1} \frac{\varrho_i}{\varrho_i + \lambda} \PSi{\frac{d \mu}{d \pi} - 1, e_i}{L^2(\pi)} e_i(x)  \right) \nabla \phi(x) d\mu(x) . \label{appeq:chard_hessian_tecond_term}
\end{align}
Since $\pi \ll \calL^d$, so \eqref{appeq:tilde_k_mercer} is true for all $x, y \in \R^d$ hence the third equality is true.
From  \Cref{lem:k_tilde}, we have $\int \tilde{k}(x,y) d\mu(y) \leq \frac{K}{\lambda}$, $x\mapsto \tilde{k}(x,y)$ is second-order differentiable, $\sup_{x} |\bH_1 \tilde{k}(x,y)| \leq \frac{\sqrt{K K_{2d}}}{\lambda}$. So we are allowed to interchange integration and Hessian in the second equality using the differentiation lemma~\citep[Theorem 6.28]{klenke2013probability}.
Consider the difference of \eqref{appeq:chard_hessian_tecond_term} and \eqref{appeq:chi2_hessian_tecond_term}, we have
\begin{align}
& \quad 2 \Bigg| \int \nabla \phi(x)^\top \bH \left( \sum_{i \geq 1} \PSi{\frac{d \mu}{d \pi}-1, e_i}{L^2(\pi)} e_i(x) \right) \nabla \phi(x) \mu(x) dx \nonumber \\
&\qquad\qquad- \int \nabla \phi(x)^\top  \bH \left( \sum_{i \geq 1} \frac{\varrho_i}{\varrho_i + \lambda} \PSi{\frac{d \mu}{d \pi} - 1, e_i}{L^2(\pi)} e_i(x)  \right) \nabla \phi(x) d\mu(x) \Bigg| \nonumber \\
&= 2 \left| \int \nabla \phi(x)^\top \bH \left( \sum_{i \geq 1}\frac{\lambda}{\varrho_i + \lambda} \PSi{\frac{d \mu}{d \pi}-1, e_i}{L^2(\pi)} e_i(x) \right) \nabla \phi(x) \mu(x) dx \right| \label{appeq:chi2_drmmd_diff}.
\end{align}
Given $\frac{d \mu}{d\pi} - 1 \in \calH$, there exists $q \in L^2(\pi)$ such that $\frac{d \mu}{d\pi} - 1 =\mathcal{T}^{1/2}_\pi q $ so that $\langle \frac{d \mu}{d\pi} - 1, e_i \rangle = \varrho_i^{1/2} \langle q, e_i \rangle$ for all $i$. 
For $j, r \in \{1, \cdots, d\}$, we have 
\begin{align*}
    g_{M_0}(x) &:= \left| \sum_{i \geq M_0} \frac{\lambda}{\varrho_i + \lambda} \PSi{\frac{d \mu}{d \pi} - 1, e_i}{L^2(\pi)} \partial_j \partial_r e_i(x) \right| =  \left| \sum_{i \geq M_0} \frac{\lambda}{\varrho_i + \lambda} \varrho_i^{1/2} \PSi{q, e_i}{L^2(\pi)} \partial_j \partial_r e_i(x) \right|  \nonumber \\
    & \leq \left( \sum_{i \geq M_0} \left( \frac{\lambda}{\varrho_i + \lambda} \right)^2 \PSi{q, e_i}{L^2(\pi)}^2 \right)^{1/2}
    \left( \sum_{i \geq M_0} \varrho_i \left( \partial_j \partial_r e_i(x) \right)^2 \right)^{1/2}  \nonumber \\
    &\leq \left( \sum_{i \geq M_0} \PSi{q, e_i}{L^2(\pi)}^2 \right)^{1/2}
    \left( \sum_{i \geq M_0} \varrho_i  \left( \partial_j \partial_r e_i(x) \right)^2 \right)^{1/2}  
    \leq \sqrt{ K_{2d} } \left( \sum_{i \geq M_0} \PSi{q, e_i}{L^2(\pi)}^2 \right)^{1/2}. 
\end{align*}
The final inequality holds because,
\begin{align}\label{appeq:sum_kernel_hessian}
    \sum_{j,r = 1}^d \sum_{i \geq 1} \varrho_i \left( \partial_j \partial_r e_i(x) \right)^2 &\leq \sum_{j,r = 1}^d \sum_{i \geq 1} \varrho_i \PSi{\partial_j \partial_r k(x, \cdot), e_i}{\calH}^2 = \sum_{j,r = 1}^d \sum_{i \geq 1} \PSi{\partial_j \partial_r k(x, \cdot), \sqrt{\varrho_i} e_i}{\calH}^2 \nonumber \\
    &= \sum_{j,r = 1}^d \left\| \partial_j \partial_r k(x, \cdot) \right\|_\calH^2 = \left\| \bH_1 k(x,\cdot) \right\|_{\calH^{d \times d}}^2 \leq K_{2d} .
\end{align}
Since $\left\| sdq \right\|_{L^2(\pi)}$ is bounded, so $\sum_{i \geq M_0} \PSi{q, e_i}{L^2(\pi)}^2$ converges to $0$ uniformly as $M_0 \to \infty$ and hence $g_{M_0}(x)$ converge to $0$ uniformly. Therefore, we are allowed to interchange the Hessian and the infinite sum~\citep{rudin1976principles} in \eqref{appeq:chi2_drmmd_diff} to achieve,
\begin{align}
 &\quad 2 \left| \int \nabla \phi(x)^\top \bH \left( \sum_{i \geq 1}\frac{\lambda}{\varrho_i + \lambda} \PSi{\frac{d \mu}{d \pi}-1, e_i}{L^2(\pi)} e_i(x) \right) \nabla \phi(x) \mu(x) dx \right| \nonumber \\
 &= 2 \left| \int \nabla \phi(x)^\top  \left( \sum_{i \geq 1}\frac{\lambda}{\varrho_i + \lambda} \varrho_i^{1/2} \PSi{q, e_i}{L^2(\pi)} \bH e_i(x) \right) \nabla \phi(x) \mu(x) dx \right| \nonumber \\
&\leq 2 \left( \sum_{i \geq 1} \left( \frac{\lambda }{\varrho_i + \lambda}\right)^2 \PSi{q, e_i}{L^2(\pi)}^2  \right)^{1/2} \left( \sum_{i \geq 1} \varrho_i \left( \int \nabla \phi(x)^\top  \bH e_i(x) \nabla \phi(x) \mu(x) dx \right)^2 \right)^{1/2} \nonumber \\
&\leq 2 \left( \sum_{i \geq 1} \left( \frac{\lambda }{\varrho_i + \lambda}\right)^2 \PSi{q, e_i}{L^2(\pi)}^2  \right)^{1/2} \left( \sum_{i \geq 1} \varrho_i \left\| \bH e_i(x) \right\|_{op}^2 \left( \int \nabla \phi(x)^\top \nabla \phi(x) \mu(x) dx \right)^2 \right)^{1/2} \nonumber \\
&\leq 2 \left( \sum_{i \geq 1} \left( \frac{\lambda }{\varrho_i + \lambda}\right)^2 \PSi{q, e_i}{L^2(\pi)}^2  \right)^{1/2} \left( \sum_{i \geq 1} \varrho_i \left\| \bH e_i(x) \right\|_{op}^2 \right)^{1/2} \|\nabla \phi \|_{L^2(\mu)}^2  \nonumber \\
&\leq 2 \left( \sum_{i \geq 1} \left( \frac{\lambda }{\varrho_i + \lambda}\right)^2 \PSi{q, e_i}{L^2(\pi)}^2  \right)^{1/2} \left( \sum_{i \geq 1} \varrho_i \left\| \bH e_i(x) \right\|_{F}^2 \right)^{1/2} \|\nabla \phi \|_{L^2(\mu)}^2  \nonumber \\
&\leq 2 \left( \sum_{i \geq 1} \left( \frac{\lambda }{\varrho_i + \lambda}\right)^2 \PSi{q, e_i}{L^2(\pi)}^2  \right)^{1/2} \sqrt{K_{2d}} \|\nabla \phi \|_{L^2(\mu)}^2 . \label{appeq:diff_tecond_term}
\end{align}
The first inequality uses Cauchy Schwartz, the second last inequality uses that matrix operator norm is smaller than matrix Frobenius norm, and the last inequality uses \eqref{appeq:sum_kernel_hessian}.
Combining together \eqref{appeq:diff_first_term} and \eqref{appeq:diff_tecond_term}, we reach
\begin{align*}
    &\quad \left| \frac{1}{1 + \lambda} \frac{d^2}{d t^2}\Big|_{t=0}  \dmmd(\rho_t||\pi) - \frac{d^2}{d t^2} \Big|_{t=0} \chi^2(\rho_t ||\pi) \right| \\
    &\leq 2 \sum_{i \geq 1} \frac{\lambda}{\varrho_i + \lambda} \left( \int \nabla \phi(x)^\top \nabla e_i(x) \mu(x)  dx \right)^2 \\
    &\qquad\qquad+ 2 \left( \sum_{i \geq 1} \left( \frac{\lambda }{\varrho_i + \lambda}\right)^2 \PSi{q, e_i}{L^2(\pi)}^2  \right)^{1/2} \sqrt{K_{2d}} \|\nabla \phi \|_{L^2(\mu)}^2  \\
    & =: \bar{R}(\lambda, \mu, \nabla \phi) .
\end{align*}
Therefore, 
\begin{align*}
    \frac{d^2}{d t^2}\Big|_{t=0}  \dmmd(\rho_t ||\pi) & \geq (1+\lambda) \frac{d^2}{d t^2}\Big|_{t=0}  \chi^2(\rho_t || \pi) - (1+\lambda) \bar{R}(\lambda, \mu, \nabla \phi) \\
    & \geq (1+\lambda) \int \frac{d \mu}{ d \pi} (x) \nabla \phi(x)^{\top} \mathbf{H} V(x) \nabla \phi(x) d \mu(x) - (1+\lambda) \bar{R}(\lambda, \mu, \nabla \phi) \\
    &\geq (1+\lambda) \alpha \int \frac{d \mu}{d \pi}(x)\|\nabla \phi(x)\|^2 d \mu(x) - (1+\lambda) \bar{R} (\lambda, \mu, \nabla \phi) ,
\end{align*}
where the second inequality is using \eqref{appeq:hess_chi2} and the last inequality is using $\bH V \succeq \alpha \Id$. 
So \eqref{eq:hessian_lower_bound} is proved.

\vspace{1em}
\noindent
Define $R(\lambda, \mu, \nabla \phi) := (1+\lambda) \bar{R} (\lambda, \mu, \nabla \phi)$.
The final thing to check is $\lim_{\lambda \to 0} R(\lambda, \mu, \nabla \phi) = 0$, which is equivalent to check that $\lim_{\lambda \to 0} \bar{R}(\lambda, \mu, \nabla \phi) = 0$.
Since we know from \eqref{appeq:hess_chi2_bounded} and \eqref{appeq:chi2_hessian_first_term} that 
\begin{align*}
    \sum_{i \geq 1} \left( \int \nabla \phi(x)^\top \nabla e_i(x) \mu(x)  dx \right)^2  < \frac{d^2}{dt^2}\Big|_{t=0} \chi^2(\rho_t \| \pi) < \infty,
\end{align*}
using the dominated convergence theorem~\citep{rudin1976principles}, we are allowed to interchange infinite sum and taking limits,
\begin{align*}
    \lim_{\lambda \to 0} \sum_{i \geq 1} \frac{\lambda}{\varrho_i + \lambda} \left( \int \nabla \phi(x)^\top \nabla e_i(x) \mu(x)  dx \right)^2 &= \sum_{i \geq 1} \lim_{\lambda \to 0} \frac{\lambda}{\varrho_i + \lambda} \left( \int \nabla \phi(x)^\top \nabla e_i(x) \mu(x)  dx \right)^2 \\
    &= 0.
\end{align*}
Similarly, because $\sum_{i \geq 1} \left( \frac{\lambda }{\varrho_i + \lambda}\right)^2 \PSi{q, e_i}{L^2(\pi)}^2 < \left\| q \right\|_{L^2(\pi)}^2 < \infty$, using dominated convergence theorem~\citep{rudin1976principles} again, we have,
\begin{align*}
    \lim_{\lambda \to 0}  \sum_{i \geq 1} \left( \frac{\lambda }{\varrho_i + \lambda}\right)^2 \PSi{q, e_i}{L^2(\pi)}^2  = \sum_{i \geq 1} \lim_{\lambda \to 0} \left( \frac{\lambda }{\varrho_i + \lambda}\right)^2 \PSi{q, e_i}{L^2(\pi)}^2 = 0 .
\end{align*}
Therefore, we have that
\begin{align*}
    \lim_{\lambda \to 0} R(\lambda, \mu, \nabla \phi) = \lim_{\lambda \to 0} \bar{R}(\lambda, \mu, \nabla \phi) = 0 .
\end{align*}
And the proof of the proposition is finished.

\subsection{Proof of \Cref{thm:continuous_time_convergence}}\label{appsec:proof_continuous_time_convergence}
Considering the time derivative of $\kl(\mu_t \| \pi)$, we have
\begin{small}
\begin{align}
&\quad \frac{d}{dt} \kl(\mu_t \| \pi) \nonumber \\
&= -(1+\lambda) \int \nabla h_{\mu_t, \pi}^\ast(x)^\top \nabla \log \frac{d \mu_t}{ d \pi} (x) \mu_t(x) dx \nonumber \\
&= -(1+\lambda) \int \nabla h_{\mu_t, \pi}^\ast(x)^\top \nabla \frac{d \mu_t}{ d \pi} (x) \pi(x) dx \nonumber \\
&= -(1+\lambda) \int \left( \nabla h_{\mu_t, \pi}^\ast(x) - 2 \nabla \frac{d \mu_t}{ d \pi} (x) \right)^\top \nabla \frac{d \mu_t}{ d \pi} (x) \pi(x) dx - 2(1+\lambda) \int \pi(x) \left\| \nabla \frac{d \mu_t}{ d \pi} (x) \right\|^2 dx \nonumber \\
&= -(1+\lambda) \int \left( \nabla h_{\mu_t, \pi}^\ast(x) - 2 \nabla \left( \frac{d \mu_t}{ d \pi} (x) -1 \right) 
 \right)^\top \nabla \frac{d \mu_t}{ d \pi} (x) \pi(x) dx \nonumber \\
&- 2(1+\lambda) \int \pi(x) \left\| \nabla \frac{d \mu_t}{ d \pi} (x) \right\|^2 dx . \label{eq:kl_d_dt}
\end{align}
\end{small}
\textit{Case one:} $\int \pi(x) \left\| \nabla \frac{d \mu_t}{ d \pi} (x) \right\|^2 dx < \infty$. 
We use integration by parts for the first term in \eqref{eq:kl_d_dt} and we can safely ignore the boundary term due to condition 5 that for $i = 1, \ldots, d$, $\lim\limits_{x\to \infty} \left( h_{\mu_t, \pi}^\ast(x) - 2 \frac{d \mu_t}{ d \pi} (x) \right) \left( \partial_i \frac{d \mu_t}{ d \pi} (x) \right) \pi(x) \to 0$. So, we obtain
\begin{align}
&\quad \frac{d}{dt} \kl(\mu_t \| \pi) = (1+\lambda) \int \left( h_{\mu_t, \pi}^\ast(x) - 2 \left( \frac{d \mu_t}{ d \pi} (x) -1 \right)  \right) \nabla \cdot \left( \pi(x) \nabla \frac{d \mu_t}{ d \pi} (x) \right) dx \nonumber \\
&\qquad\qquad- 2(1+\lambda) \int \pi(x) \left\| \nabla \frac{d \mu_t}{ d \pi} (x) \right\|^2 dx \nonumber \\
&= (1+\lambda) \int \left( h_{\mu_t, \pi}^\ast(x) - 2 \left( \frac{d \mu_t}{ d \pi} (x) -1 \right)  \right) \frac{ \nabla \cdot \left( \pi(x) \nabla \frac{d \mu_t}{ d \pi} (x) \right) }{ \pi(x) } \pi(x) dx \nonumber \\
&\qquad\qquad- 2(1+\lambda) \int \pi(x) \left\| \nabla \frac{d \mu_t}{ d \pi} (x) \right\|^2 dx \nonumber \\
&\leq (1+\lambda) \left\|  h_{\mu_t, \pi}^\ast - 2 \left( \frac{d \mu_t}{d \pi} -1 \right) \right\|_{L^2(\pi)} \left\| \frac{ \nabla \cdot \left( \pi \nabla \frac{d \mu_t}{ d \pi} \right) }{ \pi } \right\|_{L^2(\pi)} - \frac{ 2(1+\lambda)}{ C_P} \kl(\mu_t \| \pi) \label{appeq:continous_term},
\end{align}
where the first part of the last inequality holds by using Cauchy Schwartz, and the second part holds by the fact that $\kl(\mu_t \| \pi) \leq \chi^2(\mu_t \| \pi)$~\citep{van2014renyi} and by applying the Poincar\'{e} inequality with $f = \frac{d \mu_t}{d \pi} - 1$ (notice that $ \| \nabla f \|_{L^2(\pi)} < \infty$ from \textit{Case one} and $\|f\|_{L^2(\pi)} < \infty$ from condition 3),
\begin{align*}
    \int \pi(x) \left\| \nabla \frac{d \mu_t}{ d \pi} (x) \right\|^2 dx \geq \frac{1}{C_P} \chi^2(\mu_t \| \pi) \geq \frac{1}{C_P} \kl(\mu_t \| \pi).
\end{align*}
Since $\frac{ d\mu_t}{d \pi} - 1 \in \operatorname{Ran} \left( \mathcal{T}_\pi^r \right)$ with $r > 0$, using \Cref{lem:h_difference} we have
\begin{align}\label{appeq:continous_term_1}
    \left\| h_{\mu_t, \pi}^\ast  - 2 \left( \frac{d \mu_t }{d \pi} -1 \right) \right\|_{L^2(\pi)} \leq 2 \lambda^{r} \left\|q_t \right\|_{L^2(\pi)}. 
\end{align}
Then, notice that
\begin{align}
&\left\| \frac{ \nabla \cdot \left( \pi \nabla \frac{d \mu_t}{ d \pi}  \right) }{ \pi } \right\|_{L^2(\pi)}^2 = \int \frac{ \left[ \nabla \cdot\left( \pi(x) \nabla \frac{d \mu_t}{ d \pi} (x)\right) \right]^2 }{\pi(x)^2} d \pi(x) \nonumber \\
& = \int \frac{ \left(  \nabla \pi(x)^\top \nabla \frac{d \mu_t}{ d \pi} (x) + \pi(x) \nabla \cdot \nabla \frac{d \mu_t}{ d \pi} (x) \right)^2 }{\pi(x)^2} d \pi(x) \nonumber \\
&= \int  \left( \nabla \log \pi(x)^\top \nabla \frac{d \mu_t}{ d \pi} (x) + \Delta \frac{d \mu_t}{ d \pi} (x) \right)^2 d \pi(x) \nonumber \\
&\leq \int 2 \left( \nabla \log \pi(x)^\top \nabla \frac{d \mu_t}{ d \pi} (x) \right)^2 + 2 \left( \Delta \frac{d \mu_t}{ d \pi} (x) \right)^2 d \pi(x) \nonumber \\
&= 2 \left\| \nabla \left(\log \pi\right)^\top \nabla \left( \frac{d \mu_t}{d\pi} \right) \right\|_{L^2(\pi)}^2 + 2 \left\| \Delta \left( \frac{d \mu_t}{d\pi} \right) \right\|_{L^2(\pi)}^2 
\leq 2 \mathcal{J}_t^2 + 2 \mathcal{I}_t^2 . \label{appeq:continous_term_2}
\end{align}
Therefore, plugging \eqref{appeq:continous_term_1} and \eqref{appeq:continous_term_2} back to \eqref{appeq:continous_term}, we have
\begin{align}\label{eq:kl_d_dt_final}
    \frac{d}{dt} \kl(\mu_t \| \pi) \leq 4 (1+\lambda)\lambda^{r} \left\| q_t \right\|_{L^2(\pi)} \left( \mathcal{J}_t +\mathcal{I}_t \right) - \frac{ 2 (1+\lambda)}{ C_P} \kl(\mu_t \| \pi).
\end{align}
\textit{Case two:} $\int \pi(x) \left\| \nabla \frac{d \mu_t}{ d \pi} (x) \right\|^2 dx = \infty$. 
The first term of \eqref{eq:kl_d_dt} remains the same, and the second term of \eqref{eq:kl_d_dt} now equals infinity which is larger than the finite $\frac{ 2 (1+\lambda)}{ C_P} \kl(\mu_t \| \pi)$, so we also obtain \eqref{eq:kl_d_dt_final} as in \textit{Case one}.

Therefore, both \textit{Case one} and \textit{Case two} result in \eqref{eq:kl_d_dt_final}. Using the Gronwall lemma, we obtain that for any $T > 0$,
\begin{align*}
\begin{aligned}
     \kl(\mu_T \| \pi) & \leq \exp\left( -\frac{ 2 (1+\lambda)}{ C_P} T\right) \kl(\mu_0 \| \pi) \nonumber \\
    &+ 4 (1+\lambda) \lambda^r \int_0^T \exp\left( -\frac{ 2(1+\lambda)}{C_P} (T-t) \right)\left\| q_t \right\|_{L^2(\pi)} (\mathcal{J}_t + \mathcal{I}_t )  dt,
    \end{aligned}
\end{align*}
which concludes the proof of \Cref{thm:continuous_time_convergence}.

\subsubsection{Derivation of \eqref{eq:fourth_condition} under stronger range assumption $r = \frac{1}{2}$.}
Notice that for any $x \in \R^d$, since $k$ is differentiable
\begin{align*}
    \nabla \left( \frac{d \mu_t}{d \pi} -1 \right)(x) = \left\langle \nabla k(x, \cdot), \frac{d \mu_t}{d \pi} -1 \right\rangle_{\calH^d} \leq \sqrt{K_{1d}} \left\| \frac{d \mu_t}{d \pi} -1 \right\|_\calH .
\end{align*}
And since $\frac{d \mu_t}{d \pi} -1 \in \calH$, there exists $q_t \in L^2(\pi)$ such that $\left\langle\frac{d \mu_t}{d \pi}-1, e_i\right\rangle_{L^2(\pi)}=\varrho_i^{1/2} \left\langle q_t, e_i\right\rangle_{L^2(\pi)} $ for all $i$, so
\begin{align*}
    \left\| \frac{d \mu_t}{d \pi} -1 \right\|_\calH^2 = \sum_{i \geq 1} \frac{1}{\varrho_i} \left\langle \frac{d \mu_t}{d \pi} -1, e_i \right\rangle_{L^2(\pi)}^2 = \sum_{i \geq 1} \left\langle q_t, e_i \right\rangle_{L^2(\pi)}^2 = \| q_t \|_{L^2(\pi)}^2 .
\end{align*}
Combining the above two equations, we have
\begin{align}
    \left\| \nabla \left( \log \pi \right)^\top \nabla \left( \frac{d \mu_t}{d \pi} -1 \right) \right\|_{L^2(\pi)} \leq \sqrt{K_{1d}} \| q_t \|_{L^2(\pi)} \left\| \nabla  \log \pi \right\|_{L^2(\pi)} .\nonumber
\end{align}
Also, for the other one, notice that
\begin{align*}
    \Delta \left(\frac{d \mu_t}{d \pi} \right) (x) &\leq \left\| \bH \left(\frac{d \mu_t}{d \pi} -1 \right) (x) \right\|_F = \left\| \left\langle \bH k(x, \cdot ) , \frac{d \mu_t}{d \pi} -1 \right\rangle \right\|_F \\
    &\leq \left\| \bH k(x, \cdot ) 
    \right\|_{\calH^{d \times d}} \left\| \frac{d \mu_t}{d \pi} -1 \right\|_\calH \leq \sqrt{K_{2d} } \| q_t \|_{L^2(\pi)} .
\end{align*}
Therefore, 
\begin{align}
    \left\| \Delta \left(\frac{d \mu_t}{d \pi} \right)  \right\|_{L^2(\pi)} \leq \sqrt{K_{2d} } \| q_t \|_{L^2(\pi)} .\nonumber
\end{align}

\subsection{Proof of \Cref{prop:descent_lemma_KL}}\label{appsec:proof_kl_descent_lemma}
We know that $\mu_{n+1} = (\Id-\gamma (1+\lambda) \nabla h)_{\#} \mu_n$ where we drop the subscripts of the witness function $h_{\mu_n ,\pi}$ when it causes no ambiguity.
Denote $\rho_t = (\Id- t (1+\lambda) \nabla h)_{\#} \mu_n$, so $\rho_0 = \mu_n$, and $\rho_{\gamma} = \mu_{n+1}$. 
Consider the difference of KL divergence between the two iterates $\mu_{n+1}$ and $\mu_n$ along the time-discretized $\dmmd$ flow:
\begin{align}
    \kl(\mu_{n+1} \| \pi) - \kl(\mu_{n} \| \pi) &= \kl(\rho_{\gamma} \| \pi) - \kl(\rho_0 \| \pi)  \nonumber
    \\ 
    &= \frac{d }{d t}\Big|_{t=0} \kl(\rho_t \| \pi) \gamma + \int_0^{\gamma} (\gamma - t)\frac{d^2 }{d t^2} \kl(\rho_t \| \pi) dt . \label{appeq:discrete_all}
\end{align}
For the first term of \eqref{appeq:discrete_all},
\begin{small}
\begin{align}\label{appeq:first_term}
    &\quad \frac{d }{d t}\Big|_{t=0} \kl(\rho_t \| \pi) \nonumber \\
    &= -(1 + \lambda)
    \E_{\mu_n} \left[ \nabla \log \frac{d \mu_n}{ d \pi} ^\top \nabla h \right] = - (1 + \lambda) \E_{\pi} \left[ \nabla \frac{d \mu_n}{ d \pi} ^\top \nabla h \right] \nonumber \\
    &= -2 (1 + \lambda) \left\| \nabla \frac{d \mu_n}{ d \pi}  \right\|_{L^2(\pi)}^2 + (1 + \lambda) \E_{\pi} \left[ \nabla \frac{d \mu_n}{ d \pi} ^\top \left( 2 \nabla \frac{d \mu_n}{ d \pi}  - \nabla h \right) \right] \nonumber \\
    &= -2 (1 + \lambda) \left\| \nabla \frac{d \mu_n}{ d \pi}  \right\|_{L^2(\pi)}^2 + (1 + \lambda) \int \left( \nabla \frac{d \mu_n}{ d \pi} (x) \right)^\top \left( 2 \nabla \left( \frac{d \mu_n}{ d \pi} (x) -1 \right) - \nabla h(x) \right) \pi(x) dx \nonumber \\
    &= -2 (1 + \lambda) \left\| \nabla \frac{d \mu_n}{ d \pi}  \right\|_{L^2(\pi)}^2 - (1 + \lambda) \int \left( 2 \left( \frac{d \mu_n}{ d \pi} (x) - 1 \right) - h(x) \right) \nabla \cdot \left( \pi(x) \nabla \frac{d \mu_n}{ d \pi} (x) \right) dx \nonumber \\
    &\leq - (1 + \lambda) \frac{2}{C_P} \chi^2(\mu_n \| \pi) + (1 + \lambda) \left\|  h - 2 \left( \frac{d \mu_n}{d \pi} -1 \right) \right\|_{L^2(\pi)} \left\| \frac{ \nabla \cdot \left( \pi \nabla \frac{d \mu_n}{ d \pi}  \right) }{ \pi } \right\|_{L^2(\pi)},
\end{align}
\end{small}
where the fourth equality uses an integration by parts, and the last inequality uses Poincar\'{e} inequality for the first term under similar arguments in \Cref{appsec:proof_continuous_time_convergence} and uses Cauchy-Schwarz for the second term. 
Using \Cref{lem:h_difference} and the derivations in \eqref{appeq:continous_term_2}, \eqref{appeq:first_term} can be further upper bounded by
\begin{align}\label{appeq:first_term_cont}
    \frac{d }{d t}\Big|_{t=0} \kl(\rho_t \| \pi) \leq -(1 + \lambda) \frac{2}{C_P} \chi^2(\mu_n \| \pi)  + 2(1 + \lambda) \lambda^{r} Q \left( \mathcal{J} +\mathcal{I} \right). 
\end{align}
Then, for the second term of \eqref{appeq:discrete_all}, we know from Example 15.9 of \cite{villani2009optimal} (taking $m=1$) that,
\begin{align*}
\frac{d^2 }{d t^2} \kl(\rho_t \| \pi) &= (1 + \lambda)^2 \int \nabla h( x )^\top \bH V( \varphi_t(x) ) \nabla h(x) d \mu_n(x) \\
&+ (1 + \lambda)^2 \int \left\| \bH h(x) \left(\Id - t (1 + \lambda) \bH h(x) \right)^{-1} \right\|_F^2 d\mu_n(x).
\end{align*}
Because $2(1+\lambda) \gamma \sqrt{\chi^2\left(\mu_0 \| \pi\right) \frac{K_{2 d}}{\lambda}} \leq \frac{\zeta-1}{\zeta}$ for $1 < \zeta < 2$, applying \Cref{lem:hessian_terms} we have,
\begin{align}\label{appeq:second_term}
    \frac{d^2 }{d t^2} \kl(\rho_t \| \pi) \leq 4 (1 + \lambda)^2 \beta \chi^2(\mu_n \| \pi) \frac{K_{1d}}{\lambda} + 4 (1 + \lambda)^2 \zeta^2 \chi^2(\mu_n \| \pi) \frac{K_{2d}}{\lambda}.
\end{align}
Combining the above two inequalities \eqref{appeq:first_term_cont} and \eqref{appeq:second_term} and plugging them back into \eqref{appeq:discrete_all}, we obtain
\begin{align*}
\kl(\mu_{n+1} \| \pi) - \kl(\mu_{n} \| \pi) &\leq -\frac{2}{C_P} (1 + \lambda ) \chi^2(\mu_n \| \pi) \gamma + 2 (1 + \lambda )\gamma \lambda^{r} Q \left( \mathcal{J} +\mathcal{I} \right) \\
&+ 2(1 + \lambda )^2 \gamma^2 (\beta + \zeta^2) \chi^2(\mu_n \| \pi) \frac{K_{1d}+K_{2d} }{\lambda} \\
& \leq -\frac{2}{C_P} \chi^2(\mu_n \| \pi) \gamma \\
&+ 4 \gamma \lambda^{r} Q \left( \mathcal{J} +\mathcal{I} \right) + 8 \gamma^2 (\beta + \zeta^2) \chi^2(\mu_n \| \pi) \frac{K_{1d}+K_{2d} }{\lambda},
\end{align*}
where the last inequality holds by using $0 < \lambda \leq 1$, and the result follows.

\subsection{Proof of \Cref{thm:discrete_time_KL}}\label{appsec:proof_discrete_time_convergence}
In order to use \Cref{prop:descent_lemma_KL} in the proof, first we are going to show that \Cref{prop:descent_lemma_KL} holds under the conditions of \Cref{thm:discrete_time_KL}.
Notice that conditions 1-4 of \Cref{prop:descent_lemma_KL} are precisely the conditions 1-4 of \Cref{thm:discrete_time_KL}. To use \Cref{prop:descent_lemma_KL} in the proof of \Cref{thm:discrete_time_KL}, the only thing left is to check that the condition of step size $\gamma$ in \eqref{appeq:gamma_condition_1} is satisfied. 

In \Cref{thm:discrete_time_KL}, $\lambda_n$ is selected to be $\left( 2 \gamma \frac{\chi^2(\mu_n \| \pi) ( \beta + \zeta^2 ) (K_{1d} + K_{2d})}{ Q( \mathcal{J} +\mathcal{I}) } \right)^{\frac{1}{r+1}} \wedge 1$. 
If $\lambda_n$ is taken to be the former, then
\begin{align}
    &\quad 2 \gamma (1+\lambda_n) \sqrt{ \chi^2(\mu_n \| \pi) \frac{K_{2d}}{\lambda_n} } \leq 4 \gamma \sqrt{ \chi^2(\mu_n \| \pi) \frac{K_{2d}}{\lambda_n} } \nonumber \\
    &= 2 (2 \gamma)^{\frac{2r + 1}{2 r + 2 }} \chi^2\left(\mu_n \| \pi\right)^{\frac{r }{2 r +  2 } } \left( Q (\mathcal{J} + \mathcal{I} ) \right)^{\frac{1}{2r + 2 }} \left(  \frac{ 1 }{K_{1d} + K_{2d}}  \frac{ 1 }{\beta + \zeta^2 } \right)^{\frac{1}{2r + 2}} K_{2d}^{\frac{1}{2}} \nonumber  \\
    &\stackrel{(*)}{\leq} \left(8 \gamma\right)^{\frac{2r + 1}{2 r + 2 }} Q^{\frac{2r + 1}{2r + 2}} (\mathcal{J} + \mathcal{I} )^{\frac{1}{2r + 2 }} \left(\frac{ 1 }{\beta + \zeta^2 } \right)^{\frac{1}{2r + 2}} K_{2d}^{\frac{r}{2r + 2 }}  \nonumber  \\
    &\leq \frac{\zeta - 1 }{\zeta} \left(\frac{ 1 }{\beta + \zeta^2 } \right)^{\frac{r}{2r + 2}} \leq \frac{\zeta - 1 }{\zeta}  . \label{appeq:dummy_2}
\end{align}
$(*)$ holds because $\frac{K_{2d}}{K_{1d} + K_{2d}} \leq 1$ and
\begin{align}\label{appeq:KQ}
    \chi^2\left(\mu_n \| \pi\right) = \left\| \frac{d \mu_n}{d \pi} - 1\right\|_{L^2(\pi)}^2 = \left\| \calT_\pi^{r} q_n \right\|_{L^2(\pi)}^2 \leq K^{2r} Q^2 \leq Q^2.
\end{align}
The second last inequality of \eqref{appeq:dummy_2} holds due to the constraint on $\gamma$ in \eqref{eq:gamma_condition}, and the last inequality of \eqref{appeq:dummy_2} holds because $\beta + \zeta^2 \geq 1$.

\vspace{1em}
\noindent
On the other hand, if $\lambda_n$ is chosen to be $1$, similarly based on the constraint on $\gamma$ in \eqref{eq:gamma_condition}, we have
\begin{align*}
    2 \gamma (1+\lambda_n) \chi^2(\mu_n \| \pi) \frac{K_{2d}}{\lambda_n} &= 4 \gamma \chi^2(\mu_n \| \pi) K_{2d} \leq \frac{\zeta-1}{\zeta}.
\end{align*}
Therefore, all the conditions of \Cref{prop:descent_lemma_KL} have been verified. 
So, if we select $\lambda_n = \left( 2 \gamma \frac{\chi^2(\mu_n \| \pi) ( \beta + \zeta^2 ) (K_{1d} + K_{2d})}{ Q( \mathcal{J} +\mathcal{I}) } \right)^{\frac{1}{r+1}} \wedge 1$, then 
\begin{align}\label{appeq:KL_descent_1}
\kl(\mu_{n+1} \| \pi) - \kl(\mu_{n} \| \pi) &\leq - \gamma \frac{2}{C_P} \chi^2(\mu_n \| \pi) \nonumber \\
&\qquad+ \underbrace{ 4 \gamma \lambda_n^{r} Q \left( \mathcal{J} +\mathcal{I} \right) }_{(\Delta_1)} + \underbrace{ 8 \gamma^2 (\beta + \zeta^2) \chi^2(\mu_n \| \pi) \frac{K_{1d}+K_{2d} }{\lambda_n} }_{(\Delta_2)}.
\end{align}
By observing \eqref{appeq:KL_descent_1}, the first term on the right-hand side $-\frac{2}{C_P} \chi^2(\mu_n \| \pi) \gamma $ is strictly negative and is decreasing KL divergence at each iteration $n$ of the $\dmmd$ gradient descent. 
In contrast, the second term $(\Delta_1) := 4 \gamma \lambda_n^{r} Q( \mathcal{J} +\mathcal{I})$ and the third term $(\Delta_2) := 8 \gamma^2 (\beta + \zeta^2) \chi^2(\mu_n \| \pi) \frac{K_{1d}+K_{2d} }{\lambda_n}$ are positive and prevent the KL divergence from decreasing. 
Denote $G(\lambda_n) = (\Delta_1) + (\Delta_2)$ and the optimal $\lambda_n$ is achieved by taking $\frac{d}{ d \lambda_n} G(\lambda_n) = 0$, which leads to $\lambda_n = \left( 2 \gamma \frac{\chi^2(\mu_n \| \pi) ( \beta + \zeta^2 ) (K_{1d} + K_{2d})}{ Q( \mathcal{J} +\mathcal{I}) } \right)^{\frac{1}{r+1}}$.
Plugging the value of $\lambda_n$ back to \eqref{appeq:KL_descent_1} to obtain,
\begin{align}\label{eq:kl_descent_2}
&\quad \kl(\mu_{n+1} \| \pi) - \kl(\mu_{n} \| \pi) \nonumber \\
&\le -\frac{2}{C_P} \chi^2(\mu_n \| \pi) \gamma + 4 \gamma \left( 2 \gamma\right)^{\frac{r}{r + 1} } \left( \chi^2(\mu_n \| \pi) ( \beta + \zeta^2 ) (K_{1d} + K_{2d}) \right)^{\frac{ r }{ r + 1 }}  \left( Q( \mathcal{J} +\mathcal{I}) \right)^{\frac{ 1 }{ r + 1 }} \nonumber \\
&\leq -\frac{2}{C_P} \chi^2(\mu_n \| \pi) \gamma + 8 \gamma^{1 + \frac{ r }{ r + 1 }} \left( \chi^2(\mu_n \| \pi) ( \beta + \zeta^2 ) (K_{1d} + K_{2d}) \right)^{\frac{ r }{ r + 1 }}  \left( Q( \mathcal{J} +\mathcal{I})\right)^{\frac{ 1 }{ r + 1 }} \nonumber \\
&\leq -\frac{2}{C_P} \chi^2(\mu_n \| \pi) \gamma + 8 \gamma^{1 + \frac{ r }{ r + 1 } } Q^{\frac{2r+1}{r+1}} \Big( (K_{1d} + K_{2d})(\beta + \zeta^2) \Big)^{\frac{ r }{ r + 1 }} ( \mathcal{J} + \mathcal{I})^{\frac{ 1 }{ r + 1 }} ,
\end{align}
where the last inequality holds because of \eqref{appeq:KQ}.
Since $ \chi^2(\mu_n \| \pi) \geq \kl(\mu_n \| \pi)$~\citep[Equation (7)]{van2014renyi}, we have
\begin{align*}
    \kl(\mu_{n+1} \| \pi)& \leq \left(1 - \gamma \frac{2}{C_P} \right) \kl(\mu_{n} \| \pi) \nonumber\\
    &\qquad\qquad+ 8 \gamma^{1 + \frac{ r }{ r + 1 } } Q^{\frac{2r+1}{r+1}} \Big( (K_{1d} + K_{2d})(\beta + \zeta^2) \Big)^{\frac{ r }{ r + 1 }} ( \mathcal{J} +\mathcal{I})^{\frac{ 1 }{ r + 1 }}  .
\end{align*}
After iterating, we obtain
\begin{align*}
    &\quad \kl(\mu_{n_{max} } \| \pi) \\
    &\leq \left(1 - \gamma 
    \frac{2}{C_P} \right)^{n_{max}} \kl(\mu_0 \| \pi) + 4 \gamma^{ \frac{ r }{ r + 1 } } C_P Q^{\frac{2r+1}{r+1}} \Big( (K_{1d} + K_{2d})(\beta + \zeta^2) \Big)^{\frac{ r }{ r + 1 }} ( \mathcal{J} +\mathcal{I})^{\frac{ 1 }{ r + 1 }} \\
    &\leq \exp\left(- \frac{ 2 n_{max} \gamma}{C_P} \right) \kl(\mu_0 \| \pi) 
    + 4 \gamma^{ \frac{ r }{ r + 1 } } C_P Q^{\frac{2r+1}{r+1}} \Big( (K_{1d} + K_{2d})(\beta + \zeta^2) \Big)^{\frac{ r }{ r + 1 }} ( \mathcal{J} +\mathcal{I})^{\frac{ 1 }{ r + 1 }}
\end{align*}
and the result follows.

\subsection{Proof of \Cref{thm:discrete_time_KL_2}}\label{appsec:proof_discrete_time_convergence_2}
The proof of \Cref{thm:discrete_time_KL_2} is also based on \Cref{prop:descent_lemma_KL} proved in the last section. 
Recalling \eqref{eq:kl_descent_2} yet with adaptive step size $\gamma_n$, we have
\begin{align*}
    &\quad \kl(\mu_{n+1} \| \pi) - \kl(\mu_{n} \| \pi) \\
    &\leq -\frac{2}{C_P} \chi^2(\mu_n \| \pi) \gamma_n + 8 \gamma_n^{1 + \frac{ r }{ r + 1 }} \left( \chi^2(\mu_n \| \pi) ( \beta + \zeta^2 ) (K_{1d} + K_{2d}) \right)^{\frac{ r }{ r + 1 }}  \left( Q( \mathcal{J} +\mathcal{I})\right)^{\frac{ 1 }{ r + 1 }}.
\end{align*}
From \eqref{eq:gamma_condition_2}, we have 
\begin{align*}
    8 \gamma_n^{\frac{ r }{ r + 1 } } \Big( (K_{1d} + K_{2d})(\beta + \zeta^2) \Big)^{\frac{ r }{ r + 1 }} ( Q( \mathcal{J} +\mathcal{I}) )^{\frac{ 1 }{ r + 1 }} \leq \frac{1}{C_P} \chi^2(\mu_n \| \pi)^{\frac{1}{r + 1}},
\end{align*}
so that we have
\begin{align*}
    \kl(\mu_{n+1} \| \pi) - \kl(\mu_{n} \| \pi) \leq -\frac{1}{C_P} \chi^2(\mu_n \| \pi) \gamma_n \leq -\frac{1}{C_P} \kl(\mu_n \| \pi) \gamma_n.
\end{align*}
Hence
\begin{align}\label{eq:kl_monotone}
    \kl(\mu_{n+1} \| \pi) \leq \left( 1 - \frac{1}{C_P} \gamma_n \right) \kl(\mu_n \| \pi) .
\end{align}
After iterating $n$ from $1$ to $n_{\max}$, the theorem is proved.

\subsection{Proof of \Cref{thm:final_rate}}\label{sec:proof_final}
In order to analyze the error of space discretization, we introduce another particle descent scheme using the population witness function $h_{\mu_n,\pi}^\ast$ defined in \eqref{eq:time_discretized_flow} starting from the same initialization as that of \eqref{eq:particle_biased},
\vspace{-3pt}
\begin{align}\label{eq:particle_unbiased}
    \overline{y}_{n+1}^{(i)} = \overline{y}_n^{(i)}-\gamma(1+\lambda_n) \nabla h_{\mu_n,\pi}^\ast (\overline{y}_n^{(i)}), \quad \overline{y}_0^{(i)} = y_0^{(i)}.
\vspace{-3pt}
\end{align}
The corresponding empirical distribution of the particles at time step $n$ is defined as $\overline{\mu}_n = \frac{1}{N} \sum_{i=1}^N \overline{y}_n^{(i)}$. 
Note that \eqref{eq:particle_unbiased} is an unbiased sampled version (since it is composed of $N$ i.i.d. realizations) of \eqref{eq:time_discretized_flow}. 
The following proposition shows that $W_2(\overline{\mu}_n, \hat{\mu}_n) \to 0$ as $N,M \to \infty$, i.e., with a sufficient number of samples from $\mu_0$ and $\pi$,
\eqref{eq:particle_biased} can approximate \eqref{eq:particle_unbiased} with
arbitrary precision. 
The proof of \Cref{prop:space_discretized} is provided in \Cref{appsec:proof_space_discretized}.

\begin{prop}\label{prop:space_discretized}
Suppose $k$ satisfies \Cref{assumption:universal} and  \ref{assumption:bounded_kernel}. Given initial particles $\{y_0^{(i)}\}_{i=1}^N$ that are i.i.d sampled from $\mu_0$, a sequence $\left(\overline{\mu}_n\right)_{n\in \mathbb{N}}$  of empirical distributions arising from \eqref{eq:particle_unbiased}, and  a sequence $\left(\hat{\mu}_n\right)_{n\in \mathbb{N}}$  arising from \eqref{eq:particle_biased}, then for all $n \geq 1$, we have
\begin{align*}
    \E [W_2 \left(\hat{\mu}_{n}, \overline{\mu}_{n} \right)] \leq A \left(\frac{ K }{\sqrt{M} \tilde{\lambda} } + \frac{ 1}{\sqrt{M}} + \frac{1}{\sqrt{N}} \right) \left( \exp\left( \gamma n \frac{ 2 (1 + \tilde{\lambda} ) R }{\tilde{\lambda} } \right) - 1 \right),
\end{align*}
where $A = \frac{2 \sqrt{K K_{1 d}}}{ \sqrt{K K_{2 d}} + K_{1 d}}$ and $R = K_{1 d} + \sqrt{K K_{2 d}}$ are constants that only depend on the kernel, and $\tilde{\lambda} = \min\limits_{i = 1, \ldots, n } \lambda_i$ denotes the smallest regularization coefficient. 
\end{prop}

Now we are ready to prove \Cref{thm:final_rate}. By the triangle inequality, we have,
\begin{align*}
    \E \left[W_2 \left(\hat{\mu}_n, \pi \right)\right] \leq \E \left[W_2 \left(\hat{\mu}_n, \bar{\mu}_n \right)\right] + \E \left[W_2 \left(\bar{\mu}_n, \mu_n \right)\right] + W_2 \left(\mu_n, \pi \right).
\end{align*}
From \Cref{prop:space_discretized}, the first term is upper bounded by
\begin{align*}
    \E [W_2 \left(\hat{\mu}_{n}, \overline{\mu}_{n} \right)] &\leq A \left(\frac{ K }{\sqrt{M} \tilde{\lambda} } + \frac{ 1}{\sqrt{M}} + \frac{1}{\sqrt{N}} \right) \left( \exp\left( n \gamma \frac{ 2 (1 + \tilde{\lambda} ) R }{ \tilde{\lambda} } \right) - 1 \right) \\
    &\leq A \left(\frac{1}{\sqrt{M} \tilde{\lambda} } + \frac{1}{\sqrt{N}} \right) \left( \exp\left( \frac{ 4 n \gamma R }{ \tilde{\lambda} } \right) - 1 \right) ,
\end{align*}
where the second inequality uses $\tilde{\lambda} \leq 1$ and $K \leq 1$. Since $(\mu_n)_n$ has finite fourth moment, then by taking $p=2, q=4$ in \cite[Theorem 3.1]{lei2020convergence} and \cite[Theorem 1]{fournier2015rate}, the second term is upper bounded by,
\begin{align*}
    \E \left[W_2 \left( \bar{\mu}_n, \mu_n \right)\right] = \calO \left( N^{-\frac{1}{d \vee 4}} \right) .
\end{align*}
For the third term, since the Wasserstein-2 distance is upper bounded by the square root of the KL divergence, if the target $\pi$ that satisfies Talagrand-2 inequality with constant $C_T$~\citep[Definition 22.1]{villani2009optimal}, we have
\begin{align*}
    W_2 \left(\mu_n, \pi \right) \leq \sqrt{ 2 C_T} \sqrt{ \kl \left(\mu_n \| \pi \right)} \leq \sqrt{ 2 C_T} \exp \left(-\frac{ n \gamma}{C_P}\right) \sqrt{ \kl\left(\mu_0 \| \pi\right) } + \calO \left( \gamma^{\frac{r}{2r + 2}} \right) ,
\end{align*}
where the last inequality follows from \Cref{thm:discrete_time_KL}.
Combining the above three terms, we obtain
\begin{align}\label{eq:W_2_mu_n_pi}
    \E \left[W_2 \left(\hat{\mu}_n, \pi \right)\right] &\leq A \left( \frac{1}{\sqrt{M} \tilde{\lambda} } + \frac{1}{\sqrt{N}} \right) \exp\left( \frac{ 4 n \gamma R }{ \tilde{\lambda} } \right) + \calO\left( N^{-\frac{1}{d \vee 4} } \right) \nonumber \\
    &+ \sqrt{ 2 C_T} \exp \left(-\frac{n \gamma}{C_P}\right) \sqrt{ \kl\left(\mu_0 \| \pi\right) } + \calO \left( \gamma^{\frac{r}{2r + 2}} \right) ,
\end{align}
where $A = \frac{ 2 \sqrt{K K_{1 d}}}{ \sqrt{K K_{2 d}} + K_{1 d}}$. 
Recall from \Cref{thm:discrete_time_KL} that $\lambda_i = \left( \gamma \chi^2(\mu_i \| \pi) Z \right)^{\frac{1}{r+1}} \wedge 1$ for $i = 1, \ldots, n$ where $Z$ is the constant that depends on $\beta, \zeta, K_{1d}, K_{2d}, Q, \mathcal{J}, \mathcal{I}$.
From the condition on the number of samples $M$ and $N$ in \eqref{eq:M_N_condition}, we obtain that if $\tilde{\lambda} = \lambda_j = \left( \gamma \chi^2(\mu_j \| \pi) Z \right)^{\frac{1}{r+1}}$ for some $j \in \{1, \ldots, n \}$,
\begin{align*}
    \sqrt{M} &\gtrsim \left( \frac{1}{\gamma}\right) \left( \frac{1}{\min \limits_{i = 1, \ldots, n} \kl(\mu_i \| \pi) Z \wedge 1 } \right)^{\frac{1}{r+1}} \exp\left( \frac{ 4 n \gamma^{\frac{r}{ r + 1}} R }{ \left( \min \limits_{i = 1, \ldots, n} \kl(\mu_i \| \pi) Z \right)^{\frac{1}{r+1}} \wedge 1 } \right) \\
    &\geq \frac{1}{\gamma} \left( \frac{1}{ \min \limits_{i = 1, \ldots, n} \kl(\mu_i \| \pi) Z }  \right)^{ \frac{1}{r+1} } \exp\left( \frac{ 4 n \gamma^{\frac{r}{ r + 1}} R }{ \left( \min \limits_{i = 1, \ldots, n} \kl(\mu_i \| \pi) Z \right)^{\frac{1}{r+1}} } \right) \\
    &\geq \frac{1}{\gamma} \left( \frac{1}{\kl(\mu_j \| \pi) Z }  \right)^{ \frac{1}{r+1} } \exp\left( \frac{ 4 n \gamma^{\frac{r}{ r + 1}} R }{ \left( \kl(\mu_j \| \pi) Z \right)^{\frac{1}{r+1}} } \right)  \\ 
    &\geq \frac{1}{\gamma} \left( \frac{1}{\chi^2(\mu_j \| \pi) Z } \right)^{\frac{1}{r+1}} \exp\left( \frac{ 4 n \gamma^{\frac{r}{ r + 1}} R }{ \left( \chi^2(\mu_j \| \pi) Z \right)^{\frac{1}{r+1}} } \right) \\
    &= A \gamma^{-\frac{r}{r + 1}} \frac{1}{ \tilde{ \lambda} } \exp\left( \frac{ 4 n \gamma R }{\tilde{ \lambda} } \right)  
    \geq \gamma^{-\frac{r}{2r + 2}} \frac{1}{ \tilde{ \lambda} } \exp\left( \frac{ 4 n \gamma R }{\tilde{ \lambda} } \right) .
\end{align*}
On the other hand, if $\tilde{ \lambda} = 1$, since $\gamma \leq 1$,
\begin{align*}
    \sqrt{M} & \gtrsim \left( \frac{1}{\gamma}\right) \left( \frac{1}{\min \limits_{i = 1, \ldots, n} \kl(\mu_i \| \pi) Z \wedge 1 } \right)^{\frac{1}{r+1}} \exp\left( \frac{ 4 n \gamma^{\frac{r}{ r + 1}} R }{ \left( \min \limits_{i = 1, \ldots, n} \kl(\mu_i \| \pi) Z \right)^{\frac{1}{r+1}} \wedge 1 } \right) \\
    &\geq \frac{1}{\gamma} \exp\left( 4 n \gamma^{\frac{r}{ r + 1}} R  \right) 
    \geq \gamma^{-\frac{r}{2r + 2}}  \exp\left( 4 n \gamma R  \right)
    = \gamma^{-\frac{r}{2r + 2}} \frac{1}{ \tilde{ \lambda} } \exp\left( \frac{ 4 n \gamma R }{ \tilde{ \lambda} } \right).
\end{align*}
Similarly for $N$, we have
\begin{align*}
    \sqrt{N} &\gtrsim \gamma^{-\frac{r}{2r + 2}} \exp\left( \frac{ 4 n \gamma R }{ \tilde{ \lambda} } \right) \quad\text{and}\quad N^{-\frac{1}{d \vee 4} } \lesssim  \gamma^{\frac{r}{2r + 2}} .
\end{align*}
Plugging them back to \eqref{eq:W_2_mu_n_pi}, we obtain
\begin{align*}
    \E \left[W_2 \left(\hat{\mu}_n, \pi \right)\right] \leq \sqrt{ 2 C_T} \exp \left(-\frac{ n \gamma}{C_P}\right) \sqrt{ \kl\left(\mu_0 \| \pi\right) } + \calO \left( \gamma^{\frac{r}{2r + 2}} \right),
\end{align*}
which completes the proof.

\subsection{Proof of \Cref{prop:space_discretized}}\label{appsec:proof_space_discretized}
Since the proof below works for any regularization coefficient $\lambda$, we use a fixed $\lambda$ for the majority of the analysis and resort back to adaptive $\lambda_n$ at the end of the proof.
For empirical distributions $\hat{\mu}_n = \frac{1}{N} \sum_{i=1}^N y_{n}^{(i)}$ and $\bar{\mu}_n = \frac{1}{N} \sum_{i=1}^N \overline{y}_{n}^{(i)}$ defined in \eqref{eq:particle_biased} and \eqref{eq:particle_unbiased}, note that
\begin{align*}
    \begin{aligned}
        \E W_2^2(\hat{\mu}_n, \bar{\mu}_n) \leq \frac{1}{N} \sum_{i=1}^N \E\left[\left\|y_{n}^{(i)}-\overline{y}_{n}^{(i)}\right\|^2\right] := c_n^2 .
    \end{aligned}
\end{align*}
Consider
\begin{align*}
    c_{n+1} &= \sqrt{\frac{1}{N} \sum_{i=1}^N \E\left[\left\|y_{n+1}^{(i)}-\overline{y}_{n+1}^{(i)}\right\|^2\right]} \\
    &= \sqrt{\frac{1}{N} \sum_{i=1}^N \E \left[\left\|y_{n}^{(i)}-\overline{y}_{n}^{(i)} - \gamma (1 + \lambda) \left(\nabla h_{\hat{\mu}_n, \hat{\pi}}^\ast(y_n^{(i)}) -  \nabla h_{{\mu}_n, {\pi}}^\ast(\overline{y}_n^{(i)}) \right)\right\|^2\right]} \\
    &\leq \sqrt{\frac{1}{N} \sum_{i=1}^N \E \left\|y_{n}^{(i)}-\overline{y}_{n}^{(i)} \right\|^2} \\
    &\qquad\qquad+ \sqrt{ \frac{1}{N} \E \sum_{i=1}^N \left\| \gamma (1 + \lambda) \left(\nabla h_{\hat{\mu}_n, \hat{\pi}}^\ast(y_n^{(i)}) -  \nabla h_{{\mu}_n, {\pi}}^\ast(\overline{y}_n^{(i)}) \right)\right\|^2\ } \\
    &= c_n + \frac{\gamma (1 + \lambda)}{\sqrt{N}} \sqrt{\sum_{i=1}^N \E \left\| \nabla h_{\hat{\mu}_n, \hat{\pi}}^\ast(y_n^{(i)}) -  \nabla h_{{\mu}_n, {\pi}}^\ast(\overline{y}_n^{(i)}) \right\|^2},
\end{align*}
where we used Minkowski's inequality,
\begin{align*}
    \sqrt{\sum_{i=1}^N \| a_i + b_i \|^2} \leq \sqrt{\sum_{i=1}^N \| a_i\|^2} + \sqrt{\sum_{i=1}^N \| b_i \|^2}
\end{align*}
in the above inequalities.
Again, by Minkowski's inequality, we have
\begin{align*}
c_{n+1} \leq c_n+\gamma(1+\lambda)\Bigg(
\underbrace{\frac{1}{\sqrt{N}} \sqrt{\sum_{i=1}^N \E\left[\left\|\nabla h_{\hat{\mu}_n, \hat{\pi}}^\ast(y_n^{(i)})-\nabla h_{\hat{\mu}_n, \hat{\pi}}^\ast(\overline{y}_n^{(i)})\right\|^2\right]}}_{(i)} 
\\
+\underbrace{\frac{1}{\sqrt{N}} \sqrt{\sum_{i=1}^N \E\left[\left\|\nabla h_{\hat{\mu}_n, \hat{\pi}}^\ast(\overline{y}_n^{(i)})-\nabla h_{{\mu}_n, {\pi}}^\ast(\overline{y}_n^{(i)})\right\|^2\right]}}_{(i i)} \Bigg) .
\end{align*}

\subsubsection{Controlling $(i)$:}
\begin{align*}
    &\sum_{i=1}^N \E \left\|\nabla h_{\hat{\mu}_n, \hat{\pi}}^\ast(y_n^{(i)})-\nabla h_{\hat{\mu}_n, \hat{\pi}}^\ast(\overline{y}_n^{(i)})\right\|^2 
    = \sum_{i=1}^N \E \left[ \sum_{j=1}^d \PSi{\partial_j k(y_n^{(i)}, \cdot) - \partial_j k(\overline{y}_n^{(i)}, \cdot), h_{\hat{\mu}_n, \hat{\pi}}^\ast}{\calH}^2 \right] \\
    &\leq \sum_{i=1}^N \E \left[ \sum_{j=1}^d \left\| \partial_j k(y_n^{(i)}, \cdot) - \partial_j k(\overline{y}_n^{(i)}, \cdot) \right\|_\calH^2 \right] \left\| h_{\hat{\mu}_n, \hat{\pi}}^\ast \right\|_\calH^2 
    \leq \frac{4 K K_{2d}}{\lambda^2} \sum_{i=1}^N \E \left\|y_n^{(i)} - \overline{y}_n^{(i)} \right\|^2,
\end{align*}
where the second inequality uses Cauchy-Schwarz inequality and the third follows from using \Cref{lem:k_tilde} and \Cref{lem:lipschitz_witness}. 
So we have
\begin{align*}
(i) = \frac{1}{\sqrt{N}} \sqrt{\sum_{i=1}^N \E \left\|\nabla h_{\hat{\mu}_n, \hat{\pi}}^\ast(y_n^{(i)})-\nabla h_{\hat{\mu}_n, \hat{\pi}}^\ast(\overline{y}_n^{(i)})\right\|^2} \leq \frac{2 \sqrt{K K_{2d}}}{\lambda} c_n .
\end{align*}
\subsubsection{Controlling $(ii)$:}
First, we introduce some auxiliary witness functions,
\begin{align*}
    h_{\bar{\mu}_n, \hat{\pi}}^\ast = 2 \left(\Sigma_{\hat{\pi}} + \lambda \Id \right)^{-1} \left(m_{\bar{\mu}_n} - m_{\hat{\pi}} \right), \quad 
    \overline{\overline{h}}_n^\ast = 2 \left(\Sigma_{\pi} + \lambda \Id \right)^{-1} \left(m_{\bar{\mu}_n} - m_{\hat{\pi}} \right) ,
\end{align*}
and for completeness, we recall the witness function we are interested in:
\begin{align*}
    h_{{\mu}_n, {\pi}}^\ast = 2 \left(\Sigma_{\pi} + \lambda \Id \right)^{-1} \left(m_{\mu_n} - m_{\pi} \right), \quad h_{\hat{\mu}_n, \hat{\pi}}^\ast = 2 \left(\Sigma_{\hat{\pi}} + \lambda \Id \right)^{-1} \left(m_{\hat{\mu}_n} - m_{\hat{\pi}} \right) .
\end{align*}
We know that 
\begin{align*}
    (ii) &= \frac{1}{\sqrt{N}} \sqrt{\sum_{i=1}^N \E\left[\left\|\nabla h_{\hat{\mu}_n, \hat{\pi}}^\ast(\overline{y}_n^{(i)}) - \nabla h_{{\mu}_n, {\pi}}^\ast(\overline{y}_n^{(i)})\right\|^2\right]} 
    \\
    &\leq \frac{1}{\sqrt{N}} \sqrt{\sum_{i=1}^N \E\left[\left\|\nabla h_{\hat{\mu}_n, \hat{\pi}}^\ast(\overline{y}_n^{(i)}) - \nabla h_{\bar{\mu}_n, \hat{\pi}}^\ast(\overline{y}_n^{(i)})\right\|^2\right]} \\ 
    &\quad\qquad+ \frac{1}{\sqrt{N}} \sqrt{\sum_{i=1}^N \E\left[\left\|\nabla h_{\bar{\mu}_n, \hat{\pi}}^\ast(\overline{y}_n^{(i)}) - \nabla \overline{\overline{h}}_n^\ast(\overline{y}_n^{(i)})\right\|^2\right]} \\
    &\quad\qquad\quad\qquad+ \frac{1}{\sqrt{N}} \sqrt{\sum_{i=1}^N \E\left[\left\|\nabla \overline{\overline{h}}_n^\ast(\overline{y}_n^{(i)})-\nabla h_{{\mu}_n, {\pi}}^\ast(\overline{y}_n^{(i)})\right\|^2\right]} \\
    &\leq \frac{1}{\sqrt{N}} \Bigg( \sqrt{\sum_{i=1}^N K_{1d} \E \left\|h_{\hat{\mu}_n, \hat{\pi}}^\ast - h_{\bar{\mu}_n, \hat{\pi}}^\ast \right\|_\calH^2 } +  \sqrt{\sum_{i=1}^N K_{1d} \E \left\|h_{\bar{\mu}_n, \hat{\pi}}^\ast - \overline{\overline{h}}_n^\ast \right\|_\calH^2 } \\
    &\quad\qquad+ \sqrt{\sum_{i=1}^N K_{1d} \E \left\|\overline{\overline{h}}_n^\ast - h_{{\mu}_n, {\pi}}^\ast \right\|_\calH^2 } \Bigg) \\
    &= \sqrt{K_{1d}} \left( \sqrt{\E \left\|h_{\hat{\mu}_n, \hat{\pi}}^\ast - h_{\bar{\mu}_n, \hat{\pi}}^\ast \right\|_\calH^2} + \sqrt{\E \left\|h_{\bar{\mu}_n, \hat{\pi}}^\ast - \overline{\overline{h}}_n^\ast \right\|_\calH^2} + \sqrt{\E \left\|\overline{\overline{h}}_n^\ast - h_{{\mu}_n, {\pi}}^\ast \right\|_\calH^2} \right),
\end{align*}
where the first inequality follows from Minkowski's inequality, and the second inequality uses the fact that for $h_0, h_1 \in \calH$,
\begin{align*}
\begin{aligned}
    \left\|\nabla h_1 - \nabla h_0 \right\|_{\calH^d}^2 
    \leq \left\|\nabla_1 k(x,\cdot) \right\|_{\calH^d}^2 \left\| h_1 - h_0 \right\|_\calH^2 \leq K_{1d} \left\| h_1 - h_0 \right\|_\calH^2 . \\
\end{aligned}
\end{align*}
Next, we will bound $\sqrt{\E \|h_{\hat{\mu}_n, \hat{\pi}}^\ast - h_{\bar{\mu}_n, \hat{\pi}}^\ast \|_\calH^2}, \sqrt{\E \|h_{\bar{\mu}_n, \hat{\pi}}^\ast - \overline{\overline{h}}_n^\ast \|_\calH^2}$, and $ \sqrt{\E \|\overline{\overline{h}}_n^\ast - h_{{\mu}_n, {\pi}}^\ast \|_\calH^2}$ separately.

First, by noticing that $h_{\hat{\mu}_n, \hat{\pi}}^\ast = 2 \left( \Sigma_{\hat{\pi}} + \lambda \Id \right)^{-1} \left( m_{\hat{\mu}_n} - m_{\hat{\pi}} \right)$ is the witness function associated with $\dmmd(\hat{\mu}_n || \hat{\pi})$, and $h_{\bar{\mu}_n, \hat{\pi}}^\ast = 2 \left( \Sigma_{\hat{\pi}} + \lambda \Id \right)^{-1} \left( m_{\bar{\mu}_n} - m_{\hat{\pi}} \right)$ is the witness function associated with $\dmmd(\bar{\mu}_n || \hat{\pi})$, by using \Cref{lem:lipschitz_witness}, we have
\begin{align}
    \sqrt{\E \left\|h_{\hat{\mu}_n, \hat{\pi}}^\ast - h_{\bar{\mu}_n, \hat{\pi}}^\ast \right\|_\calH^2} \leq \sqrt{\E \frac{4 K_{1 d}}{\lambda^2} W_2^2\left(\bar{\mu}_n, \hat{\mu}_n \right)} \leq \frac{2 \sqrt{K_{1 d}}}{\lambda} c_n . \label{appeq:first}
\end{align}
Second,
\begin{align}
    &\sqrt{\E \left\|h_{\bar{\mu}_n, \hat{\pi}}^\ast - \overline{\overline{h}}_n^\ast \right\|_\calH^2}
    = \sqrt{ \E \left\| 2 \left( \Sigma_{\hat{\pi}} + \lambda \Id \right)^{-1} \left( m_{\bar{\mu}_n} - m_{\hat{\pi}} \right) - 2 \left( \Sigma_{\pi} + \lambda \Id \right)^{-1} \left( m_{\bar{\mu}_n} - m_{\hat{\pi}} \right) \right \|_\calH^2} \nonumber \\
    &\leq 2 \sqrt{ \E \left\| \left( \Sigma_{\hat{\pi}} + \lambda \Id \right)^{-1} -  \left( \Sigma_{\pi} + \lambda \Id \right)^{-1} \right \|_{\HS}^2 \left\| m_{\bar{\mu}_n} - m_{\hat{\pi}} \right\|_\calH^2 } \nonumber \\ 
    &\leq 4 \sqrt{K} \sqrt{\E \left\| \left( \Sigma_{\hat{\pi}} + \lambda \Id \right)^{-1} - \left( \Sigma_{\pi} + \lambda \Id \right)^{-1} \right\|^2_{\HS}} \nonumber \\
    &= 4 \sqrt{K} \sqrt{\E \left\| \left( \Sigma_{\hat{\pi}} + \lambda \Id \right)^{-1} \Big( (\Sigma_{\hat{\pi}} + \lambda \Id) - (\Sigma_{\pi} + \lambda \Id) \Big) \left( \Sigma_{\pi} + \lambda \Id \right)^{-1} \right\|^2_{\HS}} \nonumber \\
    & \leq 4 \sqrt{K} \frac{1}{\lambda^2} \sqrt{\E \left\|  \Sigma_{\hat{\pi}} - \Sigma_{\pi} \right\|^2_{\HS}} 
     \leq 4 \sqrt{K} \frac{1}{\lambda^2} \sqrt{ \frac{K^2}{M}},\label{appeq:second}
\end{align}
where the last inequality follows from using \Cref{lem:hoeffding} and the fact that $\| k(x, \cdot) \otimes  k(x, \cdot) \|_{\HS} \leq K$.

Third,
\begin{align}
        &\sqrt{\E \left\| \overline{\overline{h}}_n^\ast - h_{{\mu}_n, {\pi}}^\ast \right\|_\calH^2} = \sqrt{ \E \left\| 2 \left( \Sigma_{\pi} + \lambda \Id \right)^{-1} \left( m_{\bar{\mu}_n} - m_{\hat{\pi}} \right) - 2 \left( \Sigma_{\pi} + \lambda \Id \right)^{-1} \left( m_{\mu_n} - m_{\pi} \right) \right \|_\calH^2} \nonumber \\
        &\leq \frac{2}{\lambda} \sqrt{\E \left\| \left( m_{\bar{\mu}_n} - m_{\hat{\pi}} \right) - \left( m_{\mu_n} - m_{\pi} \right) \right \|_\calH^2} \nonumber \\
        &\leq \frac{2}{\lambda} \left( \sqrt{\E \left\|  m_{\bar{\mu}_n} - m_{\mu_n}  \right \|_\calH^2} + \sqrt{\E \left\| m_{\hat{\pi}} - m_{\pi} \right \|_\calH^2} \right) \leq \frac{4}{\lambda} \left( \sqrt{\frac{K}{N}} + \sqrt{\frac{K}{M}} \right),\label{appeq:third}
\end{align}
where the first inequality follows from Cauchy-Schwartz, and the last inequality from \Cref{lem:hoeffding} since $\| k(x, \cdot) \|_\calH \leq \sqrt{K}$.
Therefore, combining \eqref{appeq:first}, \eqref{appeq:second} and \eqref{appeq:third}, we have
\begin{align*}
\begin{aligned}
    (ii) \leq 2 \sqrt{K_{1d}} \left(\frac{ \sqrt{K_{1 d}}}{\lambda} c_n + \frac{2 K^{3/2}}{\sqrt{M} \lambda^2} + \frac{ 2 \sqrt{K}}{\sqrt{N} \lambda} + \frac{2 \sqrt{K}}{\sqrt{M} \lambda}  \right).
\end{aligned}
\end{align*}
Combining $(i)$ and $(ii)$, we have
\begin{small}
\begin{align*}
    c_{n+1} \leq c_n \left(1 + \gamma (1 + \lambda) \frac{2 \sqrt{K K_{2d}}  + 2 K_{1d} }{\lambda} \right) + 2 \gamma (1 + \lambda) \sqrt{K_{1d}} \left(\frac{ 2 K^{3/2} }{\sqrt{M} \lambda^2} + \frac{2 \sqrt{K}}{\sqrt{M} \lambda} + \frac{2 \sqrt{K}}{\sqrt{N} \lambda} \right) .
\end{align*}
\end{small}
Denoting $A = \frac{2 \sqrt{K K_{1 d}}}{ \sqrt{K K_{2 d}} + K_{1 d}}$ and $R = K_{1 d} + \sqrt{K K_{2 d}}$ as constants that only depend on the kernel, and using the discrete Gronwall lemma (Lemma 26 from \citealt{arbel2019maximum}) along with $c_0 = 0$, we obtain
\begin{align*}
    c_{n_{\max}} \leq A \left(\frac{ K }{\sqrt{M} \lambda } + \frac{ 1}{\sqrt{M}} + \frac{1}{\sqrt{N}} \right) \left( \exp\left( \gamma n_{\max} \frac{ 2 (1 + \lambda ) R }{\lambda } \right) - 1 \right) .
\end{align*}
Since $\E W_2(\hat{\mu}_n, \bar{\mu}_n) \leq \sqrt{\E W_2^2(\hat{\mu}_n, \bar{\mu}_n)} \leq c_n$, we reach
\begin{align*}
    \E W_2(\hat{\mu}_{n_{max}}, \bar{\mu}_{n_{max}}) \leq A \left(\frac{ K }{\sqrt{M} \lambda } + \frac{ 1}{\sqrt{M}} + \frac{1}{\sqrt{N}} \right) \left( \exp\left( \gamma n_{\max} \frac{ 2 (1 + \lambda ) R }{\lambda } \right) - 1 \right) .
\end{align*}
Finally, the proof is completed by noting that the r.h.s.~is monotonically decreasing in $\lambda$ and therefore the r.h.s.~can be bounded by replacing $\lambda$ with $\tilde{\lambda} = \min\limits_{i = 1, \ldots, n_{\max}} \lambda_i$.

\subsection{Proof of \Cref{prop:empirical_witness}}\label{appsec:proof_empirical_chard}
By defining the following operators,
\begin{align*}
    &S_x: \calH \rightarrow \mathbb{R}^M, \quad f \rightarrow \frac{1}{\sqrt{M}}\left[ f ( x^{(1)}), \ldots, f ( x^{(M)} ) \right]^{\top}, \\
    &S_x^\ast : \R^M \rightarrow \calH , \quad \alpha \rightarrow \frac{1}{\sqrt{M} } \sum_{i=1}^M \alpha_i k( x^{(i)}, \cdot), \\
    &S_y: \calH \rightarrow \mathbb{R}^N, \quad f \rightarrow \frac{1}{\sqrt{N}}\left[ f ( y^{(1)}), \ldots, f ( y^{(N)} ) \right]^{\top}, \\
    &S_y^\ast : \R^N \rightarrow \calH , \quad \alpha \rightarrow \frac{1}{\sqrt{N}} \sum_{i=1}^N \alpha_i k( y^{(i)}, \cdot) .
\end{align*}
Then we have
\begin{align*}
    \Sigma_{\hat{\pi}} = S_x^\ast S_x, \quad K_{xx} = M S_x S_x^\ast, \quad K_{xy} = \sqrt{MN} S_x S_y^\ast.
\end{align*}
Using these, note that
\begin{align}
h_{\hat{\mu}, \hat{\pi}}^\ast &= 2 \left(\Sigma_{\hat{\pi}}+\lambda \Id\right)^{-1}\left( m_{\hat{\mu}} - m_{\hat{\pi}} \right) \nonumber \\
& =2\left(\frac{1}{M} \sum_{i=1}^M k\left(x^{(i)}, \cdot\right) \otimes k\left(x^{(i)}, \cdot\right)+\lambda \Id\right)^{-1}\left(\frac{1}{N} \sum_{i=1}^N k\left(y^{(i)}, \cdot\right)-\frac{1}{M} \sum_{i=1}^M k\left(x^{(i)}, \cdot\right)\right) \nonumber \\
& =2 \Big( S_x^\ast S_x + \lambda \Id \Big)^{-1}\left( \frac{1}{\sqrt{N} } S_y^\ast \one_N - \frac{1}{\sqrt{M}} S_x^\ast \one_M \right) . \label{eq:empirical_h_temp} 
\end{align}
From the Woodbury inversion lemma, we have that
\begin{align*}
    \Big( S_x^\ast S_x + \lambda \Id \Big)^{-1} = \frac{1}{\lambda} \Id - \frac{1}{\lambda} S_x^\ast (S_x S_x^\ast + \lambda\Id )^{-1} S_x .
\end{align*}
Plugging the above into \eqref{eq:empirical_h_temp}, we obtain
\begin{align}
    h_{\hat{\mu}, \hat{\pi}}^\ast &= 2 \left( \frac{1}{\lambda} \Id - \frac{1}{\lambda} S_x^\ast (S_x S_x^\ast + \lambda\Id )^{-1}  S_x \right) \left( \frac{1}{\sqrt{N} } S_y^\ast \one_N - \frac{1}{\sqrt{M}} S_x^\ast \one_M \right) \nonumber \\
    &= \frac{2}{\lambda} \left( \frac{1}{\sqrt{N} } S_y^\ast \one_N - \frac{1}{\sqrt{M}} S_x^\ast \one_M \right) - \frac{2}{\lambda} S_x^\ast \left( \frac{1}{M} K_{xx} + \lambda\Id \right)^{-1}  \frac{1}{N \sqrt{M} } K_{xy} \one_N \nonumber \\
    &\qquad\qquad+ \frac{2}{\lambda} S_x^\ast \left( \frac{1}{M} K_{xx} + \lambda\Id \right)^{-1} \frac{1}{M \sqrt{M} } K_{xx} \one_M \nonumber \\
    &=\frac{2}{N \lambda} k\left(\cdot, y^{1:N}\right) \one_N - \frac{2}{M \lambda} k\left(\cdot, x^{1:M}\right) \one_M - \frac{2}{N \lambda} k\left(\cdot, x^{1:M}\right) \left( K_{xx} + M \lambda \Id  \right)^{-1} K_{xy} \one_N \nonumber \\
    &\qquad\qquad +\frac{2}{M \lambda} k\left(\cdot, x^{1:M} \right)\left( K_{xx} + M \lambda \Id \right)^{-1} K_{xx} \one_M. \label{appeq:empirical_h}
\end{align}
Obtaining $\dmmd(\hat{\mu} \| \hat{\pi})$ is then easy with $h_{\hat{\mu}, \hat{\pi}}^\ast$ shown in \eqref{appeq:empirical_h}.
\begin{align}\label{appeq:dmmd_samples}
    &\quad \dmmd(\hat{\mu} \| \hat{\pi}) = (1+\lambda)\left\|\left(\Sigma_{\hat{\pi}} + \lambda \Id \right)^{-\frac{1}{2}}\left( m_{\hat{\mu}} - m_{\hat{\pi}} 
    \right)\right\|_{\calH}^2 \nonumber \\
    &= (1 + \lambda) \left\langle\frac{1}{2} h_{\hat{\mu}, \hat{\pi}}^*, m_{\hat{\mu}} - m_{\hat{\pi}} \right\rangle_{\calH} \nonumber \\
    & = (1 + \lambda) \left\langle\frac{1}{2} h_{\hat{\mu}, \hat{\pi}}^*, \frac{1}{N} \sum_{i=1}^N k\left(y^{(i)}, \cdot\right) - \frac{1}{M} \sum_{i=1}^M k\left(x^{(i)}, \cdot\right) \right\rangle_{\calH} \nonumber \\ 
    & =\frac{1 + \lambda }{\lambda} 
    \Bigg( \frac{1}{N^2} \one_N^{\top} K_{yy} \one_N + \frac{1}{M^2} \one_M^{\top} K_{xx} \one_M - \frac{2}{M N} \one_M^{\top} K_{xy}  \one_N \nonumber \\
    &\qquad- \frac{1}{N^2} \one_N^{\top} K_{xy}^\top \left(K_{xx} + M \lambda \Id\right)^{-1} K_{xy}  \one_N +\frac{2}{N M} \one_M^{\top} K_{xx}\left(K_{xx} + M \lambda \Id\right)^{-1} K_{xy}  \one_N \nonumber \\ 
    & \qquad\qquad-\frac{1}{M^2} \one_M^{\top} K_{xx}\left( K_{xx} + M \lambda \Id \right)^{-1} K_{xx} \one_M \Bigg) .
\end{align}

\acks{We would like to thank Gabriele Steidl and Viktor Stein for fruitful discussions on $\dmmd$ and related functionals. Zonghao Chen is supported by the Engineering and Physical Sciences Research Council (EPSRC) through grant [EP/S021566/1]. Aratrika Mustafi and Bharath K. Sriperumbudur are partially supported by the National Science Foundation (NSF) grant DMS-2413425 and NSF CAREER award DMS-1945396. Pierre Glaser and Arthur Gretton are supported by the Gatsby Charitable Foundation. Anna Korba  thanks Google
for their academic support in the form of a gift in support of her
academic research.}
\bibliography{main}

@inproceedings{bonet2024mirror,
  title={Mirror and preconditioned gradient descent in {W}asserstein space},
  author={Bonet, Cl{\'e}ment and Uscidda, Th{\'e}o and David, Adam and Aubin-Frankowski, Pierre-Cyril and Korba, Anna},
  booktitle={Advances in Neural Information Processing Systems},
  volume={37},
  pages={25311--25374},
  year={2024},
  editor={A. Globerson and L. Mackey and D. Belgrave and A. Fan and U. Paquet and J. Tomczak and C. Zhang},
  publisher={Curran Associates, Inc.}
}

@inproceedings{
balasubramanian2024improved,
title={Improved Finite-Particle Convergence Rates for {S}tein Variational Gradient Descent},
author={Sayan Banerjee and Krishna Balasubramanian and Promit Ghosal},
booktitle={The Thirteenth International Conference on Learning Representations},
year={2025},
url={https://openreview.net/forum?id=sbG8qhMjkZ}
}

@article{gronwall1919note,
  title={Note on the derivatives with respect to a parameter of the solutions of a system of differential equations},
  author={Gronwall, Thomas Hakon},
  journal={Annals of Mathematics},
  volume={20},
  number={4},
  pages={292--296},
  year={1919},
  publisher={JSTOR}
}

@inproceedings{nitanda2022convex,
  title={Convex analysis of the mean field {L}angevin dynamics},
  author={Nitanda, Atsushi and Wu, Denny and Suzuki, Taiji},
  booktitle={International Conference on Artificial Intelligence and Statistics},
  pages={9741--9757},
  year={2022},
  editor = {G. Camps-Valls and F. J. R. Ruiz and I. Valera},
  organization={PMLR}
}

@inproceedings{li2017mmd,
  title={{MMD GAN}: {T}owards deeper understanding of moment matching network},
  author={Li, Chun-Liang and Chang, Wei-Cheng and Cheng, Yu and Yang, Yiming and Poczos, Barnabas},
  booktitle={Advances in Neural Information Processing Systems},
  volume={30},
  year={2017},
  editor={I. Guyon and U. Von Luxburg and S. Bengio and H. Wallach and R. Fergus and S. Vishwanathan and R. Garnett},
publisher={Curran Associates, Inc.}
}

@inproceedings{arbel2018gradient,
  title={On gradient regularizers for {MMD GAN}s},
  author={Arbel, Michael and Sutherland, Danica J and Binkowski, Mikolaj and Gretton, Arthur},
  booktitle={Advances in Neural Information Processing Systems},
  volume={31},
  editor={S. Bengio and H. Wallach and H. Larochelle and K. Grauman and N. Cesa-Bianchi and R. Garnett},
  year={2018},
  publisher={Curran Associates, Inc.}
}

@inproceedings{brock2019large,
      title={Large Scale {GAN} Training for High Fidelity Natural Image Synthesis}, 
      author={Andrew Brock and Jeff Donahue and Karen Simonyan},
      year={2019},
  booktitle={International Conference on Learning Representations}
     }

@article{gu2022lipschitz,
  title={Lipschitz-regularized gradient flows and generative particle algorithms for high-dimensional scarce data},
  author={Gu, Hyemin and Birmpa, Panagiota and Pantazis, Yannis and Rey-Bellet, Luc and Katsoulakis, Markos A.},
  journal={SIAM Journal on Mathematics of Data Science},
  volume={6},
  number={4},
  pages={1205--1235},
  year={2024},
  publisher={SIAM}
}

@article{dalalyan2019user,
  title={User-friendly guarantees for the {Langevin Monte Carlo} with inaccurate gradient},
  author={Dalalyan, Arnak S and Karagulyan, Avetik},
  journal={Stochastic Processes and their Applications},
  volume={129},
  number={12},
  pages={5278--5311},
  year={2019},
  publisher={Elsevier}
}

@article{dalalyan2022bounding,
  title={Bounding the error of discretized {L}angevin algorithms for non-strongly log-concave targets},
  author={Dalalyan, Arnak S and Karagulyan, Avetik and Riou-Durand, Lionel},
  journal={Journal of Machine Learning Research},
  volume={23},
  number={235},
  pages={1--38},
  year={2022}
}

@inproceedings{dalalyan2017further,
  title={Further and stronger analogy between sampling and optimization: Langevin {M}onte {C}arlo and gradient descent},
  author={Dalalyan, Arnak},
  booktitle={Conference on Learning Theory},
  pages={678--689},
  year={2017},
  organization={PMLR},
  editor = {S. Kale and O. Shamir},
}

@article{neumayer2024wasserstein,
  title={Wasserstein gradient flows for {M}oreau envelopes of $f$-divergences in reproducing kernel {H}ilbert spaces},
  author={Stein, Viktor and Neumayer, Sebastian and Rux, Nicolaj and Steidl, Gabriele},
  journal={Analysis and Applications},
  pages={1--45},
  year={2025},
  publisher={World Scientific}
}

@inproceedings{ho2020denoising,
 author = {Ho, Jonathan and Jain, Ajay and Abbeel, Pieter},
 booktitle = {Advances in Neural Information Processing Systems},
 pages = {6840--6851},
 title = {Denoising Diffusion Probabilistic Models},
 volume = {33},
 year = {2020},
 editor= {H. Larochelle and M. Ranzato and R. Hadsell and M. F. Balcan and H. Lin},
 publisher={Curran Associates, Inc.}
}

@article{sejdinovic13energy,
author = {Dino Sejdinovic and Bharath Sriperumbudur and Arthur Gretton and Kenji Fukumizu},
title = {{Equivalence of distance-based and RKHS-based statistics in hypothesis testing}},
volume = {41},
journal = {The Annals of Statistics},
number = {5},
publisher = {Institute of Mathematical Statistics},
pages = {2263--2291},
year = {2013}
}

@INPROCEEDINGS{mika99fisher,
  author = {Sebastian Mika and Gunnar R{\"a}tsch and Jason Weston and Bernhard Sch{\"o}lkopf and Klaus-Robert M{\"u}ller},
  booktitle={Neural Networks for Signal Processing IX: Proceedings of the 1999 IEEE Signal Processing Society Workshop (Cat. No.98TH8468)}, 
  title={Fisher discriminant analysis with kernels}, 
  year={1999},
  pages={41-48}
}

@article{gretton2012kernel, title={A kernel two-sample test}, author={Gretton, Arthur and Borgwardt, Karsten M. and Rasch, Malte J. and Sch{\"o}lkopf, Bernhard and Smola, Alexander}, journal={Journal of Machine Learning Research}, volume={13}, number={1}, pages={723--773}, year={2012}, publisher={JMLR. org} }

@misc{boufadene2023global,
      title={On the global convergence of {W}asserstein gradient flow of the {C}oulomb discrepancy}, 
      author={Siwan Boufadène and François-Xavier Vialard},
      year={2024},
      eprint={2312.00800},
      archivePrefix={arXiv},
      primaryClass={math.AP},
      url={https://arxiv.org/abs/2312.00800}, 
}

@article{hertrich2024wasserstein,
  title={Wasserstein steepest descent flows of discrepancies with {R}iesz kernels},
  author={Hertrich, Johannes and Gr{\"a}f, Manuel and Beinert, Robert and Steidl, Gabriele},
  journal={Journal of Mathematical Analysis and Applications},
  volume={531},
  number={1},
  pages={127829},
  year={2024},
  publisher={Elsevier}
}

@inproceedings{
ansari2020refining,
title={Refining Deep Generative Models via Discriminator Gradient Flow},
author={Abdul Fatir Ansari and Ming Liang Ang and Harold Soh},
booktitle={International Conference on Learning Representations},
year={2021},
url={https://openreview.net/forum?id=Zbc-ue9p_rE}
}

@inproceedings{simons2022variational,
  title={Variational Likelihood-Free Gradient Descent},
  author={Simons, Jack and Liu, Song and Beaumont, Mark},
  booktitle={Fourth Symposium on Advances in Approximate Bayesian Inference},
  year={2022}
}

@inproceedings{gao2019deep,
  title={Deep generative learning via variational gradient flow},
  author={Gao, Yuan and Jiao, Yuling and Wang, Yang and Wang, Yao and Yang, Can and Zhang, Shunkang},
  booktitle={International Conference on Machine Learning},
  pages={2093--2101},
  year={2019},
  publisher={PMLR},
  editor = {K. Chaudhuri and R. Salakhutdinov},
}

@article{fischer2020sobolev,
  title={Sobolev norm learning rates for regularized least-squares algorithms},
  author={Fischer, Simon and Steinwart, Ingo},
  journal={Journal of Machine Learning Research},
  volume={21},
  number={205},
  pages={1--38},
  year={2020}
}

@article{aronszajn1950theory,
  title={Theory of reproducing kernels},
  author={Aronszajn, Nachman},
  journal={Transactions of the American Mathematical Society},
  volume={68},
  number={3},
  pages={337--404},
  year={1950}
}

@inproceedings{balasubramanian2022towards,
  title={Towards a theory of non-log-concave sampling: {F}irst-order stationarity guarantees for {Langevin Monte Carlo}},
  author={Balasubramanian, Krishna and Chewi, Sinho and Erdogdu, Murat A. and Salim, Adil and Zhang, Shunshi},
  booktitle={Conference on Learning Theory},
  pages={2896--2923},
  year={2022},
  organization={PMLR},
  editor={Po-Ling Loh and Maxim Raginsky}
}

@article{van2014renyi,
  title={R{\'e}nyi divergence and {Kullback-Leibler} divergence},
  author={Van Erven, Tim and Harremos, Peter},
  journal={IEEE Transactions on Information Theory},
  volume={60},
  number={7},
  pages={3797--3820},
  year={2014},
  publisher={IEEE}
}

@article{birrell2022f,
  title={$(f, \text{$\Gamma$})$-Divergences: Interpolating between $f$-Divergences and Integral Probability Metrics},
  author={Birrell, Jeremiah and Dupuis, Paul and Katsoulakis, Markos A and Pantazis, Yannis and Rey-Bellet, Luc},
  journal={Journal of Machine Learning Research},
  volume={23},
  number={39},
  pages={1--70},
  year={2022}
}

@inproceedings{vempala2019rapid,
  title={Rapid convergence of the unadjusted {L}angevin algorithm: {I}soperimetry suffices},
  author={Vempala, Santosh and Wibisono, Andre},
  booktitle={Advances in Neural Information Processing Systems},
  volume={32},
  year={2019},
  editor={ H. Wallach and H. Larochelle and A. Beygelzimer and F. d'Alché-Buc and E. Fox and R. Garnett},
  publisher={Curran Associates, Inc.}
}

@book{boyd2004convex, title={Convex Optimization}, author={Boyd, Stephen P. and Vandenberghe, Lieven}, year={2004}, publisher={Cambridge University Press} }

@book{klenke2013probability,
  title={Probability Theory: A Comprehensive Course},
  author={Klenke, Achim},
  year={2013},
  publisher={Springer Science \& Business Media}
}

@book{panaretos2020wasserstein,
 author={Panaretos, Victor M and Zemel, Yoav},
 title={An Invitation to Statistics in Wasserstein Space},
year={2020},
  publisher={Springer}
}

@article{nguyen2010estimating,
  title={Estimating divergence functionals and the likelihood ratio by convex risk minimization},
  author={Nguyen, XuanLong and Wainwright, Martin J. and Jordan, Michael I.},
  journal={IEEE Transactions on Information Theory},
  volume={56},
  number={11},
  pages={5847--5861},
  year={2010},
  publisher={IEEE}
}

@software{jax2018github,
  author = {James Bradbury and Roy Frostig and Peter Hawkins and Matthew James Johnson and Chris Leary and Dougal Maclaurin and George Necula and Adam Paszke and Jake Vander{P}las and Skye Wanderman-{M}ilne and Qiao Zhang},
  title = {{JAX}: composable transformations of {P}ython+{N}um{P}y programs},
  url = {http://github.com/google/jax},
  version = {0.3.13},
  year = {2018},
}

@book{cucker2007learning,
  title={Learning {T}heory: {A}n Approximation Theory Viewpoint},
  author={Cucker, Felipe and Zhou, Ding Xuan},
  volume={24},
  year={2007},
  publisher={Cambridge University Press}
}

@inproceedings{hertrich2023wasserstein,
  title={Wasserstein gradient flows of the discrepancy with distance kernel on the line},
  author={Hertrich, Johannes and Beinert, Robert and Gr{\"a}f, Manuel and Steidl, Gabriele},
  booktitle={International Conference on Scale Space and Variational Methods in Computer Vision},
  pages={431--443},
  year={2023},
  organization={Springer}
}

@inproceedings{liutkus2019sliced,
  title={Sliced-{W}asserstein flows: {N}onparametric generative modeling via optimal transport and diffusions},
  author={Liutkus, Antoine and Simsekli, Umut and Majewski, Szymon and Durmus, Alain and St{\"o}ter, Fabian-Robert},
  booktitle={International Conference on Machine Learning},
  pages={4104--4113},
  year={2019},
  publisher={PMLR},
  editor = {K. Chaudhuri and R. Salakhutdinov},
}

@inproceedings{
hertrich2023generative,
title={Generative Sliced {MMD} Flows with Riesz Kernels},
author={Johannes Hertrich and Christian Wald and Fabian Altekr{\"u}ger and Paul Hagemann},
booktitle={The Twelfth International Conference on Learning Representations},
year={2024},
url={https://openreview.net/forum?id=VdkGRV1vcf}
}

@inproceedings{genevay2018learning,
  title={Learning generative models with {S}inkhorn divergences},
  author={Genevay, Aude and Peyr{\'e}, Gabriel and Cuturi, Marco},
  booktitle={International Conference on Artificial Intelligence and Statistics},
  pages={1608--1617},
  year={2018},
  organization={PMLR},
  editor = {A. Storkey and F. Perez-Cruz},
}

@inproceedings{franceschi2023unifying,
  title={Unifying {GAN}s and score-based diffusion as generative particle models},
  author={Franceschi, Jean-Yves and Gartrell, Mike and Dos Santos, Ludovic and Issenhuth, Thibaut and de B{\'e}zenac, Emmanuel and Chen, Micka{\"e}l and Rakotomamonjy, Alain},
  booktitle={Advances in Neural Information Processing Systems},
  volume={36},
  year={2024},
  editor={A. Globerson and L. Mackey and D. Belgrave and A. Fan and U. Paquet and J. Tomczak and C. Zhang},
  publisher={Curran Associates, Inc.}
}

@inproceedings{nowozin2016f,
  title={$f$-{GAN}: {T}raining generative neural samplers using variational divergence minimization},
  author={Nowozin, Sebastian and Cseke, Botond and Tomioka, Ryota},
  booktitle={Advances in Neural Information Processing Systems},
  volume={29},
  year={2016},
  editor={D. Lee and M. Sugiyama and U. Luxburg and I. Guyon and R. Garnett},
  publisher={Curran Associates, Inc.}
}

@article{craig2023blob,
  title={A blob method for inhomogeneous diffusion with applications to multi-agent control and sampling},
  author={Craig, Katy and Elamvazhuthi, Karthik and Haberland, Matt and Turanova, Olga},
  journal={Mathematics of Computation},
  year={2023}
}

@article{jordan1998variational,
  title={The variational formulation of the {F}okker--{P}lanck equation},
  author={Jordan, Richard and Kinderlehrer, David and Otto, Felix},
  journal={SIAM Journal on Mathematical Analysis},
  volume={29},
  number={1},
  pages={1--17},
  year={1998},
  publisher={SIAM}
}

@article{craig2023nonlocal,
  title={Nonlocal approximation of slow and fast diffusion},
  author={Craig, Katy and Jacobs, Matt and Turanova, Olga},
  journal={Journal of Differential Equations},
  volume={426},
  pages={782--852},
  year={2025},
  publisher={Elsevier}
}

@inproceedings{liu2017stein,
  title={Stein variational gradient descent as gradient flow},
  author={Liu, Qiang},
  booktitle={Advances in Neural Information Processing Systems},
  volume={30},
  year={2017},
  editor={I. Guyon and U. Von Luxburg and S. Bengio and H. Wallach and R. Fergus and S. Vishwanathan and R. Garnett},
  publisher={Curran Associates, Inc.}
}

@article{ohta2011displacement,
  title={Displacement convexity of generalized relative entropies},
  author={Ohta, Shin-ichi and Takatsu, Asuka},
  journal={Advances in Mathematics},
  volume={228},
  number={3},
  pages={1742--1787},
  year={2011},
  publisher={Elsevier}
}

@inproceedings{eric2007testing,
  title={Testing for homogeneity with kernel {F}isher discriminant analysis},
  author={Harchaoui, Za{\"\i}d and Bach, Francis and Moulines, Eric},
  booktitle={Advances in Neural Information Processing Systems},
  volume={20},
  year={2007},
  editor={J. Platt and D. Koller and Y. Singer and S. Roweis},
  publisher={Curran Associates, Inc.}
}

@article{santambrogio2017euclidean,
  title={\{Euclidean, metric, and {W}asserstein\} gradient flows: {A}n overview},
  author={Santambrogio, Filippo},
  journal={Bulletin of Mathematical Sciences},
  volume={7},
  pages={87--154},
  year={2017},
  publisher={Springer}
}

@article{hagrass2022spectral,
  title={Spectral regularized kernel two-sample tests},
  author={Hagrass, Omar and K. Sriperumbudur, Bharath and Li, Bing},
  journal={The Annals of Statistics},
  volume={52},
  number={3},
  pages={1076--1101},
  year={2024},
  publisher={Institute of Mathematical Statistics}
}

@inproceedings{kac1956foundations,
  title={Foundations of kinetic theory},
  author={Kac, Mark},
  booktitle={Proceedings of the Third Berkeley Symposium on Mathematical Statistics and Probability},
  volume={3},
  pages={171--197},
  year={1956}
}

@book{villani2009optimal,
  title={Optimal {T}ransport: {O}ld and {N}ew},
  author={Villani, C{\'e}dric and others},
  volume={338},
  year={2009},
  publisher={Springer}
}

@article{durmus2019analysis,
  title={Analysis of {L}angevin {M}onte {C}arlo via convex optimization},
  author={Durmus, Alain and Majewski, Szymon and Miasojedow, Bl{\.{a}}zej},
  journal={Journal of Machine Learning Research},
  volume={20},
  number={1},
  pages={2666--2711},
  year={2019},
  publisher={JMLR. org}
}

@book{ambrosio2005gradient,
  title={Gradient Flows: In Metric Spaces and in the Space of Probability Measures},
  author={Ambrosio, Luigi and Gigli, Nicola and Savar{\'e}, Giuseppe},
  year={2005},
  journal={Springer Science \& Business Media},
  publisher={Springer Science \& Business Media}
}

@inproceedings{glaser2021kale,
  title={K{ALE} flow: A relaxed {KL} gradient flow for probabilities with disjoint support},
  author={Glaser, Pierre and Arbel, Michael and Gretton, Arthur},
  booktitle={Advances in Neural Information Processing Systems},
  volume={34},
  pages={8018--8031},
  year={2021},
  editor={M. Ranzato and A. Beygelzimer and Y. Dauphin and P.S. Liang and J. Wortman Vaughan},
  publisher={Curran Associates, Inc.}
}

@article{agrawal2021optimal,
  title={Optimal bounds between $f$-divergences and integral probability metrics},
  author={Agrawal, Rohit and Horel, Thibaut},
  journal={Journal of Machine Learning Research},
  volume={22},
  number={1},
  pages={5662--5720},
  year={2021},
  publisher={JMLRORG}
}

@inproceedings{arbel2019maximum,
  title={Maximum mean discrepancy gradient flow},
  author={Arbel, Michael and Korba, Anna and Salim, Adil and Gretton, Arthur},
  booktitle={Advances in Neural Information Processing Systems},
  volume={32},
  year={2019},
  editor={H. Wallach and H. Larochelle and A. Beygelzimer and F. d'Alché-Buc and E. Fox and R. Garnett},
  publisher={Curran Associates, Inc.}
}

@article{epanechnikov1969non,
  title={Non-parametric estimation of a multivariate probability density},
  author={Epanechnikov, Vassiliy A},
  journal={Theory of Probability \& Its Applications},
  volume={14},
  number={1},
  pages={153--158},
  year={1969},
  publisher={SIAM}
}

@book{steinwart2008support,
  title={Support {V}ector {M}achines},
  author={Steinwart, Ingo and Christmann, Andreas},
  year={2008},
  publisher={Springer Science \& Business Media}
}

@inproceedings{korba2020non,
  title={A non-asymptotic analysis for {S}tein variational gradient descent},
  author={Korba, Anna and Salim, Adil and Arbel, Michael and Luise, Giulia and Gretton, Arthur},
  booktitle={Advances in Neural Information Processing Systems},
  volume={33},
  pages={4672--4682},
  year={2020},
  editor={H. Larochelle and M. Ranzato and R. Hadsell and M. F. Balcan and H. Lin},
publisher={Curran Associates, Inc.}
}

@article{fournier2015rate,
  title={On the rate of convergence in {W}asserstein distance of the empirical measure},
  author={Fournier, Nicolas and Guillin, Arnaud},
  journal={Probability {T}heory and {R}elated {F}ields},
  volume={162},
  number={3-4},
  pages={707--738},
  year={2015},
  publisher={Springer}
}

@book{rudin1976principles,
  title={Principles of {M}athematical {A}nalysis},
  author={Rudin, Walter},
  volume={3},
  year={1976},
  publisher={McGraw-Hill New York}
}

@article{steinwart2012mercer,
  title={Mercer’s theorem on general domains: On the interaction between measures, kernels, and {RKHS}s},
  author={Steinwart, Ingo and Scovel, Clint},
  journal={Constructive Approximation},
  volume={35},
  pages={363--417},
  year={2012},
  publisher={Springer}
}

@inproceedings{chizat2018global,
  title={On the global convergence of gradient descent for over-parameterized models using optimal transport},
  author={Chizat, Lenaic and Bach, Francis},
  booktitle={Advances in Neural Information Processing Systems},
  volume={31},
  year={2018},
  editor={S. Bengio and H. Wallach and H. Larochelle and K. Grauman and N. Cesa-Bianchi and R. Garnett},
  publisher={Curran Associates, Inc.}
}

@inproceedings{liu2016stein, title={Stein variational gradient descent: A general purpose {B}ayesian inference algorithm}, author={Liu, Qiang and Wang, Dilin}, booktitle={Advances in Neural Information Processing Systems}, volume={29}, year={2016}, editor={D. Lee and M. Sugiyama and U. Luxburg and I. Guyon and R. Garnett}, publisher={Curran Associates, Inc.} }

@inproceedings{chewi2020svgd,
  title={{SVGD} as a kernelized {W}asserstein gradient flow of the chi-squared divergence},
  author={Chewi, Sinho and Le Gouic, Thibaut and Lu, Chen and Maunu, Tyler and Rigollet, Philippe},
  booktitle={Advances in Neural Information Processing Systems},
  volume={33},
  pages={2098--2109},
  year={2020},
  editor={H. Larochelle and M. Ranzato and R. Hadsell and M. F. Balcan and H. Lin},
  publisher={Curran Associates, Inc.}
}

@article{he2022regularized,
  title={Regularized {S}tein variational gradient flow},
  author={He, Ye and Balasubramanian, Krishna and Sriperumbudur, Bharath K. and Lu, Jianfeng},
  journal={Foundations of Computational Mathematics},
  pages={1--59},
  year={2024},
  publisher={Springer}
}

@article{garcia2020bayesian,
  title={The {B}ayesian update: {V}ariational Formulations and Gradient Flows},
  author={Trillos, Nicolas Garcia and Sanz-Alonso, Daniel},
  journal={Bayesian Analysis},
  volume={15},
  number={1},
  pages={29--56},
  year={2020}
}

@inproceedings{
li2023sampling,
title={Sampling with Mollified Interaction Energy Descent},
author={Lingxiao Li and Qiang Liu and Anna Korba and Mikhail Yurochkin and Justin Solomon},
booktitle={The Eleventh International Conference on Learning Representations },
year={2023},
url={https://openreview.net/forum?id=zWy7dqOcel}
}

@inproceedings{korba2021kernel,
  title={Kernel {S}tein discrepancy descent},
  author={Korba, Anna and Aubin-Frankowski, Pierre-Cyril and Majewski, Szymon and Ablin, Pierre},
  booktitle={International Conference on Machine Learning},
  pages={5719--5730},
  year={2021},
  organization={PMLR},
  editor = {M. Meila and T. Zhang},
}

@article{duncan2019geometry,
  title={On the geometry of {S}tein variational gradient descent},
  author={Duncan, Andrew and N{\"u}sken, Nikolas and Szpruch, Lukasz},
  journal={Journal of Machine Learning Research},
  volume={24},
  number={56},
  pages={1--39},
  year={2023}
}

@article{balasubramanian2017optimality,
  author  = {Krishna Balasubramanian and Tong Li and Ming Yuan},
  title   = {On the Optimality of Kernel-Embedding Based goodness-of-fit Tests},
  journal = {Journal of Machine Learning Research},
  year    = {2021},
  volume  = {22},
  number  = {1},
  pages   = {1--45}
}

@inproceedings{liu2023variational,
  title={Minimizing $ f $-Divergences by Interpolating Velocity Fields},
  author={Liu, Song and Yu, Jiahao and Simons, Jack and Yi, Mingxuan and Beaumont, Mark},
  booktitle={International Conference on Machine Learning},
  pages={32308--32331},
  year={2024},
  organization={PMLR},
  editor = {R. Salakhutdinov and Z. Kolter and K. Heller and A. Weller and N. Oliver and J. Scarlett and F. Berkenkamp},
}

@inproceedings{song2020score,
  title={Score-Based Generative Modeling through Stochastic Differential Equations},
  author={Song, Yang and Sohl-Dickstein, Jascha and Kingma, Diederik P and Kumar, Abhishek and Ermon, Stefano and Poole, Ben},
  booktitle={International Conference on Learning Representations},
  year={2021}
}

@article{hagrass2023spectralgof,
  title={Spectral regularized kernel goodness-of-fit tests},
  author={Hagrass, Omar and Sriperumbudur, Bharath K. and Li, Bing},
  journal={Journal of Machine Learning Research},
  volume={25},
  number={309},
  pages={1--52},
  year={2024}
}

@article{sriperumbudur2010hilbert,
  title={Hilbert space embeddings and metrics on probability measures},
  author={Sriperumbudur, Bharath K. and Gretton, Arthur and Fukumizu, Kenji and Sch{\"o}lkopf, Bernhard and Lanckriet, Gert R. G.},
  journal={Journal of Machine Learning Research},
  volume={11},
  pages={1517--1561},
  year={2010},
  publisher={JMLR. org}
}

@article{shi2009data,
  title={Data spectroscopy: Eigenspaces of convolution operators and clustering},
  author={Shi, Tao and Belkin, Mikhail and Yu, Bin},
  journal={The Annals of Statistics},
  pages={3960--3984},
  year={2009},
  publisher={JSTOR}
}

@article{liang2022mehler,
  title={Mehler’s formula, branching process, and compositional kernels of deep neural networks},
  author={Liang, Tengyuan and Tran-Bach, Hai},
  journal={Journal of the American Statistical Association},
  volume={117},
  number={539},
  pages={1324--1337},
  year={2022},
  publisher={Taylor \& Francis}
}

@book{engl1996regularization,
  title={Regularization of Inverse Problems},
  author={Engl, Heinz Werner and Hanke, Martin and Neubauer, Andreas},
  volume={375},
  year={1996},
  publisher={Springer Science \& Business Media}
}

@inproceedings{lambert2022variational,
  title={Variational inference via {W}asserstein gradient flows},
  author={Lambert, Marc and Chewi, Sinho and Bach, Francis and Bonnabel, Silv{\`e}re and Rigollet, Philippe},
  booktitle={Advances in Neural Information Processing Systems},
  volume={35},
  pages={14434--14447},
  year={2022},
  editor={S. Koyejo and S. Mohamed and A. Agarwal and D. Belgrave and K. Cho and A. Oh},
  publisher={Curran Associates, Inc.}
}

@article{chewi2024analysis,
  title={Analysis of {L}angevin {M}onte {C}arlo from {P}oincare to log-{S}obolev},
  author={Chewi, Sinho and Erdogdu, Murat A. and Li, Mufan and Shen, Ruoqi and Zhang, Matthew S.},
  journal={Foundations of Computational Mathematics},
  pages={1--51},
  year={2024},
  publisher={Springer}
}

@inproceedings{liu2024towards,
  title={Towards understanding the dynamics of {Gaussian-Stein} variational gradient descent},
  author={Liu, Tianle and Ghosal, Promit and Balasubramanian, Krishna and Pillai, Natesh},
  booktitle={Advances in Neural Information Processing Systems},
  volume={36},
  year={2024}
}

@article{belafhal2020note,
  title={A note on some integrals involving {H}ermite polynomials and their applications},
  author={Belafhal, A and Hricha, Z and Dalil-Essakali, L and Usman, T},
  journal={Advanced Mathematical Models and Applications},
  volume={5},
  number={3},
  pages={313--319},
  year={2020}
}

@article{simon2023metrizing,
  title={Metrizing weak convergence with maximum mean discrepancies},
  author={Simon-Gabriel, Carl-Johann and Barp, Alessandro and Sch{\"o}lkopf, Bernhard and Mackey, Lester},
  journal={Journal of Machine Learning Research},
  volume={24},
  number={184},
  pages={1--20},
  year={2023}
}

@article{BAUER200752,
title = {On regularization algorithms in learning theory},
journal = {Journal of Complexity},
volume = {23},
number = {1},
pages = {52-72},
year = {2007},
author = {Frank Bauer and Sergei Pereverzev and Lorenzo Rosasco}
}

@inproceedings{chatterji2020langevin,
  title={Langevin {M}onte {C}arlo without smoothness},
  author={Chatterji, Niladri and Diakonikolas, Jelena and Jordan, Michael I. and Bartlett, Peter},
  booktitle={International Conference on Artificial Intelligence and Statistics},
  pages={1716--1726},
  year={2020},
  organization={PMLR},
  editor = {S. Chiappa and R. Calandra},
}

@article{sriperumbudur2011universality, title={Universality, Characteristic Kernels and {RKHS} Embedding of Measures.}, author={Sriperumbudur, Bharath K. and Fukumizu, Kenji and Lanckriet, Gert R. G.}, journal={Journal of Machine Learning Research}, volume={12}, number={7}, 
pages={2389--2410},
year={2011} }

@book{wendland2004scattered,
  title={Scattered Data Approximation},
  author={Wendland, Holger},
  volume={17},
  year={2004},
  publisher={Cambridge university press}
}

@article{carmeli2010vector,
  title={Vector-valued reproducing kernel {H}ilbert spaces and universality},
  author={Carmeli, Claudio and De Vito, Ernesto and Toigo, Alessandro and Umanit{\'a}, Veronica},
  journal={Analysis and Applications},
  volume={8},
  number={01},
  pages={19--61},
  year={2010},
  publisher={World Scientific}
}

@inproceedings{renyi1961measures,
  title={On measures of entropy and information},
  author={R{\'e}nyi, Alfr{\'e}d},
  booktitle={Proceedings of the Fourth Berkeley Symposium on Mathematical Statistics and Probability, volume 1: Contributions to the Theory of Statistics},
  volume={4},
  pages={547--562},
  year={1961},
  organization={University of California Press}
}

@article{kruger2003frechet,
  title={On {F}r{\'e}chet subdifferentials},
  author={Kruger, A. Ya.},
  journal={Journal of Mathematical Sciences},
  volume={116},
  number={3},
  pages={3325--3358},
  year={2003},
  publisher={Springer}
}

@article{sriperumbudur2016optimal,
  title={On the optimal estimation of probability measures in weak and strong topologies},
  author={Sriperumbudur, Bharath K.},
  journal={Bernoulli},
  pages={1839--1893},
volume={22},
number={3},
year={2016}
}

@article{chen2025stationary,
  title={Stationary {MMD} Points for Cubature},
  author={Chen, Zonghao and Karvonen, Toni and Kanagawa, Heishiro and Briol, Fran{\c{c}}ois-Xavier and Oates, Chris},
  journal={https://arxiv.org/abs/2505.20754},
  year={2025}
}

@inproceedings{
galashov2024deep,
title={Deep {MMD} Gradient Flow without adversarial training},
author={Alexandre Galashov and Valentin De Bortoli and Arthur Gretton},
booktitle={The Thirteenth International Conference on Learning Representations},
year={2025},
url={https://openreview.net/forum?id=Pf85K2wtz8}
}

@article{caponnetto2007optimal,
  title={Optimal rates for the regularized least-squares algorithm},
  author={Caponnetto, Andrea and De Vito, Ernesto},
  journal={Foundations of Computational Mathematics},
  volume={7},
  pages={331--368},
  year={2007},
  publisher={Springer}
}

@book{scholkopf2002learning,
  title={Learning with Kernels: Support Vector Machines, Regularization, Optimization, and Beyond},
  author={Sch{\"o}lkopf, Bernhard and Smola, Alexander J},
  year={2002},
  publisher={MIT press}
}

@inproceedings{shi2024finite,
  title={A finite-particle convergence rate for {S}tein variational gradient descent},
  author={Shi, Jiaxin and Mackey, Lester},
  booktitle={Advances in Neural Information Processing Systems},
  volume={36},
  year={2024},
  editor={A. Globerson and L. Mackey and D. Belgrave and A. Fan and U. Paquet and J. Tomczak and C. Zhang},
  publisher={Curran Associates, Inc. }
}

@article{lei2020convergence,
  title={Convergence and concentration of empirical measures under {W}asserstein distance in unbounded functional spaces},
  author={Lei, Jing},
  journal={Bernoulli},
  volume={26},
  number={1},
  pages={767--798},
  year={2020}
}

@article{kloeckner2012approximation,
  title={Approximation by finitely supported measures},
  author={Kloeckner, Benoit},
  journal={ESAIM: Control, Optimisation and Calculus of Variations},
  volume={18},
  number={2},
  pages={343--359},
  year={2012}
}

@inproceedings{feydy2019interpolating,
  title={Interpolating between optimal transport and {MMD} using {S}inkhorn divergences},
  author={Feydy, Jean and S{\'e}journ{\'e}, Thibault and Vialard, Fran{\c{c}}ois-Xavier and Amari, Shun-ichi and Trouv{\'e}, Alain and Peyr{\'e}, Gabriel},
  booktitle={International Conference on Artificial Intelligence and Statistics},
  pages={2681--2690},
  year={2019},
  organization={PMLR},
  editor = {K. Chaudhuri and M. Sugiyama},
}

@inproceedings{
suzuki2023uniform,
title={Uniform-in-time propagation of chaos for the mean-field gradient {L}angevin dynamics},
author={Taiji Suzuki and Atsushi Nitanda and Denny Wu},
booktitle={The Eleventh International Conference on Learning Representations },
year={2023},
url={https://openreview.net/forum?id=_JScUk9TBUn}
}

@article{chen2024uniform,
  title={Uniform-in-time propagation of chaos for kinetic mean field {L}angevin dynamics},
  author={Chen, Fan and Lin, Yiqing and Ren, Zhenjie and Wang, Songbo},
  journal={Electronic Journal of Probability},
  volume={29},
  pages={1--43},
  year={2024},
  publisher={The Institute of Mathematical Statistics and the Bernoulli Society}
}

@inproceedings{pillaud2020statistical,
  title={Statistical estimation of the {P}oincar{\'e} constant and application to sampling multimodal distributions},
  author={Pillaud-Vivien, Loucas and Bach, Francis and Leli{\`e}vre, Tony and Rudi, Alessandro and Stoltz, Gabriel},
  booktitle={International Conference on Artificial Intelligence and Statistics},
  pages={2753--2763},
  year={2020},
  editor = {S. Chiappa and R. Calandra},
  organization={PMLR}
}
\renewcommand{\clearpage}{}

\begin{appendices}

\crefalias{section}{appendix}
\crefalias{subsection}{appendix}
\crefalias{subsubsection}{appendix}

\setcounter{equation}{0}
\renewcommand{\theequation}{\thesection.\arabic{equation}}
\newcommand{\appsection}[1]{
  \refstepcounter{section}
  \section*{Appendix \thesection: #1}
  \addcontentsline{toc}{section}{Appendix \thesection: #1}
}

\onecolumn

\appsection{Further Background on $(\calP_2(\R^d), W_2)$ }\label{appsec:wass_discussion}
Let $\mu$ and $\pi$ be two probability measures in $\cP_2(\R^d)$ and let $\Pi (\mu,\pi)$ denote the set of all admissible transport plans between $\mu$ and $\pi$, i.e., $\Pi (\mu, \pi) =\{ \Gamma \in \cP(\R^d\times \mathbb{R}^d); \left( \operatorname{proj}_{1}\right)_{\#} \Gamma=\mu,\left(\operatorname{proj}_{2}\right)_{\#} \Gamma=\pi\}$, where $\operatorname{proj}_{1}$ and $\operatorname{proj}_{2}$ respectively stand for projection maps $(x, y) \mapsto x$ and $(x, y) \mapsto y$, and $\#$ is the pushforward operator. The Wasserstein-2 distance between $\mu$ and $\pi$ is then defined as
\begin{equation}\label{appeq:Wasserstein_2_distance}
W_2(\mu, \pi)=\left(\inf _{\Gamma \in \Pi(\mu, \pi)} \int \| x-y \|^2 d \Gamma(x, y)\right)^{\frac{1}{2}},\nonumber
\end{equation}
and $(\cP_2(\R^d), W_2)$ is a metric space called the Wasserstein space \citep{panaretos2020wasserstein}.
Brenier's theorem guarantees that if $\mu$ is an absolutely continuous measure, then the optimal transport map is unique and is of the form $\Gamma^{\ast} = (\Id, T)_\# \mu$, i.e., $T_\# \mu = \pi$~\citep{santambrogio2017euclidean}. $T$ can also be expressed as $T(x) = x + \nabla \phi(x)$, where $\phi$ is known as the Kantorovich potential function and is differentiable $\mu-$a.e. 

For an absolutely continuous $\mu$ and the optimal transport plan $T$ such that $T_\# \mu = \pi$, the shortest path $(\rho_{t})_{0\leq t \leq 1}$ from $\mu$ to $\pi$ is called the (Wasserstein) geodesic given by the following form:
\begin{align*}
\rho_{t} = \left((1-t) \Id + t T \right)_{\#} \mu = \left( \Id + t \nabla \phi \right)_{\#} \mu. 
\end{align*}
Therefore, in this paper, we always use $(\Id + t \nabla \phi)_\# \mu$ with 
$\phi \in C_c^\infty(\R^d)$\footnote{$\phi$ is compactly supported because the tangent space of $\mu \in\calP_2(\R^d)$ is $\overline{ \{\nabla \psi, \psi \in C_c^{\infty} (\R^d)\} }^{L^2(\mu)}$~\citep[Definition 8.4.1]{ambrosio2005gradient}. } to define a geodesic curve that starts at $\mu$.
Define $\varphi_t: \R^d \to \R^d, x \mapsto x + t \nabla \phi(x)$, then $\omega_t: \R^d \to \R^d, x \mapsto \left[\nabla \phi \circ \varphi_t^{-1}\right](x)$ becomes the optimal transport map from $\rho_t$ to $\pi$. Notice that $\| \omega_t \|_{L^2(\rho_t)} = \| \nabla \phi \|_{L^2(\mu)}$ for all $t\in [0,1]$, so
$(\rho_{t})_{0\leq t \leq 1}$ is also a constant-speed geodesic.
The notion of a constant-speed geodesic is crucial in the introduction of geodesic convexity below.

A functional $\mathcal{F}: \calP_2(\R^d) \to \R$ is geodesically convex if for any $\mu$ and $\pi$, the following inequality holds:
\begin{align}\label{appeq:displacement_convex_defi}
    \mathcal{F}\left(\rho_t\right) \leq(1-t) \mathcal{F}\left(\mu \right)+t \mathcal{F}\left(\pi \right), \quad \forall t \in[0,1],
\end{align}
where $\left(\rho_t\right)_{t \in[0,1]}$ is the constant-speed geodesic between $\mu$ and $\pi$. 
The geodesic convexity of $\calF$ can be equivalently characterized through the Wasserstein Hessian~\citep{villani2009optimal}. 
The geodesic convexity ought not to be confused with mixture convexity, which replaces displacement geodesic $\rho_t$ in \eqref{appeq:displacement_convex_defi} with the mixture geodesic $\nu_t = (1-t) \mu + t \pi$. 
\setcounter{equation}{0}
\appsection{Auxiliary Results}
In this appendix, we collect all technical results required to prove the main results of the paper.
\begin{lem}\label{lem:variation_chi2}
For $\mu \ll \pi$, $\chi^2$-divergence admits the following variational form:
\begin{align*}
    \chi^2(\mu \| \pi)=\sup _{h \in L^2(\pi)} \int h d \mu-\int\left(h+\frac{1}{4} h^2\right) d \pi,
\end{align*}
where it is sufficient to restrict the variational set to $L^2(\pi)$ in contrast to the set of all measurable functions for general $f$-divergences \citep{nowozin2016f, nguyen2010estimating}.
\end{lem}
\begin{proof}
For $\mu \ll \pi$, we have:
\begin{align*}
    &\chi^2(\mu \| \pi)  =\sup_h\left\{\int h d \mu-\int\left(\frac{h^2}{4} + h\right) d \pi\right\} \\ & =\sup_h\left\{\int h \frac{d \mu}{d \pi} d \pi-\int\left(\frac{h^2}{4}+h\right) d \pi\right\} 
     =\sup_h\left\{\int\left(h \frac{d \mu}{d \pi} -\frac{h^2}{4}-h\right) d \pi\right\} \\ 
    & =-\inf _h \left\{\int\left[\frac{h^2}{4}- h \left( \frac{d \mu}{d \pi} - 1 \right) \right] d \pi\right\} \\ & =-\inf _h \left\{\int\left[\frac{h}{2}- \left( \frac{d \mu}{d \pi} -1 \right) \right]^2 d \pi\right\} + \int \left( \frac{d \mu}{d \pi} - 1 \right)^2 d \pi .
\end{align*}
Clearly, the above equation is minimized at $h^*=2(\frac{d \mu}{d \pi} - 1)$ and $\chi^2(\mu \| \pi)=\int(\frac{d \mu}{d \pi} - 1)^2 d \pi$ which is finite if and only if $\frac{d \mu}{d \pi} -1 \in L^2(\pi)$. 
Therefore, it is sufficient to consider the above maximization over $L^2(\pi)$.
\end{proof}
\begin{lem}\label{lem:k}
Under \Cref{assumption:universal} and \ref{assumption:bounded_kernel}, the mappings  $x \mapsto k(x, \cdot)$ and $x \mapsto \nabla_1 k(x, \cdot)$ are differentiable and Lipschitz:
\begin{align*}
\begin{aligned}
\left\|k(x, \cdot) - k(y,\cdot) \right\|_\calH & \leq \sqrt{K_{1 d}}\|x-y\|, \\
\left\|\nabla_1 k(x, \cdot) - \nabla_1 k(y, \cdot) \right\|_{\calH^d} & \leq \sqrt{K_{2 d}}\|x-y\| . 
\end{aligned}
\end{align*}
\end{lem}
\begin{proof}
This is Lemma 7 from \cite{glaser2021kale}. 
\end{proof}

\begin{lem}\label{lem:k_tilde}
Under \Cref{assumption:universal} and \ref{assumption:bounded_kernel}, the regularized kernel $\tilde{k}$ defined in \eqref{eq:tilde_k} satisfies the following properties:\vspace{2mm}
\begin{enumerate}[itemsep=5.0pt,topsep=0pt,leftmargin=*]
\item $\big| \tilde{k}(x, y) \big| \leq \frac{K}{\lambda} $;
\item $\partial_i \tilde{k}(x, y)  =  \big\langle (\Sigma_\pi+\lambda \Id)^{-1} \partial_i k(x, \cdot), k(y, \cdot) \big\rangle_{\calH} $;
\item $ \big\| \nabla_1 \tilde{k}(x, y) \big\|^2 = \sum_{i=1}^d \partial_i \tilde{k}(x, y)^2 \leq \frac{ K K_{1d} }{\lambda^2}$;
\item $\partial_i \partial_{i+d} \tilde{k}(x, y) = \big\langle \left(\Sigma_\pi+\lambda \Id\right)^{-1} \partial_i k(x, \cdot), \partial_i k(y, \cdot) \big\rangle_{\calH} $; 
\item $ \big\| \nabla_1 \nabla_2 \tilde{k}(x, y) \big \|_F^2 = \sum_{i=1}^d \partial_i \partial_{i+d} \tilde{k}(x, y)^2 \leq \frac{K_{1d} }{\lambda^2 }$;
\item $\partial_i \partial_j \tilde{k}(x, y) = \big \langle \left(\Sigma_\pi+\lambda \Id\right)^{-1} \partial_i \partial_j k(x, \cdot), k(y, \cdot) \big\rangle_{\calH} $;
\item $\big\| \bH_1 \tilde{k}(x, y) \big\|_F^2 = \sum_{i,j = 1}^d \partial_i \partial_{j} \tilde{k}(x, y)^2 \leq \frac{ K K_{2d} }{\lambda^2}$; 
\item $\big\|\nabla_1 \tilde{k} \left(x, x^{\prime}\right) - \nabla_1 \tilde{k} \left(y, y^{\prime}\right) \big\| \leq \frac{\sqrt{K K_{2d}} }{\lambda}  \left( \left\| x - y \right\| + \left\| x^\prime - y^\prime \right\| \right)$.
\end{enumerate}
\end{lem}
\begin{proof}
Notice that
\begin{align*}
    \tilde{k}(x, y) =  \left\langle\left(\Sigma_\pi+\lambda \Id\right)^{-1} k(x, \cdot), k(y, \cdot) \right\rangle_{\calH} \leq \frac{1}{\lambda} \left\| k(x, \cdot) \right\|_\calH \left\| k(y, \cdot) \right\|_\calH \leq \frac{K}{\lambda},
\end{align*}
so the first bullet point is proved.
Before we prove the second bullet point, we first prove the differentiability of $x \mapsto \tilde{k}(x,y)$.
For $i \in \{1, \cdots, d\}$, consider $h \in \R$ and denote $\Delta_i \in \R^d$ as a vector of all $0$ except the value at $i$ being equal to $h$. Then, for any $y \in \R^d$,
\begin{align*}
    &\lim_{h \to 0}  \frac{ \tilde{k}(x + \Delta_i, y) - \tilde{k}(x, y)}{h} = \lim_{h \to 0} \frac{\left\langle (\Sigma_\pi + \lambda \Id)^{-1} \left( k(x + \Delta_i, \cdot) - k(x, \cdot)\right), k(y, \cdot) \right\rangle_\calH}{h} \\
    &\leq \lim_{h \to 0} \frac{\sqrt{K} }{\lambda}  \frac{\left\|k(x, \cdot) - k(x + \Delta_i,\cdot) \right\|_\calH}{h} 
    \leq \lim_{h \to 0} \frac{\sqrt{K} }{\lambda} \sqrt{K_{1d}} \frac{\| \Delta_i \| }{h} 
    = \frac{\sqrt{K} }{\lambda} \sqrt{K_{1d}}.
\end{align*}
So $x \mapsto \tilde{k}(x,y)$ is differentiable for any $y \in \R^d$.
Since the kernel $k$ is differentiable per \Cref{assumption:bounded_kernel}, for any 
$f \in \calH$, $\partial_{x_i} f(x) = \big \langle \partial_i k(x, \cdot) , f \big\rangle_\calH$. Hence,
\begin{align*}
    \partial_i \tilde{k}(x,y) = \partial_{x_i} \left\langle  k(x, \cdot), \left(\Sigma_\pi+\lambda \Id\right)^{-1} k(y, \cdot) \right\rangle_{\calH} = \left\langle  \partial_i k(x, \cdot), \left(\Sigma_\pi+\lambda \Id\right)^{-1} k(y, \cdot) \right\rangle_{\calH}.
\end{align*}
So the second bullet point is proved.

Next, notice that
\begin{align*}
    \left\| \nabla_1 \tilde{k}(x, y) \right\|^2 &= \sum_{i=1}^d \left\langle \partial_i k(x, \cdot), \left(\Sigma_\pi+\lambda \Id\right)^{-1} k(y, \cdot) \right\rangle_{\calH}^2 
    \leq \frac{1}{\lambda^2} \sum_{i=1}^d \left\| \partial_i k(x, \cdot) \right\|_\calH^2 \left\|  k(y, \cdot) \right\|_\calH^2 \\
    &\leq \frac{1}{\lambda^2} \left\| \nabla_1 k(x, \cdot) \right\|_{\calH^d}^2 \left\|  k(y, \cdot) \right\|_\calH^2 
    \leq \frac{K K_{1d}}{\lambda^2}.
\end{align*}
So the third bullet point is proved. 
Similar arguments above lead to bullet points $4$ to $7$. 

\noindent
Finally, to prove bullet point $8$, notice that
\begin{align*}
    &\left\|\nabla_1 \tilde{k} (x, x^{\prime}) - \nabla_1 \tilde{k} (y, x^{\prime} ) \right\|^2 = \sum_{i=1}^d \left \langle \left( \partial_i k(x, \cdot) - \partial_i k(y, \cdot) \right), \left( \Sigma_\pi+\lambda \Id \right)^{-1} k(x^\prime,\cdot) \right\rangle_{\calH}^2 \\
    &\leq \frac{1}{\lambda^2} \sum_{j=1}^d \left\| \partial_i k(x, \cdot) - \partial_i k(y, \cdot)  \right\|_\calH^2 \left\| k(x^\prime,\cdot)  \right\|_\calH^2 
    \leq \frac{1}{\lambda^2} \left\|\nabla_1 k(x, \cdot) - \nabla_1 k(y, \cdot)\right\|_{\calH^d}^2 \left\| k(x^\prime,\cdot)  \right\|_\calH^2 \\
    &\leq \frac{ K K_{2d} }{\lambda^2} \left\| x - y \right\|^2 ,
\end{align*}
where the last inequality uses \Cref{lem:k}.
Therefore,
\begin{align*}
    \left\|\nabla_1 \tilde{k} (x, x^{\prime}) - \nabla_1 \tilde{k} (y, y^{\prime}) \right\| &\leq \left\|\nabla_1 \tilde{k} (x, x^{\prime}) - \nabla_1 \tilde{k} (y, x^{\prime}) \right\| + \left\|\nabla_1 \tilde{k} (y, x^{\prime}) - \nabla_1 \tilde{k} (y, y^{\prime}) \right\| \\
    &\leq \frac{\sqrt{K K_{2d}} }{\lambda}  \left( \| x - y \| + \| x^\prime - y^\prime \| \right) .
\end{align*}
So the bullet point $8$ is proved.
\end{proof}
\begin{lem}\label{lem:lipschitz_witness}
    For any two distributions $\mu_0, \mu_1 \in \mathcal{P}_2(\R^d) $, with associated $\dmmd$ witness functions $h_{\mu_0, \pi}^\ast, h_{\mu_1, \pi}^\ast$ defined in \Cref{prop:drmmd_representation_no_ratio}, we have
\begin{align*}
    \left\|h_{\mu_1, \pi}^\ast - h_{\mu_0, \pi}^\ast \right\|_\calH \leq \frac{2 \sqrt{K_{1d} }}{\lambda} W_2\left(\mu_0, \mu_1\right) ,\quad \text{and}\quad \left\|h_{\mu_0, \pi}^\ast \right\|_\calH \leq \frac{2 \sqrt{K}}{\lambda} .
\end{align*}
Both the witness function $h_{\mu, \pi}^\ast$ and its gradient $\nabla h_{\mu, \pi}^\ast$ are Lipschitz continuous, i.e.,
\begin{align*}
    \Big| h_{\mu_1, \pi}^\ast(x) - h_{\mu_0, \pi}^\ast(y) \Big| &\leq L  \Big( W_2(\mu_0, \mu_1) + \left\| x - y \right\| \Big); \\
    \Big\| \nabla h_{\mu_1, \pi}^\ast(x) - \nabla h_{\mu_0, \pi}^\ast(y) \Big\| &\leq L \Big( W_2(\mu_0, \mu_1) + \left\| x - y \right\| \Big),
\end{align*}
where the constant $L= \frac{ 1}{\lambda} \max \left\{ 2 \sqrt{K K_{1 d}}, 2 \sqrt{K K_{2 d}}, 2 K_{1d} \right\}$. 
\end{lem}
\begin{proof}
Let $\gamma \in \Gamma(\mu_0, \mu_1)$ be the optimal coupling between $\mu_0$ and $\mu_1$. Then
\begin{align*}
\begin{aligned}
    &\left\| h_{\mu_1, \pi}^\ast - h_{\mu_0, \pi}^\ast \right\|_\calH^2 = 4 \left\| \Big( \Sigma_\pi + \lambda \Id \Big)^{-1} \left(\int k(x, \cdot) d (\mu_1 - \mu_0) \right)\right\|^2_\calH \\
    &\leq \frac{4}{\lambda^2} \left\| \int k(x, \cdot) - k(y, \cdot) d \gamma(x,y) \right\|^2_\calH 
    \leq \frac{4 K_{1d} }{\lambda^2} \int \left\|x - y\right\|^2 d \gamma(x,y) 
    = \frac{4 K_{1d} }{\lambda^2} W_2^2(\mu_0, \mu_1),
\end{aligned}
\end{align*}
where the first inequality holds because $\Sigma_\pi$ is a positive and self-adjoint operator, and the second inequality uses \Cref{lem:k}. Also note that
\begin{align*}
    \begin{aligned}
        \left\|h_{\mu_0, \pi}^\ast \right\|_\calH = \left\| 2 \Big( \Sigma_\pi + \lambda \Id \Big)^{-1} \left( \int k(x, \cdot) d \mu_0 - \int k(x, \cdot) d \pi \right) \right\|_\calH \leq \frac{2 \sqrt{K}}{\lambda}.
    \end{aligned}
\end{align*}
So the first part has been proved. Furthermore, note that
\begin{align*}
    &\Big| h_{\mu_1, \pi}^\ast(x) - h_{\mu_0, \pi}^\ast(y) \Big| \leq \Big| h_{\mu_1, \pi}^\ast(x) - h_{\mu_0, \pi}^\ast(x) \Big| + \Big| h_{\mu_0, \pi}^\ast(x) - h_{\mu_0, \pi}^\ast(y) \Big| \\ 
    &\leq \left\| k(x, \cdot)\right\|_\calH \left\| h_{\mu_1, \pi}^\ast - h_{\mu_0, \pi}^\ast \right\|_\calH + \left\|k(x,\cdot) - k(y,\cdot) \right\|_\calH \left\|h_{\mu_0, \pi}^\ast \right\|_\calH \\
    &\leq \frac{2\sqrt{K K_{1d} }}{\lambda} W_2(\mu_0,\mu_1) + \frac{2 \sqrt{K K_{1d} }}{\lambda} \left\| x - y \right\| \\
    &\leq \frac{2\sqrt{K K_{1d} }}{\lambda} \Big( W_2(\mu_0, \mu_1) + \left\| x - y \right\| \Big) 
    \leq L \Big( W_2(\mu_0, \mu_1) + \left\| x - y \right\| \Big)
\end{align*}
and
\begin{align*}
    &\quad \Big\| \nabla h_{\mu_1, \pi}^\ast(x) - \nabla h_{\mu_0, \pi}^\ast(y) \Big\| \leq \Big\| \nabla h_{\mu_1, \pi}^\ast(x) - \nabla h_{\mu_0, \pi}^\ast(x) \Big\| + \Big\| \nabla h_{\mu_0, \pi}^\ast(x) - \nabla h_{\mu_0, \pi}^\ast(y) \Big\| \\ 
    &\leq \left\| \nabla_1 k(x, \cdot)\right\|_{\calH^d} \left\| h_{\mu_1, \pi}^\ast - h_{\mu_0, \pi}^\ast \right\|_\calH + \left\|\nabla_1 k(x,\cdot) - \nabla_1 k(y,\cdot) \right\|_{ \calH^d } \left\|h_{\mu_0, \pi}^\ast \right\|_\calH \\
    &\leq \frac{2 K_{1d} }{\lambda} W_2(\mu_0,\mu_1) + \frac{2 \sqrt{K K_{2d}} }{\lambda} \left\| x - y \right\| 
    \leq L \Big( W_2(\mu_0, \mu_1) + \left\| x - y \right\| \Big)
\end{align*}
and the result follows.
\end{proof}
\begin{lem}\label{lem:h_difference}
Given two probability measures $\mu \ll \pi$ that are both absolutely continuous with respect to Lebesgue measure. If $\frac{ d\mu }{d \pi} - 1 \in \operatorname{Ran} ( \mathcal{T}_\pi^{r} )$ with $r > 0$, i.e., there exists $q \in L^2(\pi)$ such that $\frac{d\mu}{d\pi} - 1 = \calT_\pi^{r} q$, then
\begin{align*}
    \left\| h  - 2 \left( \frac{d \mu}{d \pi} -1 \right) \right\|_{L^2(\pi)} \leq 2 \lambda^{r} \left\|q \right\|_{L^2(\pi)},
\end{align*}
where $h = 2(\Sigma_\pi + \lambda \Id )^{-1}(m_\pi - m_{\mu})$ is defined in \Cref{prop:drmmd_representation_no_ratio}.
\end{lem}
\begin{proof}
Given the assumption that $\frac{ d\mu }{d \pi} - 1 \in \operatorname{Ran} \left( \mathcal{T}_\pi ^{r} \right)$ with $r > 0$, there exists $q \in L^2(\pi)$ such that $\frac{d\mu}{d\pi} - 1 = \calT_\pi^{r} q$ and $\langle \frac{ d\mu }{d \pi} - 1, e_i \rangle_{L^2(\pi)} = \varrho_i^{r} \PSi{ q, e_i}{L^2(\pi)}$. 
Since $\mu \ll \pi$, the Mercer decomposition of $k$ in \eqref{eq:mercer} holds for any $x$ in the support of $\mu$ and in the support of $\pi$,
\begin{align}\label{appeq:m_mu_m_pi}
    m_{\mu} - m_\pi = \int k(x, \cdot) d (\mu - \pi)(x) = \sum_{i \geq 1} \varrho_i \left( \int e_i(x) d( \mu - \pi)(x) \right) e_i = \varrho_i  \PSi{\frac{d \mu}{d \pi}- 1, e_i}{L^2(\pi)} e_i.
\end{align}
So we have, 
\begin{align*}
    &\left\| h - 2 \left( \frac{d \mu}{d \pi} -1 \right) \right\|_{L^2(\pi)} = 2 \left\| \left(\Sigma_{\pi}+\lambda \Id\right)^{-1}\left(m_{\mu} - m_{\pi}\right) - \left( \frac{d \mu}{d \pi} -1 \right) \right\|_{L^2(\pi)} \nonumber \\
    &= 2 \left\| \sum_{i \geq 1} \frac{\varrho_i}{\varrho_i + \lambda} \PSi{\frac{d \mu}{d \pi}- 1, e_i}{L^2(\pi)} e_i - \sum_{i \geq 1} \PSi{\frac{d \mu}{d \pi}- 1, e_i}{L^2(\pi)} e_i \right\|_{L^2(\pi)} \nonumber \\
    &= 2 \left\| \sum_{i \geq 1} \frac{\lambda}{\varrho_i + \lambda} \PSi{\frac{d \mu}{d \pi}- 1, e_i}{L^2(\pi)} e_i \right\|_{L^2(\pi)} 
    = 2 \left\| \sum_{i \geq 1} \frac{\lambda \varrho_i^{r}}{\varrho_i + \lambda} \PSi{q, e_i}{L^2(\pi)} e_i \right\|_{L^2(\pi)} \nonumber \\
    &\leq 2 \lambda^{r} \left\| q \right\|_{L^2(\pi)}, 
\end{align*}
where the last inequality is obtained by using 
\begin{align*}
    \frac{\lambda \varrho_i^{r}}{\varrho_i + \lambda} = \left( \frac{\varrho_i}{\varrho_i + \lambda} \right)^{r} \left( \frac{\lambda}{\varrho_i + \lambda} \right)^{1 - r} \lambda^{r} \leq \lambda^{r}.
\end{align*}
Hence the proof.
\end{proof}

\begin{lem}\label{lem:wass_gradient_hess_chard}
Let $\rho \in \calP_2\left(\mathbb{R}^d\right)$ and $\phi \in C_c^\infty\left(\mathbb{R}^d\right)$. Consider the path $(\rho_s)_{0 \leq s \leq 1}$ from $\rho$ to $(\Id + \nabla \phi)_{\#} \rho$ given by $\rho_s = (\Id + s \nabla \phi)_{\#} \rho$. 
Define $\varphi_s: \R^d \to \R^d, x \mapsto x + s \nabla \phi(x)$. Let 
$\tilde{k}$ be the regularized kernel defined in \eqref{eq:tilde_k} along with its associated RKHS $\tilde{\calH}$. 
The mapping $s \mapsto \dmmd(\rho_s \| \pi)$ is continuous and differentiable, and its first-order time derivative is given by
\begin{align}\label{appeq:chard_wass_gradient}
     &\quad \frac{d}{d s} \dmmd(\rho_s||\pi) \nonumber \\
     &= 2(1+\lambda) \int \nabla \phi(x)^\top \left( \int  \nabla_1 \tilde{k}(\varphi_s(x ), \varphi_s(z) ) d \rho (z) - \int \nabla_1 \tilde{k}(\varphi_s(x ), z) d \pi(z) \right) d \rho(x).
\end{align}
Moreover, 
\begin{align}\label{appeq:chard_wass_gradient_0}
     &\frac{d}{d s} \Big|_{s=0} \dmmd(\rho_s||\pi) \nonumber\\
     &= 2(1+\lambda) \int \nabla \phi(x)^\top \left( \int  \nabla_1 \tilde{k}(x, z) d \rho(z) - \int \nabla_1 \tilde{k}(x, z) d \pi(z) \right) d \rho(x) .
\end{align}
Additionally, the mapping $s \to \frac{d}{ds} \dmmd(\rho_s \| \pi)$ is continuous and differentiable, and the second-order time derivative of $\dmmd(\rho_s \| \pi)$ is given by
\begin{small}
\begin{align}\label{appeq:chard_wass_hessian}
    &\quad \frac{d^2}{d s^2} \dmmd(\rho_s||\pi) = 2 (1+\lambda) \iint \nabla \phi(x)^\top \nabla_1 \nabla_2 \tilde{k}(\varphi_s(x ), \varphi_s(z)) \nabla \phi(z) d\rho(x) d\rho(z) \\
    & + 2 (1+\lambda) \int \nabla \phi( x )^\top \left( \int\bH_1 \tilde{k} \left(\varphi_s(x ), \varphi_s(z)  \right) d \rho(z) -  \int \bH_1 \tilde{k} \left(\varphi_s(x ), z\right) d \pi(z) \right) \nabla \phi( x) d\rho(x) \nonumber  ,
\end{align}
\end{small}
with
\begin{align}\label{appeq:chard_wass_hessian_0}
    &\quad \frac{d^2}{d s^2}\Big|_{s=0} \dmmd(\rho_s||\pi) =  2 (1+\lambda) \iint \nabla \phi(x)^\top \nabla_1 \nabla_2 \tilde{k}(x, z) \nabla \phi(z) d\rho(x) d\rho(z) \nonumber \\
    &+ 2 (1+\lambda) \int \nabla \phi(x)^\top \left( \int\bH_1 \tilde{k} \left(x, z\right) d \rho(z) -  \int \bH_1 \tilde{k} \left(x, z\right) d \pi(z) \right) \nabla \phi(x) d\rho(x) .
\end{align}

\end{lem}
\begin{proof}
Recall that $\dmmd$ is, up to a multiplicative factor of $(1+\lambda) $, $\mmd^2$ of the regularized kernel $\tilde{k}$ defined in \eqref{eq:tilde_k} and the associated RKHS $\tilde{\calH}$.
From \Cref{lem:k_tilde}, we know that assumptions (A) and (B) of \cite{arbel2019maximum} are satisfied for the regularized kernel $\tilde{k}$, so using Lemma 22 and Lemma 23 from \cite{arbel2019maximum}, \eqref{appeq:chard_wass_gradient} and \eqref{appeq:chard_wass_hessian} are proved.
Then \eqref{appeq:chard_wass_gradient_0} and \eqref{appeq:chard_wass_hessian_0} are subsequently proved by taking $s=0$.
\end{proof}

\begin{lem}\label{lem:hessian_terms}
For $\mu_0 \ll \pi$, define $h = 2(\Sigma_\pi + \lambda \Id )^{-1}(m_\pi - m_{\mu_0})$ and $\varphi_t: \R^d \to \R^d, x \mapsto x - t (1+\lambda) \nabla h(x)$. Suppose $\pi \propto \exp(-V)$, $\bH V \preceq \beta \Id$, and the step size $\gamma$ satisfies 
\begin{align*}
    2 (1+\lambda) \gamma \sqrt{ \chi^2(\mu_0 \| \pi) \frac{K_{2d}}{\lambda} } \leq \frac{\zeta - 1}{\zeta}
\end{align*}
for some constant $1 < \zeta < 2$. Then, for  $t \in [0, \gamma]$, the following inequalities hold,
\begin{align}
    \int \nabla h( x )^\top \bH V( \varphi_t(x) ) \nabla h(x) d \mu_0(x) & \leq 4 \beta \chi^2(\mu_0 \| \pi) \frac{K_{1d}}{\lambda}; \label{eq:hessian_terms_1} \\
    \int \left\| \bH h(x) \left(\Id -t (1 + \lambda) \bH h(x) \right)^{-1} \right\|_F^2 d\mu_0(x) &\leq 4 \zeta^2 \chi^2(\mu_0 \| \pi) \frac{K_{2d}}{\lambda} . \label{eq:hessian_terms_2} 
\end{align}
\end{lem}
\begin{proof}
First, recall \eqref{appeq:m_mu_m_pi},
\begin{align}
    \left\| \nabla h (x) \right\|^2 &= 4 \left\| \nabla (\Sigma_\pi + \lambda \Id)^{-1} (m_\pi - m_{\mu_0})(x) \right\|^2 \nonumber \\
    &= 4 \left\| \nabla \left( \sum_{i \geq 1} \frac{\varrho_i}{\varrho_i + \lambda} \PSi{\frac{ d \mu_0}{ d \pi} - 1, e_i}{L^2(\pi)} e_i(x) \right) \right\|^2. \label{appeq:nabla_h_1}
\end{align}
For any $j \in \{1, \cdots, d\}$, consider 
\begin{align*}
    g_{M_0}(x) &:= \sum_{i \geq M_0} \left| \varrho_i^{1/2} \left\langle \frac{d \mu_0}{ d \pi} - 1, e_i \right\rangle_{L^2(\pi)} \partial_j e_i(x)\right| \\
    &\leq \left( \sum_{i \geq M_0} \varrho_i \partial_j e_i(x)^2 \right)^{\frac{1}{2}} \left( \sum_{i \geq M_0} \PSi{ \frac{d \mu_0}{ d \pi} - 1 , e_i}{L^2(\pi)} ^2 \right)^{\frac{1}{2}} \\
    &\leq \left( \sum_{i \geq M_0} \varrho_i \partial_j e_i(x)^2 \right)^{\frac{1}{2}} \left\| \frac{d \mu_0}{ d \pi} - 1 \right\|_{L^2(\pi)}.
\end{align*}
Since
\begin{align}\label{appeq:sum_kernel_gradient}
    \sum_{i \geq 1} \sum_{j=1}^d \varrho_i \left( \partial_j e_i(x) \right)^2 &= \sum_{i \geq 1} \sum_{j=1}^d \varrho_i \PSi{\partial_j k(x, \cdot), e_i}{\calH}^2 = \sum_{i \geq 1} \sum_{j=1}^d \PSi{\partial_j k(x, \cdot), \sqrt{\varrho_i} e_i}{\calH}^2 \nonumber \\
    &= \sum_{j=1}^d \left\| \partial_j k(x, \cdot) \right\|_{\calH}^2 = \left\| \nabla_1 k(x,\cdot) \right\|_\calH^2 \leq K_{1d} ,
\end{align}
so $ \sum_{i \geq M_0} \varrho_i (\partial_j e_i(x))^2$ converges uniformly to $0$, and hence $g_{M_0}(x)$ also converges uniformly to $0$. Therefore, we are allowed to interchange the derivative and the infinite sum \citep{rudin1976principles} in \eqref{appeq:nabla_h_1} to achieve,
\begin{align}
    \left\| \nabla h (x) \right\|^2 &= 4 \left\| \sum_{i \geq 1} \frac{\varrho_i}{\varrho_i + \lambda} \PSi{\frac{d \mu_0}{ d \pi} - 1, e_i}{L^2(\pi)} \nabla e_i(x) \right\|^2 \nonumber \\
    &\leq 4 \left( \sum_{i \geq 1} \frac{\varrho_i}{(\varrho_i + \lambda)^2 } \PSi{\frac{ d \mu_0}{d \pi} - 1, e_i}{L^2(\pi)}^2 \right) \left( \sum_{i\geq 1 } \varrho_i \left\| \nabla e_i(x) \right\|^2 \right) \nonumber \\
    &\leq \frac{4}{\lambda} \left( \sum_{i \geq 1} \PSi{\frac{ d \mu_0}{d \pi} - 1, e_i}{L^2(\pi)}^2 \right) \left( \sum_{i\geq 1 } \varrho_i \left\| \nabla e_i(x) \right\|^2 \right) 
    \leq 4 \chi^2(\mu_0 \| \pi) \frac{K_{1d}}{\lambda} . \label{eq:nabla_h}
\end{align}
The first inequality follows from Cauchy Schwartz, the penultimate inequality follows by noting that $\frac{\varrho_i}{(\varrho_i + \lambda)^2} \leq \frac{1}{\lambda}$, and the last inequality follows from \eqref{appeq:sum_kernel_gradient}.
Given $\bH V \preceq \beta \Id$, \eqref{eq:hessian_terms_1} is proved by the following,
\begin{align*}
    \int \nabla h( x )^\top \bH V( \varphi_t(x) ) \nabla h(x) d \mu_0(x) \leq \beta \int \left\| \nabla h (x) \right\|^2 d\mu_0(x) \leq 4 \beta \chi^2(\mu_0 \| \pi) \frac{K_{1d}}{\lambda} .
\end{align*}
We now turn to proving the second statement. Similarly to \eqref{eq:nabla_h}, we have
\begin{align}\label{appeq:hess_h}
    \left\| \bH h(x) \right\|_F^2 \leq 4 \chi^2(\mu_0 \| \pi) \frac{K_{2d}}{\lambda} .
\end{align}
Using $2 (1+\lambda) \gamma \sqrt{ \chi^2(\mu_0 \| \pi) \frac{K_{2d}}{\lambda} } \leq \frac{\zeta - 1}{\zeta}$ for some constant $1 < \zeta < 2$, the inverse of $ \Id - t (1+\lambda) \bH h(x)$ can be represented by the Neumann series, and hence
\begin{align}\label{appeq:hessian_hs}
    \left\|(\Id - t (1+\lambda) \bH h(x))^{-1} \right\|_F &\leq \sum_{m \geq 0} \|t (1+\lambda) \bH h(x)\|_F^m \leq \sum_{m \geq 0} \left( \gamma(1+\lambda) 2 \sqrt{ \chi^2(\mu_0 \| \pi) \frac{K_{2d}}{\lambda} } \right)^m \nonumber \\
    &\leq \sum_{m \geq 0} \left( \frac{\zeta - 1}{\zeta} \right)^m = \zeta.
\end{align}
Therefore, \eqref{eq:hessian_terms_2} is proved by combining \eqref{appeq:hess_h} and \eqref{appeq:hessian_hs},
\begin{align*}
    &\int \left\| \bH h(x) \left(\Id -t (1+\lambda) \bH h(x) \right)^{-1} \right\|_F^2 d\mu_0(x) \\
    &\leq \int \left\| \bH h(x) \right\|_F^2 \left\|(\Id - t (1+\lambda) \bH h(x))^{-1} \right\|_F^2 d\mu_0(x) 
    \leq 4 \zeta^2 \chi^2(\mu_0 \| \pi) \frac{K_{2d}}{\lambda} 
\end{align*}
and the result follows.
\end{proof}
\vspace{-15pt}

\begin{lem}\label{lem:hoeffding}
Let $\calH$ be a separable Hilbert space, $\xi_1, \ldots, \xi_n: \Omega \rightarrow \calH$ are $n$ identical independent $\calH$-valued random variables satisfying $\left\| \xi_i \right\|_\calH \leq B$. Then
\begin{align*}
    \E \left\|\frac{1}{n} \sum_{i=1}^n \xi_i - \E[\xi_1] \right\|_\calH \leq \frac{\sqrt{2 \pi} B}{2 \sqrt{n}} ,\quad\text{and} \quad \quad \E \left\|\frac{1}{n} \sum_{i=1}^n \xi_i - \E[\xi_1] \right\|_\calH^2 \leq \frac{B^2}{n} .
\end{align*}
\end{lem}
\begin{proof}
We know from Corollary 6.15 of \cite{steinwart2008support} that
\begin{align*}
    \Pb \left( \left\|\frac{1}{n} \sum_{i=1}^n \xi_i - \mathbb{E} \left[\xi_1 \right] \right\|_\calH \geq t\right) \leq 2 \exp \left(-2 n t^2 /B^2 \right) .
\end{align*}
Then, denote $R := \| \frac{1}{n} \sum_{i=1}^n \xi_i - \E[\xi_1] \|_\calH$, we know that
\begin{align*}
    \E[R] &= \int_0^\infty \Pb(R \geq t) dt  \leq 2 \int_0^\infty \exp(- 2 n t^2 / B^2 ) dt = \frac{\sqrt{2 \pi} B}{2 \sqrt{n}}.
\end{align*}
The other part is proved similarly.
\end{proof}
\vspace{-15pt}

\definecolor{c_1}{rgb}{1.0, 0.7, 0.8}
\definecolor{c_2}{rgb}{0.17254901960784313, 0.6274509803921569, 0.17254901960784313}
\definecolor{c_3}{rgb}{0.5803921568627451, 0.403921568627451, 0.7411764705882353}
\definecolor{c_4}{rgb}{0.5490196078431373, 0.33725490196078434, 0.29411764705882354}
\definecolor{c_5}{rgb}{0.0, 0.6, 0.8}
\definecolor{c_6}{rgb}{1.0, 0.6, 0.2}
\definecolor{c_7}{rgb}{1.0, 0.0, 0.0}
\definecolor{c_8}{rgb}{0.4, 0.0, 0.6}

\begin{lem}[Wasserstein Hessian of $\calF_{\chi^2}$]\label{lem:wass_chi2}
Let $\rho \in \mathcal{P}_2(\R^d) $ and $\phi \in C_c^\infty (\R^d) $. Consider the path $(\rho_s)_{0\leq s \leq 1}$ from $\rho$ to $(\Id + \nabla \phi)_{\#} \rho$ given by $\rho_s = (\Id + s \nabla \phi)_{\#} \rho$. 
Define $\varphi_s: \R^d \to \R^d, x \mapsto x + s \nabla \phi(x)$ and $\omega_s: \R^d \to \R^d, x \mapsto \left[\nabla \phi \circ \varphi_s^{-1}\right](x)$.
For $\pi(x) \propto \exp(-V(x))$, the second order derivative of $s \to \chi^2(\rho_s \| \pi)$ is given by
\begin{align}
    & \quad \frac{d^2}{d s^2} \chi^2(\rho_s \| \pi) = \int \frac{\rho_s(x)}{\pi(x)} \left( \nabla V(x)^\top \omega_s(x) - \nabla \cdot \omega_s(x) \right)^2 d\rho_s(x)\nonumber\\
    &\quad+  \int \frac{\rho_s(x)}{\pi(x)} \omega_s(x)^\top \bH V(x) \omega_s(x) d\rho_s(x) +  \int \frac{\rho_s(x)}{\pi(x)} \left\| \nabla \omega_s(x) \right\|_F^2 d \rho_s(x) . \label{appeq:wess_chi2_1}
\end{align}
Equivalently, the second order derivative of $s \to \chi^2(\rho_s \| \pi)$ can also be written as
\begin{align}\label{appeq:wess_chi2_2}
    \frac{d^2}{d s^2} \chi^2(\rho_s \| \pi) &= 2 \int\left( \nabla \cdot \left( \omega_s(x) \rho_s(x) \right) \frac{1}{\pi(x)} \right)^2 \pi(x) d x \nonumber\\
    &\qquad\qquad+ 2 \int \omega_s(x)^\top \bH \left(\frac{\rho_s(x) }{\pi(x)}\right) \omega_s(x) \rho_s(x) dx.
\end{align}
\end{lem}
\begin{proof}
\eqref{appeq:wess_chi2_1} is provided in the Example 15.9 of \cite{villani2009optimal} by taking $m=2$. Now, we are going to prove \eqref{appeq:wess_chi2_2}. For ease of notation in the following derivations, we are going to drop the function input $x$ in $\rho_s, \pi$, and $\omega_s$.
We introduce colors to picture grouping of terms that will carry over during chains of calculation.

In order to prove \eqref{appeq:wess_chi2_2}, we need to expand the terms in \eqref{appeq:wess_chi2_1} accordingly. We denote the three terms in the RHS of \eqref{appeq:wess_chi2_1} as $(A)$, $\textcolor{c_1}{(B)}$ and $\textcolor{c_2}{(C)}$. 
Consider 
 \begin{align*}
    (A) &= \int \frac{\rho_s^2}{\pi} \left( \nabla V^\top \omega_s  - \nabla \cdot \omega_s \right)^2 dx \\
    &= \int \frac{\rho_s^2}{\pi} \left( -\frac{1}{\pi} \nabla \pi^\top \omega_s  - \nabla \cdot \omega_s \right)^2 dx \\ 
    &= \textcolor{c_3}{ \underbrace{\int \frac{\rho_s^2}{\pi^3} (\nabla \pi^\top \omega_s)^2 dx }_{(A_1)} } + 
    \textcolor{c_4}{\underbrace{2 \int \frac{\rho_s^2}{\pi^2}  (\nabla \pi^\top \omega_s) \nabla \cdot \omega_s dx }_{(A_2)}} + 
    \textcolor{c_5}{\underbrace{\int \frac{\rho_s^2}{\pi} (\nabla \cdot \omega_s)^2 dx}_{(A_3)} }.
\end{align*}
Then we are going to use integration by parts for $\textcolor{c_4}{ (A_2)}$ and $\textcolor{c_5}{(A_3)}$. 
\begin{align*}
    \textcolor{c_4}{ (A_2) } &= 2 \int \frac{\rho_s^2}{\pi^2}  (\nabla \pi^\top \omega_s) \nabla \cdot \omega_s dx = -2 \int \omega_s^\top \nabla \left( \frac{\rho_s^2}{\pi^2} (\nabla \pi^\top \omega_s) \right)  dx \\
    &= -4 \int (\omega_s^\top \nabla \rho_s) (\nabla \pi^\top \omega_s) \frac{\rho_s}{\pi^2} dx + 4 \int (\omega_s^\top \nabla \pi)^2 \frac{\rho_s^2}{\pi^3} dx \\
    &- 2 \int (\omega_s^\top \bH \pi \omega_s) \frac{\rho_s^2}{\pi^2} dx - 2 \int ( \omega_s^\top \nabla \omega_s \nabla \pi) \frac{\rho_s^2}{\pi^2} dx, \\
    \textcolor{c_5}{ (A_3) } &= \int \frac{\rho_s^2}{\pi} (\nabla \cdot \omega_s)^2 dx = \int \nabla \cdot (\omega_s) \nabla \cdot \omega_s \frac{\rho_s^2}{\pi} dx =  - \int \omega_s^\top \nabla \left( \nabla \cdot \omega_s \frac{\rho_s^2}{\pi} \right) dx \\
    &= \textcolor{c_6}{ \underbrace{- \int \omega_s^\top \nabla (\nabla \cdot \omega_s) \frac{\rho_s^2}{\pi}  dx}_{(A_{31})} }
    \textcolor{c_7}{ \underbrace{ - \int \nabla \cdot \omega_s ( \omega_s^\top \nabla \rho_s) \frac{2 \rho_s}{\pi}  dx }_{(A_{32})} }
    \textcolor{c_8}{ \underbrace{ + \int \nabla \cdot \omega_s (\omega_s^\top \nabla \pi) \frac{\rho_s^2 }{\pi^2}  dx }_{(A_{33})} } .
\end{align*}
Furthermore, we use integration by parts for $\textcolor{c_7}{ (A_{32}) }$ and $\textcolor{c_8}{ (A_{33}) }$, we have 
\begin{align*}
    \textcolor{c_7}{ (A_{32}) } &= -2 \int \nabla \cdot \omega_s ( \omega_s^\top \nabla \rho_s) \frac{ \rho_s}{\pi}  dx = 2 \int ( \omega_s^\top  \nabla \left( 
    ( \omega_s^\top \nabla \rho_s) \frac{\rho_s}{\pi} \right) ) dx \\
    &= 2 \int (\omega_s^\top \nabla \omega_s \nabla \rho_s) \frac{\rho_s}{\pi} dx + 2 \int (\omega_s^\top \bH \rho_s \omega_s ) \frac{\rho_s}{\pi} dx + 2 \int ( \omega_s^\top \nabla \rho_s)^2 \frac{1}{\pi} dx \\
    &- 2 \int ( \omega_s^\top \nabla \rho_s) (\omega_s^\top \nabla \pi) \frac{\rho_s}{\pi^2} dx, \\
    \textcolor{c_8}{  (A_{33}) } &= \int \nabla \cdot \omega_s (\omega_s^\top \nabla \pi) \frac{\rho_s^2 }{\pi^2}  dx = -\int ( \omega_s^\top  \nabla \left( 
    (\omega_s^\top \nabla \pi) \frac{\rho_s^2}{\pi^2} \right)  dx \\
    &= -\int ( \omega_s^\top \nabla \omega_s \nabla \pi) \frac{\rho_s^2}{\pi^2} dx - \int (\omega_s^\top \bH \pi \omega_s) \frac{\rho_s^2}{\pi^2} dx -2 \int ( \omega_s^\top \nabla \rho_s) (\omega_s^\top \nabla \pi) \frac{\rho_s}{\pi^2} dx \\
    &\qquad\qquad+2 \int (\omega_s^\top \nabla \pi)^2 \frac{\rho_s^2}{\pi^3} dx  .
\end{align*}
Having completed $(A)$, now we turn to $\textcolor{c_1}{ (B) }$.
\begin{align*}
    \textcolor{c_1}{ (B) } &= \int \frac{\rho_s^2}{\pi} \omega_s^\top \bH V \omega_s dx = - \int \omega_s^\top \bH \pi \omega_s \frac{\rho_s^2}{\pi^2} dx + \int (\nabla \pi^\top \omega_s)^2 \frac{\rho_s^2}{\pi^3} dx .
\end{align*}
So, combining $(A)$, $\textcolor{c_1}{ (B) }$ and $\textcolor{c_2}{ (C) }$, we have 
\begin{align*}
    &\quad \frac{d^2}{ds^2} \chi^2(\rho_s\|\pi) = (A) + \textcolor{c_1}{ (B) } + \textcolor{c_2}{ (C) } = \textcolor{c_3}{ (A_1)} + \textcolor{c_4}{ (A_2) } + \textcolor{c_6}{ (A_{31})} + \textcolor{c_7}{ (A_{32}) } + \textcolor{c_8}{ (A_{33}) } + \textcolor{c_1}{ (B) } + \textcolor{c_2}{ (C) } \\
    & = \textcolor{c_3}{ \int \frac{\rho_s^2}{\pi^3} (\nabla \pi^\top \omega_s)^2 dx } \\
    & \qquad\textcolor{c_4}{ -4 \int ( \omega_s^\top \nabla \rho_s) (\nabla \pi^\top \omega_s) \frac{\rho_s}{\pi^2} dx + 4 \int (\omega_s^\top \nabla \pi)^2 \frac{\rho_s^2}{\pi^3} dx - 2 \int (\omega_s^\top \bH \pi \omega_s) \frac{\rho_s^2}{\pi^2} dx } \\
    &\qquad\qquad\qquad\textcolor{c_4}{ - 2 \int ( \omega_s^\top \nabla \omega_s \nabla \pi) \frac{\rho_s^2}{\pi^2} dx } + \textcolor{c_6}{ (A_{31})} + \textcolor{c_7}{ (A_{32}) } + \textcolor{c_8}{ (A_{33}) }  \\
    &\qquad\qquad\qquad\qquad\textcolor{c_1}{ - \int ( \omega_s^\top \bH \pi  \omega_s) \frac{\rho_s^2}{\pi^2} dx + \int (\nabla \pi^\top \omega_s)^2 \frac{\rho_s^2}{\pi^3} dx } + \textcolor{c_2}{ \int \frac{\rho_s^2}{\pi} \left\| \nabla \omega_s \right\|_F^2 dx } .
\end{align*}
Since $\omega_s = \nabla \phi \circ \varphi_s^{-1}$ is a mapping from $\R^d$ to $\R^d$, denote $\omega_{s, i} := \left[\nabla \phi \circ \varphi_s^{-1}\right]_i$ which is a mapping from $\R^d$ to $\R$.
Notice that $\omega_{s, i} \partial_{j} \omega_{s, i} $ vanishes at boundary because $\phi \in C_c^\infty$. Hence,
\begin{align*}
    0 &= \sum_{i,j} \int \partial_j \left( \omega_{s, i} \partial_{j} \omega_{s, i} \frac{\rho_s^2}{\pi} \right) dx \\
    &= \sum_{i,j} \int  \partial_{j} \omega_{s, i}  \partial_{j} \omega_{s, i} \frac{\rho_s^2}{\pi} dx + 
    \sum_{i,j} \int \omega_{s, i} \partial_{jj} \omega_{s, i} \frac{\rho_s^2}{\pi} dx \nonumber\\
    &\qquad\qquad+\sum_{i,j} \int \omega_{s, i} \partial_{j} \omega_{s, i} \left( \frac{2 \rho_s \partial_j \rho_s}{\pi} - \frac{\rho_s^2 \partial_j \pi}{\pi^2} \right) dx \\ 
    &= \textcolor{c_2}{ \int \frac{\rho_s^2}{\pi} \left\| \nabla \omega_s \right\|_F^2 dx} + \int \frac{\rho_s^2}{\pi} ( \omega_s^\top  \nabla(\nabla \cdot \omega_s)) dx + \int (\omega_s^\top \nabla \omega_s \nabla \rho_s) \frac{2 \rho_s}{\pi} dx \\
    &\qquad\qquad- \int ( \omega_s^\top \nabla \omega_s \nabla \pi) \frac{\rho_s^2}{\pi^2} dx .
\end{align*}
Therefore, by replacing $\textcolor{c_2}{ \int \frac{\rho_s^2}{\pi} \left\| \nabla \omega_s \right\|_F^2 dx }$, and noticing that $ \textcolor{c_6}{ -\int \frac{\rho_s^2}{\pi} ( \omega_s^\top  \nabla(\nabla \cdot \omega_s)) dx }$ is exactly $\textcolor{c_6}{ (A_{31}) }$, we  have the following:
\begin{align*}
    &\quad \frac{d^2}{ds^2} \chi^2(\rho_s\|\pi) = \textcolor{c_3}{ \int \frac{\rho_s^2}{\pi^3} (\nabla \pi^\top \omega_s)^2 dx}
    - \textcolor{c_4}{ 4 \int ( \omega_s^\top \nabla \rho_s) (\nabla \pi^\top \omega_s) \frac{\rho_s}{\pi^2} dx + 4 \int (\omega_s^\top \nabla \pi)^2 \frac{\rho_s^2}{\pi^3} dx } \\
    & \qquad\textcolor{c_4}{ - 2 \int (\omega_s^\top \bH \pi \omega_s) \frac{\rho_s^2}{\pi^2} dx - 2 \int ( \omega_s^\top \nabla \omega_s \nabla \pi) \frac{\rho_s^2}{\pi^2} dx} + \textcolor{c_6}{ (A_{31})} + \textcolor{c_7}{ (A_{32}) } + \textcolor{c_8}{ (A_{33}) } \\
    &\qquad\qquad \textcolor{c_1}{- \int ( \omega_s^\top \bH \pi  \omega_s) \frac{\rho_s^2}{\pi^2} dx + \int (\nabla \pi^\top \omega_s)^2 \frac{\rho_s^2}{\pi^3} dx } \\
    & \qquad\qquad\qquad-\int \frac{\rho_s^2}{\pi} ( \omega_s^\top  \nabla(\nabla \cdot \omega_s)) dx - \int (\omega_s^\top \nabla \omega_s \nabla \rho_s) \frac{2 \rho_s}{\pi} dx + \int ( \omega_s^\top \nabla \omega_s \nabla \pi) \frac{\rho_s^2}{\pi^2} dx .
\end{align*}
Next, we combine the terms and obtain
\begin{align*}
    &= 2( \textcolor{c_6}{ (A_{31})} + \textcolor{c_7}{ (A_{32}) } + \textcolor{c_8}{ (A_{33}) } ) + 6 \int \frac{\rho_s^2}{\pi^3} (\nabla \pi^\top \omega_s)^2 dx - 4 \int ( \omega_s^\top \nabla \rho_s) (\nabla \pi^\top \omega_s) \frac{\rho_s}{\pi^2} dx \\
    &\,\,- 3 \int (\omega_s^\top \bH \pi \omega_s) \frac{\rho_s^2}{\pi^2} dx - \int ( \omega_s^\top \nabla \omega_s \nabla \pi) \frac{\rho_s^2}{\pi^2} dx - 2 \int (\omega_s^\top \nabla \omega_s \nabla \rho_s) \frac{ \rho_s}{\pi} dx - \textcolor{c_7}{ ( A_{32}) } - \textcolor{c_8}{ (A_{33}) }.
\end{align*}
Recall that $\textcolor{c_6}{ (A_{31})} + \textcolor{c_7}{ (A_{32}) } + \textcolor{c_8}{ (A_{33}) } = \textcolor{c_5}{ (A_3)}$, and replacing $\textcolor{c_7}{ (A_{32}) } $ and $\textcolor{c_8}{ (A_{33}) }$ with their expressions, we have
\begin{footnotesize}
\begin{align*}
    &=2 \textcolor{c_5}{ \int \frac{\rho_s^2}{\pi} (\nabla \cdot \omega_s)^2 dx } + 6 \int \frac{\rho_s^2}{\pi^3} (\nabla \pi^\top \omega_s)^2 dx - 4 \int ( \omega_s^\top \nabla \rho_s) (\nabla \pi^\top \omega_s) \frac{\rho_s}{\pi^2} dx \\
    &-3 \int (\omega_s^\top \bH \pi \omega_s) \frac{\rho_s^2}{\pi^2} dx 
    - \int ( \omega_s^\top \nabla \omega_s \nabla \pi) \frac{\rho_s^2}{\pi^2} dx - 2 \int (\omega_s^\top \nabla \omega_s \nabla \rho_s) \frac{\rho_s}{\pi} dx \\
    &- \textcolor{c_7}{ \left( 2 \int (\omega_s^\top \nabla \omega_s \nabla \rho_s) \frac{\rho_s}{\pi} dx + 2 \int (\omega_s^\top \bH \rho_s \omega_s ) \frac{\rho_s}{\pi} dx + 2 \int ( \omega_s^\top \nabla \rho_s)^2 \frac{1}{\pi} dx -  2 \int ( \omega_s^\top \nabla \rho_s) (\omega_s^\top \nabla \pi) \frac{\rho_s}{\pi^2} dx \right) } \\
    &- \textcolor{c_8}{ \left( -  \int ( \omega_s^\top \nabla \omega_s \nabla \pi) \frac{\rho_s^2}{\pi^2} dx -  \int (\omega_s^\top \bH \pi \omega_s) \frac{\rho_s^2}{\pi^2} dx  -  2 \int ( \omega_s^\top \nabla \rho_s) (\omega_s^\top \nabla \pi) \frac{\rho_s}{\pi^2} dx + 2 \int (\omega_s^\top \nabla \pi)^2 \frac{\rho_s^2}{\pi^3} dx  \right)} \\
    &= \underbrace{ 2\int (\omega_s^\top \bH \rho_s \omega_s ) \frac{\rho_s}{\pi} dx - 2\int (\omega_s^\top \bH \pi \omega_s) \frac{\rho_s^2}{\pi^2} dx - 4 \int ( \omega_s^\top \nabla \rho_s) (\omega_s^\top \nabla \pi) \frac{\rho_s}{\pi^2} dx + 4 \int \frac{\rho_s^2}{\pi^3} ( \nabla \pi^\top \omega_s )^2 dx }_{(M)} \\
    &\qquad+ \underbrace{ 
    2\int \frac{\rho_s^2}{\pi} (\nabla \cdot \omega_s)^2 dx - 2 \int ( \omega_s^\top \nabla \rho_s)^2 \frac{1}{\pi} dx -4 \int (\omega_s^\top \nabla \omega_s \nabla \rho_s) \frac{\rho_s}{\pi} dx }_{(N_1)} \\
    &\qquad\qquad \underbrace{ -4 \int (\omega_s^\top \bH \rho_s \omega_s ) \frac{\rho_s}{\pi} dx + 4 \int ( \omega_s^\top \nabla \rho_s) (\omega_s^\top \nabla \pi) \frac{\rho_s}{\pi^2}  dx}_{(N_2)} .
\end{align*}
\end{footnotesize}
Now we analyze $(M)$ and $(N_1) + (N_2)$ separately. Notice that
\begin{align*}
    (M) = 2\int \left( \omega_s^\top \bH \left( \frac{\rho_s}{\pi} \right) \omega_s \right) \rho_s dx ,
\end{align*} 
and 
\begin{align*}
    &\quad (N_1) + (N_2)
    = 2\int \frac{\rho_s^2}{\pi} (\nabla \cdot \omega_s)^2 dx + 2 \int ( \omega_s^\top \nabla \rho_s)^2 \frac{1}{\pi} dx     - 4 \int (\omega_s^\top \nabla \omega_s \nabla \rho_s) \frac{\rho_s}{\pi} dx \\
    &\qquad\qquad- 4\int (\omega_s^\top \bH \rho_s \omega_s ) \frac{\rho_s}{\pi} dx- 4 \int ( \omega_s^\top \nabla \rho_s)^2 \frac{1}{\pi} dx + 4 \int ( \omega_s^\top \nabla \rho_s) (\omega_s^\top \nabla \pi) \frac{\rho_s}{\pi^2} dx \\
    &= 2\int \frac{\rho_s^2}{\pi} (\nabla \cdot \omega_s)^2 dx + 2 \int ( \omega_s^\top \nabla \rho_s)^2 \frac{1}{\pi} dx -4 \int ( \omega_s^\top  \nabla ( \omega_s^\top \nabla \rho_s) ) \frac{\rho_s}{\pi} dx\\
    &\qquad\qquad- 4 \int ( \omega_s^\top \nabla \rho_s) ( \omega_s^\top  \nabla \frac{\rho_s}{\pi}) dx \\
    &= 2\int \frac{\rho_s^2}{\pi} (\nabla \cdot \omega_s)^2 dx + 2 \int ( \omega_s^\top \nabla \rho_s)^2 \frac{1}{\pi} dx + 4\int \nabla \cdot \omega_s ( \omega_s^\top \nabla \rho_s)  \frac{\rho_s}{\pi} dx \\
    &= 2 \int\left(\nabla \cdot( \omega_s \rho_s ) \frac{1}{\pi} \right)^2 \pi d x .
\end{align*}
Since $\frac{d^2}{ds^2} \chi^2(\rho_s\|\pi) = (M) + (N_1) + (N_2)$, \Cref{lem:wass_chi2} is proved.
\end{proof}

\setcounter{equation}{0}
\appsection{An Illustrative Example for Explicit Forms of $\mathcal{I}_t$, $\mathcal{J}_t$, $\|q_t\|_{L^2(\pi)}$\label{sec:example}}
Consider an illustrative example where we simulate the DrMMD gradient flow when the target is a one-dimensional Gaussian target distribution $\pi = \calN(0, \bar{\sigma}^2)$ and the initialization is also a one-dimensional Gaussian $\mu_0 = \calN(0, \frac{1}{2}\bar{\sigma}^2)$. 
We take a Gaussian kernel $k(x, y) = \exp(-\frac{1}{2}(x-y)^2)$ whose eigenvalues and eigenfunctions in its Mercer decomposition have the following closed form expressions~\citep[Proposition 1]{shi2009data},
\begin{align}\label{eq:eigvalues}
    \varrho_i = \sqrt{\frac{\frac{1}{2\bar{\sigma}^2}}{\frac{1}{2\bar{\sigma}^2} + c + 0.5 }}\left(\frac{\beta^2-1}{\beta^2+1} \right)^i , \quad e_i(x) = \sqrt{\frac{\beta}{i! 2^i}} \exp(-c x^2) \mathrm{H}_i \left( \sqrt{\frac{1}{2 \bar{\sigma}^2}} \beta x \right),
\end{align}
where $\beta = (1+ 4\bar{\sigma}^2)^{1 / 4}$, $c = \frac{\beta^2 - 1}{4\bar{\sigma}^2}$ and $\mathrm{H}_i$ is the $i$-th Hermite polynomial function. 
We pick $\bar{\sigma}^2 > 2$ so $\beta^2 > 3$. 
The Gaussian kernel is continuous, bounded and $c_0$-universal as required in Assumption 1. It also possesses bounded first- and second-order derivatives, thereby satisfying Assumption 2. 

Consider the DrMMD gradient flow $(\mu_t)_{t\geq 0}$ defined in \eqref{eq:cont_eqn} along with its particle update scheme $dx_t = -(1+\lambda)\nabla h_{\mu_t,\pi}(x_t)dt$, where  $h_{\mu_t,\pi}$ is the witness function defined in \eqref{eq:witness_alternative} and $\lambda$ is a positive regularization parameter. 
Denote $m_t, \sigma_t^2$ as the mean and covariance of $\mu_t$, respectively, then we have the following update scheme for $m_t, \sigma_t^2$ proved in \Cref{lem:bures_update}: 
\begin{align*}
    \dd  m_t &= -(1+\lambda)\E_{x_t\sim\mu_t} [\nabla h_{\mu_t,\pi}(x_t)] \; \dd  t ,\\
    \dd (\sigma_t^2) &= -2(1+\lambda) \E_{x_t\sim\mu_t}[\nabla h_{\mu_t,\pi}(x_t) \cdot x_t] \;\dd t + 2(1+\lambda)\E_{x_t\sim\mu_t}[\nabla h_{\mu_t,\pi}(x_t)] \cdot m_t \;\dd  t,
\end{align*}
where the expectations are taken over $x_t\sim\mu_t$. While the resulting distribution is not necessarily Gaussian, we may follow the existing analysis of Stein variational gradient descent~\citep{liu2024towards} and Langevin Monte Carlo dynamics~\citep{lambert2022variational} and approximate $x_t\sim\mu_t$ with a Gaussian random variable $y_t\sim \nu_t=\calN(m_t, \sigma_t^2)$;  this yields the update scheme:
\begin{align}\label{eq:bures_update}
\begin{aligned}
    \dd  m_t &= (1+\lambda)\E_{y_t\sim\nu_t}[\nabla h_{\nu_t,\pi}(y_t)] \; \dd  t, \\
    \dd (\sigma_t^2) &= -2(1+\lambda) \E_{y_t\sim\nu_t}[\nabla h_{\nu_t,\pi}(y_t) \cdot y_t] \;\dd t + 2(1+\lambda)\E_{y_t\sim\nu_t} [\nabla h_{\nu_t,\pi}(y_t)] \cdot m_t \;\dd  t, 
\end{aligned}
\end{align}
which gives an evolution of Gaussian distributions $\nu_t$. 
From \eqref{eq:witness_alternative}, the witness function $h_{\nu_t, \pi}$ admits a decomposition with eigenvalues $\varrho_i$ and eigenfunctions $e_i$: $h_{\nu_t, \pi}(y) = \sum_{i \geq 1} \frac{\varrho_i}{\varrho_i + \lambda} \langle \frac{d \nu_t}{d\pi} - 1, e_i \rangle_{L^2(\pi)} e_i(y)$.
Therefore, the velocity field $\frac{d}{dy} h_{\nu_t, \pi}(y)$ can be written as,
\begin{align*}
    \frac{d}{dy} h_{\nu_t, \pi}(y) = \sum_{i \geq 1} \frac{\varrho_i}{\varrho_i + \lambda} \left\langle \frac{d \nu_t}{d\pi} - 1, e_i \right\rangle_{L^2(\pi)} \frac{d}{dy} e_i(y).
\end{align*}
Notice that if $m_t=0$, then for odd $i$, $\langle \frac{d \nu_t}{d\pi} - 1, e_i\rangle_{L^2(\pi)}= \int e_i d\nu_t - \int e_i d\pi=0$ because $e_i$ is an odd function. When $i$ is even, $\E_{\nu_t}[\frac{d}{dy} e_i(y)] = 0$ because $y\mapsto \frac{d}{dy} e_i(y)$ is an odd function. As a result, if $m_t=0$, then $\E[\nabla h_{\nu_t,\pi}(y_t)]=0$ and hence $\frac{dm_t}{dt} = 0$ from the update scheme in \eqref{eq:bures_update}. Therefore, as long as we initialize the DrMMD gradient flow with $\nu_0 = \calN(0, \frac{1}{2}\bar{\sigma}^2)$, a zero mean Gaussian distribution, the entire trajectory will remain a zero mean Gaussian distribution $\calN(0,\sigma^2_t)$.

Next, observe that for a zero-mean Gaussian trajectory, if the initial variance satisfies $\sigma_0^2 < \bar{\sigma}^2$, it is natural to expect that the variance increases monotonically toward the target variance $\bar{\sigma}^2$ as the DrMMD flow evolves, i.e. $\sigma_0^2 < \sigma_t^2 \leq \bar{\sigma}^2$ for all $t$. 
In \Cref{lem:sign}, we provide a rigorous proof of this claim in the cases $\lambda = 0$ and $\lambda = \infty$, by showing that $(1+\lambda) \E_{y_t\sim\nu_t}[\nabla h_{\nu_t,\pi}(y_t) \cdot y_t] < 0$, which implies that the variance update in \eqref{eq:bures_update} is monotone increasing. 
The cases $\lambda=0$ and $\lambda=\infty$ correspond respectively to the $\chi^2$ flow and the MMD flow regimes. 
For general $\lambda > 0$, however, we are unable to establish a rigorous proof. 
Our argument in \Cref{lem:sign} relies heavily on Mehler’s formula~\citep[Proposition 2.2]{liang2022mehler}, which requires exponential decay of the spectrum $(\varrho_i)_{i\geq 1}$. This condition is not satisfied for DrMMD, whose spectrum is modified to $(\frac{\varrho_i}{\varrho_i + \lambda})_{i\geq 1}$. Nevertheless, we conjecture that the monotonicity property continues to hold for all $\lambda > 0$.

\vspace{1em}
\noindent
\textbf{Checking assumptions in Theorem 4.1 and Theorem 5.1:}
Now we check the assumptions from Theorem 4.1 and Theorem 5.1. The target $\pi$ is a Gaussian distribution which automatically satisfies a Poincaré inequality, and its potential $V$ is a quadratic function, hence satisfies $\mathbf{H} V \leq \beta \Id$. $\nu_t, \pi$ are Gaussians, hence absolutely continuous with respect to Lebesgue on $\R$. 
And most importantly, we have $\frac{d\nu_t}{d\pi}-1 \in \operatorname{Ran}(\calT_\pi^{0.25})$, i.e., there exists $q_t \in L^2(\pi) $ such that $\frac{d\nu_t}{d\pi} - 1 = \calT_\pi^{0.25} q_t$. To see why, we first need to upper bound $\langle \frac{d \nu_t}{d\pi} - 1, e_i\rangle_{L^2(\pi)}^2$. From \citet[Corollary 2]{belafhal2020note}, we have the following closed-form expressions for $\langle \frac{d \nu_t}{d\pi} - 1, e_i\rangle_{L^2(\pi)}$: 
\begin{align*}
    \sqrt{\frac{\beta}{i! 2^i}} \left( \sqrt{\frac{1}{1 + 2 \sigma_t^2 c} } \left( 1 - \frac{\frac{1}{2 \bar{\sigma}^2} \beta^2 }{c + \frac{1}{2 \sigma_t^2 }} \right)^{\frac{i}{2} } - \sqrt{\frac{1}{1 + 2 \bar{\sigma}^2 c} } \left( 1 - \frac{\frac{1}{2 \bar{\sigma}^2} \beta^2 }{c + \frac{1}{2\bar{\sigma}^2}} \right)^{\frac{i}{2} }
    \right) \mathrm{H}_{i}(0),
\end{align*}
when $i$ is even and $0$ otherwise. 
Therefore, we have
\begin{align*}
    \left\langle \frac{d \nu_t}{d\pi} - 1, e_i\right\rangle_{L^2(\pi)}^2 \leq \frac{1}{1 + \bar{\sigma}^2 c} \cdot \frac{\beta}{i! 2^i} |\mathrm{H}_i(0)|^2 \cdot 2 \left(\left|\frac{\frac{1}{2 \bar{\sigma}^2} \beta^2 }{c + \frac{1}{2 \sigma_t^2 }} - 1 \right|^i + \left| \frac{\frac{1}{2 \bar{\sigma}^2} \beta^2 }{c + \frac{1}{2\bar{\sigma}^2}} - 1 \right|^i \right) .
\end{align*}
By the monotonicity of the variance, $\sigma_0^2=\frac{1}{2}\bar{\sigma}^2 \leq \sigma_t^2 < \bar{\sigma}^2$ and $\beta^2 > 3$, we have
\begin{align*}
    0 < \frac{2\beta^2}{\beta^2 + 3} - 1 = \frac{\frac{1}{2 \bar{\sigma}^2} \beta^2 }{c + \frac{1}{\bar{\sigma}^2}} - 1 < \frac{\frac{1}{2 \bar{\sigma}^2} \beta^2 }{c + \frac{1}{2\sigma_t^2}} - 1 <  \frac{\frac{1}{2 \bar{\sigma}^2} \beta^2 }{c + \frac{1}{2\bar{\sigma}^2}} - 1  = \frac{2\beta^2}{\beta^2 + 1} -1 . 
\end{align*}
So we have the following upper bound 
\begin{align*}
    \left\langle \frac{d \nu_t}{d\pi} - 1, e_i\right\rangle_{L^2(\pi)}^2 \leq \frac{4}{1 + \bar{\sigma}^2 c} \cdot \frac{\beta}{i! 2^i} |\mathrm{H}_i(0)|^2 \cdot \left(\frac{2\beta^2}{\beta^2 + 1} - 1 \right)^i . 
\end{align*}
Now we are ready to study the $L^2(\pi)$-norm of $q_t$. Recall the formulas of $\varrho_i$ in \eqref{eq:eigvalues}, we have
\begin{align*}
    \|q_t\|_{L^2(\pi)} &= \sum_{i=1}^{\infty} \frac{\left\langle \frac{d \nu_t}{d\pi} - 1, e_i\right\rangle_{L^2(\pi)}^2}{\varrho_i^{0.5}} \leq \frac{4\beta}{1 + \bar{\sigma}^2 c} \sum_{i=1}^{\infty} \frac{1}{i! 2^i} \left|\mathrm{H}_{i}(0)\right|^2 \left(\frac{\beta^2-1}{\beta^2 + 1} \right)^i \left( \left(\frac{\beta^2+1}{\beta^2-1} \right)^{0.5} \right)^i \\
    &= \frac{4\beta}{1 + \bar{\sigma}^2 c} \sum_{i=1}^{\infty} \frac{1}{i! 2^i} \left|\mathrm{H}_{i}(0)\right|^2 \left( \left(\frac{\beta^2-1}{\beta^2+1} \right)^{0.5} \right)^i = \frac{4\beta}{1 + \bar{\sigma}^2 c} \sqrt{\frac{\beta^2+1}{2}} . 
\end{align*}
The last equality holds by Mehler’s formula~\citep[Proposition 2.2]{liang2022mehler}. 
A quick sanity check for the above derivations is to take $r=0.5$ and see that $\sum_{i=1}^{\infty} \varrho_i^{-1} \langle \frac{d \nu_t}{d\pi} - 1, e_i\rangle_{L^2(\pi)}^2$ is divergent. This indicates that $\frac{d \nu_t}{d\pi} - 1 \notin \calH$, which verifies the fact that the Gaussian RKHS does not contain constant functions~\citep[Corollary 4.44]{steinwart2008support}. 
Finally, $\mathcal{I}_t$ and $\mathcal{J}_t$ admit the following explicit formulas as well as uniform upper bounds 
\begin{align*}
    \mathcal{J}_t &= \left\|\nabla(\log \pi)^{\top} \nabla\left(\frac{d \nu_t}{d \pi}\right)\right\|^2_{L^2(\pi)} = \frac{\frac{1}{\sigma_t^2} - \frac{1}{\bar{\sigma}^2}}{\sqrt{2\pi}\bar{\sigma}^3 \sigma_t^2} \cdot \frac{\Gamma(7/2)}{\left(\frac{1}{\sigma_t^2} - \frac{1}{2\bar{\sigma}^2} \right)^{7/2}} \leq \frac{4 \Gamma(7/2)}{\sqrt{2\pi} \left(\frac{3}{2} \right)^{7/2}} \\
    \mathcal{I}_t &= \left\|\Delta\left(\frac{d \nu_t}{d \pi}\right)\right\|_{L^2(\pi)}= \left( \frac{\bar{\sigma}^2}{\sigma_t^2} \frac{\left(\frac{1}{\sigma_t^2} - \frac{1}{\bar{\sigma}^2}\right)^2}{\sqrt{2\pi \bar{\sigma}^2}} \right) \cdot \left(\frac{\Gamma(1/2)}{\left(\frac{1}{\sigma_t^2} - \frac{1}{2\bar{\sigma}^2} \right)^{1/2}} + \frac{2 \Gamma(3/2) \left(\frac{1}{\sigma_t^2} - \frac{1}{\bar{\sigma}^2}\right)}{\left(\frac{1}{\sigma_t^2} - \frac{1}{2\bar{\sigma}^2} \right)^{3/2}} \right.\\
    &\qquad + \left.  \frac{\Gamma(5/2)\left(\frac{1}{\sigma_t^2} - \frac{1}{\bar{\sigma}^2}\right)^2}{\left(\frac{1}{\sigma_t^2} - \frac{1}{2\bar{\sigma}^2} \right)^{5/2}} \right) \leq \frac{2}{\sqrt{2\pi} \bar{\sigma}^4} \left( \frac{\Gamma(1/2)}{\left(\frac{3}{2} \right)^{5/2}} + \frac{\Gamma(3/2)}{\left(\frac{3}{2} \right)^{3/2}} + \frac{\Gamma(5/2)}{\left(\frac{3}{2} \right)^{5/2}} \right). 
\end{align*}

\vspace{1em}
\noindent
\textbf{Empirical verification:}
While the above derivations demonstrate that a (Gaussian projected) DrMMD gradient flow satisfies all the required assumptions, it is instructive to demonstrate that a particle implementation in the discrete time setting, without explicit Gaussian projection,
shows the behavior consistent with the theory. We therefore simulate our finite particle DrMMD gradient descent $\hat{\mu}_n$ with a step size $\gamma=0.01$ and particle number $N=10,000$ and empirically inspect its convergence properties. We take $\bar{\sigma}^2 = 6$ so $\beta^2 = 5$. 
We estimate the density of the DrMMD descent $\mu_n$ from the particles with a kernel density estimator using a Gaussian kernel with lengthscale $0.1$~\citep{epanechnikov1969non}.  
Based on the estimated densities $\mu_n$, we compute the following two quantities: $\mathcal{I}_n = \| \nabla V^\top \nabla ( \frac{d \mu_n}{d\pi} ) \|_{L^2(\pi)}$ and $\mathcal{J}_n = \| \Delta ( \frac{d \mu_n}{d\pi} ) \|_{L^2(\pi)}$. 
Their evolution along the DrMMD gradient descent is shown in \Cref{fig:one_d_gaussian}.
We observe that both quantities decrease over time as desired, which is a consequence of increasing smoothness of the density ratio $\frac{d \mu_n}{d\pi}$ as $\mu_n$ converges to $\pi$.
We also report the evolution of the KL divergence $\widehat{\mathrm{KL}}(\hat{\mu}_n\|\hat{\pi})$ along the flow, estimated from particles. 

\begin{figure}
    \centering
    \includegraphics[width=0.5\linewidth]{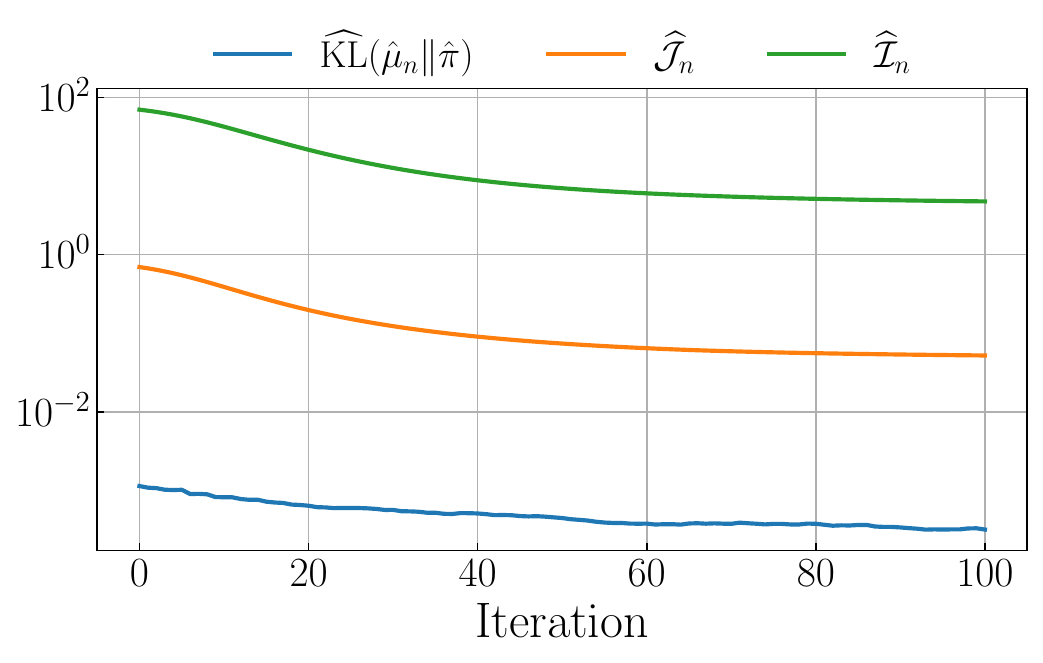}
    \caption{Evolution of $\widehat{\mathrm{KL}}(\hat{\mu}_n\|\hat{\pi})$, $\widehat{\mathcal{I}}_n$ and $\widehat{\mathcal{J}}_n$ along DrMMD particle descent, where all three terms are estimated with samples. }
    \label{fig:one_d_gaussian}
\end{figure}

\begin{lem}\label{lem:bures_update}
Given the DrMMD gradient flow update scheme $dx_t = -(1+\lambda)\nabla h_{\mu_t,\pi}(x_t)dt$, its mean $m_t=\E[x_t]$ and variance $\sigma_t^2=\E[x_t^2]-\E[x_t]^2$ update scheme can be expressed as
\begin{align*}
    \dd  m_t &= -(1+\lambda)\E_{x_t\sim\mu_t} [\nabla h_{\mu_t,\pi}(x_t)] \; \dd  t ,\\
    \dd (\sigma_t^2) &= -2(1+\lambda) \E_{x_t\sim\mu_t}[\nabla h_{\mu_t,\pi}(x_t) \cdot x_t] \;\dd t + 2(1+\lambda)\E_{x_t\sim\mu_t}[\nabla h_{\mu_t,\pi}(x_t)] \cdot m_t \;\dd  t,
\end{align*}
\end{lem}
\begin{proof}
    For the mean update, $\frac{\dd }{\dd t} m_t = \E[\frac{\dd }{\dd t}x_t]=-(1+\lambda)\E_{x_t\sim\mu_t} [\nabla h_{\mu_t,\pi}(x_t)]$. For the variance update, 
    \begin{align*}
        \frac{\dd}{\dd t} \sigma_t^2 &= \frac{\dd}{\dd t} \left(\E[x_t^2]-\E[x_t]^2\right) = 2 \E[x_t \cdot \frac{\dd}{\dd t} x_t] - 2 \E[x_t] \cdot \frac{\dd}{\dd t} \E[x_t] \\
        &= -2(1+\lambda) \E_{x_t\sim\mu_t}[\nabla h_{\mu_t,\pi}(x_t) \cdot x_t] + 2(1+\lambda)\E_{x_t\sim\mu_t}[\nabla h_{\mu_t,\pi}(x_t)] \cdot m_t. 
    \end{align*}
    Hence the result.
\end{proof}
\begin{lem}\label{lem:sign}
Let $\nu_t$ be a zero mean normal distribution $\calN(0,\sigma_t^2)$ and $0 < \sigma_t^2 < \bar{\sigma}^2$.
For the eigenvalues $(\varrho_i)_{i\geq1}$ and eigenfunctions $(e_i)_{i\geq1}$ defined in \eqref{eq:eigvalues}, we have 
\begin{align*}
    \sum_{i \geq 1} \left\langle \frac{d \nu_t}{d\pi} - 1, e_i \right\rangle_{L^2(\pi)} \E_{x\sim\nu_t}\left[\frac{d}{dx} e_i(x) x\right] < 0 , \\
    \sum_{i \geq 1} \varrho_i \left\langle \frac{d \nu_t}{d\pi} - 1, e_i \right \rangle_{L^2(\pi)} \E_{x\sim\nu_t}\left[\frac{d}{dx} e_i(x) x\right] < 0 .
\end{align*}
which correspond to the variance update in \eqref{eq:bures_update} when $\lambda=0$ and $\lambda=\infty$, respectively. 
\end{lem}
\begin{proof}
For the inner product term, we know from Corollary 2 of \citet{belafhal2020note} that
\begin{align*}
    &\left\langle \frac{d \nu_t}{d\pi} - 1, e_i \right\rangle_{L^2(\pi)} \\
    &= \mathrm{H}_{i}(0) \sqrt{\frac{\beta}{i! 2^i}} \left( \sqrt{\frac{1}{1 + 2 \sigma_t^2 c} } \left( 1 - \frac{\frac{1}{2 \bar{\sigma}^2} \beta^2 }{c + \frac{1}{2 \sigma_t^2 }} \right)^{\frac{i}{2} } - \sqrt{\frac{1}{1 + 2 \bar{\sigma}^2 c} } \left( 1 - \frac{\frac{1}{2 \bar{\sigma}^2} \beta^2 }{c + \frac{1}{2\bar{\sigma}^2}} \right)^{\frac{i}{2}} \right) .
\end{align*}
For the expectation term, notice that
\begin{align*}
    &\quad \E_{x\sim\nu_t}\left[\frac{d}{dx} e_i(x) x\right] = \sqrt{\frac{\beta}{i! 2^i}} \frac{1}{\sqrt{2\pi}\sigma_t} \int x \exp\left( -\frac{1}{2\sigma_t^2}x^2 \right) \frac{d}{dx} e_i(x) \; \dd x \nonumber \\
    &\stackrel{(\ast)}{=} - \sqrt{\frac{\beta}{i! 2^i}} \frac{1}{\sqrt{2\pi}\sigma_t} \int  \exp\left( -\frac{1}{2\sigma_t^2}x^2 \right)\left( 1 - \frac{x^2}{\sigma_t^2} \right) e_i(x) \; \dd x \nonumber \\
    &= - \sqrt{\frac{\beta}{i! 2^i}} \frac{1}{\sqrt{2\pi}\sigma_t} \int  \exp\left( -\frac{1}{2\sigma_t^2}x^2 \right) \exp(-c x^2) \mathrm{H}_i \left( \sqrt{\frac{1}{2 \bar{\sigma}^2}} \beta x \right) \; \dd x \\
    &\qquad + \sqrt{\frac{\beta}{i! 2^i}} \frac{1}{\sqrt{2\pi}\sigma_t^3} \int  \exp\left( -\frac{1}{2\sigma_t^2}x^2 \right) \exp(-c x^2) \mathrm{H}_i \left( \sqrt{\frac{1}{2 \bar{\sigma}^2}} \beta x \right) x^2 \; \dd x \\
    &=: E_{i,1} + E_{i,2}. 
\end{align*}
Here in $(\ast)$, we use integration by parts in which the boundary term vanishes, because $\lim_{x\to\infty} \exp(-x^2)\mathrm{H}_i(x)=0$. 
Next, notice that the second term $E_{i,2}$ equals precisely the derivative of $E_{i,1}$ with respect to $c$, rescaled by $\frac{1}{\sigma_t^2}$. To distinguish it from the other $c$ that will show up later in $\langle \frac{d \nu_t}{d\pi} - 1, e_i \rangle_{L^2(\pi)}$, we denote it as $\mathfrak{c}$. 
The original $\sum_{i\geq1} \langle \frac{d \nu_t}{d\pi} - 1, e_i \rangle_{L^2(\pi)} \E_{x\sim\nu_t}[\frac{d}{dx} e_i(x) x]$ can be written as the sum of two components $F_1 + F_2$:
\begin{align*}
    F_1 &= \sum_{i\geq1} \mathrm{H}_{i}(0) \sqrt{\frac{\beta}{i! 2^i}} \left( \sqrt{\frac{1}{1 + 2 \sigma_t^2 c} } \left( 1 - \frac{\frac{1}{2 \bar{\sigma}^2} \beta^2 }{c + \frac{1}{2 \sigma_t^2 }} \right)^{\frac{i}{2} } - \sqrt{\frac{1}{1 + 2 \bar{\sigma}^2 c} } \left( 1 - \frac{\frac{1}{2 \bar{\sigma}^2} \beta^2 }{c + \frac{1}{2\bar{\sigma}^2}} \right)^{\frac{i}{2}} \right) E_{i,1}, \\
    F_2 &= \sum_{i\geq1} \mathrm{H}_{i}(0) \sqrt{\frac{\beta}{i! 2^i}} \left( \sqrt{\frac{1}{1 + 2 \sigma_t^2 c} } \left( 1 - \frac{\frac{1}{2 \bar{\sigma}^2} \beta^2 }{c + \frac{1}{2 \sigma_t^2 }} \right)^{\frac{i}{2} } - \sqrt{\frac{1}{1 + 2 \bar{\sigma}^2 c} } \left( 1 - \frac{\frac{1}{2 \bar{\sigma}^2} \beta^2 }{c + \frac{1}{2\bar{\sigma}^2}} \right)^{\frac{i}{2}} \right) E_{i,2} . 
\end{align*}
From Corollary 2 of \citet{belafhal2020note}, we know 
\begin{align*}
    E_{i,1} = -\sqrt{\frac{\beta}{i! 2^i}}  \sqrt{\frac{1}{1 + 2 \sigma_t^2 \mathfrak{c} } } \left( 1 - \frac{\frac{1}{2 \bar{\sigma}^2} \beta^2 }{\mathfrak{c} + \frac{1}{2 \sigma_t^2 }} \right)^{\frac{i}{2} } \mathrm{H}_{i}(0),
\end{align*}
and $E_{i,2}=\frac{1}{\sigma_t^2} \frac{d}{d \mathfrak{c}} E_{i,1}$. The relation between $E_{i,1}$ and $E_{i,2}$ would carry over to $F_1$ and $F_2$ as well, i.e., $F_2 =\frac{1}{\sigma_t^2} \frac{d}{d \mathfrak{c}} F_1$. 
Next, we compute the first half of $F_1$. 
\begin{align}
    &\quad \sum_{i\geq 1} \mathrm{H}_{i}(0)  \sqrt{\frac{\beta}{i! 2^i}} \sqrt{\frac{1}{1 + 2 \sigma_t^2 c}} \left( 1 - \frac{\frac{1}{2 \bar{\sigma}^2} \beta^2 }{c + \frac{1}{2 \sigma_t^2 }} \right)^{\frac{i}{2} } E_{i,1} \nonumber \\
    &= -\sqrt{\frac{1}{1 + 2 \sigma_t^2 c} }  \sqrt{\frac{1}{1+ 2 \sigma_t^2 \mathfrak{c}} } \sum_{i\geq 1}\frac{\beta}{i! 2^i} \left( 1 - \frac{\frac{1}{2 \bar{\sigma}^2} \beta^2 }{c + \frac{1}{2 \sigma_t^2 }} \right)^{\frac{i}{2} } \left( 1 - \frac{\frac{1}{2 \bar{\sigma}^2} \beta^2 }{\mathfrak{c} + \frac{1}{2 \sigma_t^2 }} \right)^{\frac{i}{2}} |\mathrm{H}_i(0)|^2 \nonumber \\
    &= - \beta \sqrt{\frac{1}{1 + 2 \sigma_t^2 c} } \sqrt{\frac{1}{1+ 2 \sigma_t^2 \mathfrak{c}} } \left( 1 - \left( 1 - \frac{\frac{1}{2 \bar{\sigma}^2} \beta^2 }{c + \frac{1}{2 \sigma_t^2 }} \right) \cdot \left( 1 - \frac{\frac{1}{2 \bar{\sigma}^2} \beta^2 }{\mathfrak{c} + \frac{1}{2 \sigma_t^2 }} \right) \right)^{-\frac{1}{2}} \nonumber \\
    &= -\beta \left( (1 + 2 \sigma_t^2 c)(1 + 2 \sigma_t^2 \mathfrak{c}) - (1 + 2 \sigma_t^2 c - \frac{\sigma_t^2}{\bar{\sigma}^2} \beta^2) (1 + 2 \sigma_t^2 \mathfrak{c} - \frac{\sigma_t^2}{\bar{\sigma}^2} \beta^2) \right)^{-\frac{1}{2}} \nonumber \\
    &= - \left( \frac{\sigma_t^2}{\bar{\sigma}^2} (2 + 2 \sigma_t^2 c + 2 \sigma_t^2 \mathfrak{c}) - \left( \frac{\sigma_t^2}{\bar{\sigma}^2} \beta \right)^2 \right)^{-\frac{1}{2}} .\nonumber
\end{align}
The second last equality holds by the Mehler’s formula~\citep[Proposition 2.2]{liang2022mehler}. 
Similarly, we can also compute the second half of $F_1$. 
\begin{align}
    &\quad \sum_{i\geq 1} \mathrm{H}_{i}(0)  \sqrt{\frac{\beta}{i! 2^i}} \sqrt{\frac{1}{1 + 2 \bar{\sigma}^2 c}} \left( 1 - \frac{\frac{1}{2 \bar{\sigma}^2} \beta^2 }{c + \frac{1}{2 \bar{\sigma}^2 }} \right)^{\frac{i}{2} } E_{i,1} \nonumber \\
    &= -\beta \left( (1 + 2 \bar{\sigma}^2 c)(1 + 2 \sigma_t^2 \mathfrak{c}) - (1 + 2 \bar{\sigma}^2 c - \beta^2) (1 + 2 \sigma_t^2 \mathfrak{c} - \frac{\sigma_t^2}{\bar{\sigma}^2} \beta^2) \right)^{-\frac{1}{2}} \nonumber  \\
    &= - \left( 1 + 2 \sigma_t^2 \mathfrak{c} + (1 + 2\bar{\sigma}^2 c)\frac{\sigma_t^2}{\bar{\sigma}^2} - \frac{\sigma_t^2}{\bar{\sigma}^2} \beta^2 \right)^{-\frac{1}{2}}.\nonumber
\end{align}
Combining the above two equations, we obtain the following formula of $F_1$. 
\begin{align*}
    F_1 &= - \left( \frac{\sigma_t^2}{\bar{\sigma}^2} (2 + 2 \sigma_t^2 c + 2 \sigma_t^2 \mathfrak{c}) - \left( \frac{\sigma_t^2}{\bar{\sigma}^2} \beta \right)^2 \right)^{-\frac{1}{2}} + \left( 1 + 2 \sigma_t^2 \mathfrak{c} + (1 + 2\bar{\sigma}^2 c)\frac{\sigma_t^2}{\bar{\sigma}^2} - \frac{\sigma_t^2}{\bar{\sigma}^2} \beta^2 \right)^{-\frac{1}{2}} . 
\end{align*}
Since $F_2 =\frac{1}{\sigma_t^2} \frac{d}{d \mathfrak{c}} F_1$, we have
\begin{align*}
    F_2 = \left( \frac{\sigma_t^2}{\bar{\sigma}^2} (2 + 2 \sigma_t^2 c + 2 \sigma_t^2 \mathfrak{c}) - \left( \frac{\sigma_t^2}{\bar{\sigma}^2} \beta \right)^2 \right)^{-\frac{3}{2}} \frac{\sigma_t^2}{\bar{\sigma}^2} - \left( 1 + 2 \sigma_t^2 \mathfrak{c} + (1 + 2\bar{\sigma}^2 c)\frac{\sigma_t^2}{\bar{\sigma}^2} - \frac{\sigma_t^2}{\bar{\sigma}^2} \beta^2 \right)^{-\frac{3}{2}} .
\end{align*}
Recall that the original $\sum_{i\geq1} \langle \frac{d \nu_t}{d\pi} - 1, e_i \rangle_{L^2(\pi)} \E_{x\sim\nu_t}[\frac{d}{dx} e_i(x) x]$ can be written as the sum of two components $F_1 + F_2$, and recall that $\mathfrak{c} = c$ and $\beta^2=1+4\bar{\sigma}^2 c$ by definition. We obtain
\begin{small}
\begin{align*}
    \sum_{i \geq 1} \left\langle \frac{d \nu_t}{d\pi} - 1, e_i \right\rangle_{L^2(\pi)} \E_{x\sim\nu_t}\left[\frac{d}{dx} e_i(x) x\right] = F_1 + F_2 = \left( \frac{\sigma_t^2}{\bar{\sigma}^2} \left(2 - \frac{\sigma_t^2}{\bar{\sigma}^2} \right) \right)^{-\frac{1}{2}} \cdot \left(-1 + \left(2 - \frac{\sigma_t^2}{\bar{\sigma}^2} \right)^{-1} \right) . 
\end{align*}
\end{small}
which is negative when $0 < \sigma_t < \bar{\sigma}$. So we have concluded the proof of the first claim.
Now, we are about to prove the second claim. 
Following the same derivations as above, we can write $\sum_{i \geq 1} \langle \varrho_i \frac{d \nu_t}{d\pi} - 1, e_i \rangle_{L^2(\pi)} \E_{x\sim\nu_t}[\frac{d}{dx} e_i(x) x]$ as the sum of two terms $G_1$ and $G_2$.
\begin{small}
\begin{align*}
    G_1 &= \sum_{i\geq1} \varrho_i \mathrm{H}_{i}(0) \sqrt{\frac{\beta}{i! 2^i}} \left( \sqrt{\frac{1}{1 + 2 \sigma_t^2 c} } \left( 1 - \frac{\frac{1}{2 \bar{\sigma}^2} \beta^2 }{c + \frac{1}{2 \sigma_t^2 }} \right)^{\frac{i}{2} } - \sqrt{\frac{1}{1 + 2 \bar{\sigma}^2 c} } \left( 1 - \frac{\frac{1}{2 \bar{\sigma}^2} \beta^2 }{c + \frac{1}{2\bar{\sigma}^2}} \right)^{\frac{i}{2}} \right) E_{i,1}, \\
    G_2 &= \sum_{i\geq1} \varrho_i \mathrm{H}_{i}(0) \sqrt{\frac{\beta}{i! 2^i}} \left( \sqrt{\frac{1}{1 + 2 \sigma_t^2 c} } \left( 1 - \frac{\frac{1}{2 \bar{\sigma}^2} \beta^2 }{c + \frac{1}{2 \sigma_t^2 }} \right)^{\frac{i}{2} } - \sqrt{\frac{1}{1 + 2 \bar{\sigma}^2 c} } \left( 1 - \frac{\frac{1}{2 \bar{\sigma}^2} \beta^2 }{c + \frac{1}{2\bar{\sigma}^2}} \right)^{\frac{i}{2}} \right) E_{i,2} . 
\end{align*}
\end{small}
Since $\varrho_i \propto (\frac{\beta^2-1}{\beta^2+1})^i$ defined in \eqref{eq:eigvalues} has exponential decay, hence Mehler's formula still hold. Up to some positive multiplier coefficient that do not change the sign, $G_1$ can be written as the following formula
\begin{align*}
    G_1 &= -\left( (2\bar{\sigma}^2c)^{-2}(1 + 2 \sigma_t^2 c)^2 + 2 \frac{\sigma_t^2}{\bar{\sigma}^2} - \frac{\sigma_t^4}{\bar{\sigma}^4} \right)^{-\frac{1}{2}} + \left( (2\bar{\sigma}^2 c)^{-2} (1 + 2 \bar{\sigma}^2 c)(1 + 2 \sigma_t^2 \mathfrak{c}) + 1 \right)^{-\frac{1}{2}} . 
\end{align*}
And similar to $F_1, F_2$, we have $G_2 =\frac{1}{\sigma_t^2} \frac{d}{d \mathfrak{c}} G_1$. As a result, 
\begin{align*}
    G_2 &= \left( (2\bar{\sigma}^2c)^{-2}(1 + 2 \sigma_t^2 c)^2 + 2 \frac{\sigma_t^2}{\bar{\sigma}^2} - \frac{\sigma_t^4}{\bar{\sigma}^4} \right)^{-\frac{3}{2}} \cdot \left(\frac{1+2\sigma_t^2 c}{(2\bar{\sigma}^2 c)^2} + \frac{\sigma_t^2}{\bar{\sigma}^2} \right) \\
    &- \left( (2\bar{\sigma}^2 c)^{-2} (1 + 2 \bar{\sigma}^2 c)(1 + 2 \sigma_t^2 \mathfrak{c}) + 1 \right)^{-\frac{3}{2}} \cdot \left(\frac{1+2\bar{\sigma}^2 c}{(2\bar{\sigma}^2 c)^2} +1 \right) . 
\end{align*}
Combined, we obtain
\begin{align*}
    G_1 + G_2 &= -\left( (2\bar{\sigma}^2c)^{-2}(1 + 2 \sigma_t^2 c)^2 + 2 \frac{\sigma_t^2}{\bar{\sigma}^2} - \frac{\sigma_t^4}{\bar{\sigma}^4} \right)^{-\frac{3}{2}} \cdot \left( (1 + 2 \sigma_t^2 c) \frac{2\sigma_t^2 c}{(2\bar{\sigma}^2 c)^2} + \frac{\sigma_t^2}{\bar{\sigma}^2} - \frac{\sigma_t^4}{\bar{\sigma}^4} \right)  \\
    &+ \left( (2\bar{\sigma}^2 c)^{-2} (1 + 2 \bar{\sigma}^2 c)(1 + 2 \sigma_t^2 c) + 1 \right)^{-\frac{3}{2}} \cdot (1 + 2 \bar{\sigma}^2 c) \frac{2\sigma_t^2 c}{(2\bar{\sigma}^2 c)^2} . 
\end{align*}
When $0 < \sigma_t^2 < \bar{\sigma}^2$, we have the following relations
\begin{align*}
    (2\bar{\sigma}^2c)^{-2}(1 + 2 \sigma_t^2 c)^2 + 2 \frac{\sigma_t^2}{\bar{\sigma}^2} - \frac{\sigma_t^4}{\bar{\sigma}^4} &< (2\bar{\sigma}^2 c)^{-2} (1 + 2 \bar{\sigma}^2 c)(1 + 2 \sigma_t^2 c) + 1, \\
    (1 + 2 \sigma_t^2 c) \frac{2\sigma_t^2 c}{(2\bar{\sigma}^2 c)^2} + \frac{\sigma_t^2}{\bar{\sigma}^2} - \frac{\sigma_t^4}{\bar{\sigma}^4} &> (1 + 2 \bar{\sigma}^2 c) \frac{2\sigma_t^2 c}{(2\bar{\sigma}^2 c)^2}. 
\end{align*}
Therefore,  
\begin{align*}
    \sum_{i \geq 1} \varrho_i \left\langle \frac{d \nu_t}{d\pi} - 1, e_i \right \rangle_{L^2(\pi)} \E_{x\sim\nu_t}\left[\frac{d}{dx} e_i(x) x\right] = G_1 + G_2 < 0,
\end{align*}
proving the second claim.
\end{proof}
\newpage
\end{appendices}


\end{document}